\documentclass[10pt]{article}
\usepackage[utf8]{inputenc}
\usepackage[T1]{fontenc}
\usepackage{graphicx}
\usepackage{hyperref}
\usepackage{geometry}
\usepackage{times}
\usepackage{natbib}  
\usepackage{authblk}  
\usepackage{anyfontsize} 
\usepackage{style/arxiv}

\usepackage{array}
\usepackage{tabu}
\usepackage{wrapfig}
\usepackage{float}

\usepackage{graphicx}
\usepackage{subcaption}
\usepackage{hyperref}
\usepackage{url}
\usepackage{xcolor}

\usepackage{multirow}

\RequirePackage{algorithm}
\RequirePackage{algorithmic}

\newif\ifInlineVersion
\InlineVersiontrue  

\usepackage{amsfonts,bm}
\usepackage{amsmath, amssymb, amsthm}
\usepackage{amstext}
\usepackage{enumitem}

\def\vzero{{\bm{0}}}
\def\vone{{\bm{1}}}

\def\va{{\bm{a}}}
\def\vb{{\bm{b}}}

\def\pb{\tilde{\bm{p}}}

\def\vh{{\bm{h}}}

\def\vp{{\bm{p}}}
\def\vq{{\bm{q}}}

\def\vs{{\bm{s}}}

\def\vv{{\bm{v}}}
\def\vw{{\bm{w}}}

\def\vy{{\bm{y}}}

\def\mL{{\bm{L}}}

\def\mU{{\bm{U}}}

\def\mW{{\bm{W}}}

\def\gA{{\mathcal{A}}}

\def\gC{{\mathcal{C}}}
\def\gD{{\mathcal{D}}}

\def\gF{{\mathcal{F}}}

\def\gM{{\mathcal{M}}}

\def\gP{{\mathcal{P}}}

\def\gS{{\mathcal{S}}}
\def\gT{{\mathcal{T}}}

\def\gV{{\mathcal{V}}}

\def\sI{{\mathbb{I}}}

\def\sN{{\mathbb{N}}}

\def\sR{{\mathbb{R}}}
\def\sS{{\mathbb{S}}}

\newcommand{\E}{\mathbb{E}}

\newcommand{\R}{\mathbb{R}}

\DeclareMathOperator{\GGF}{GGF}

\newcommand{\Prob}{\mathbb{P}}
\newcommand{\M}{\mathcal{M}}

\renewcommand{\S}{\mathcal{S}}
\newcommand{\A}{\mathcal{A}}
\renewcommand{\P}{p}
\renewcommand{\R}{r}

\newcommand{\Mn}{\M^{(N)}}
\newcommand{\Sn}{\S^{(N)}}
\newcommand{\An}{\A^{(N)}}

\newcommand{\Pn}{\P^{(N)}}
\newcommand{\Rn}{\bm{r}}
\newcommand\vmu{\mu^{(N)}}
\newcommand{\cM}{\M_\phi}
\newcommand{\cS}{\Sn_f}
\newcommand{\cA}{\An_{g}}
\newcommand{\cP}{\Pn_\phi}

\newcommand{\bR}{\bar{r}_{\phi}}
\newcommand{\cmu}{\vmu_f}

\newcommand{\s}{\bm{s}}
\renewcommand{\a}{\bm{a}}

\renewcommand{\r}{\bm{r}}
\newcommand{\x}{\bm{x}}
\renewcommand{\u}{\bm{u}}
\newcommand{\so}{\bar{\s}_0} 

\renewcommand{\E}{\mathbb{E}}
\newcommand{\V}{\bm{V}}
\newcommand{\bV}{\bar{V}_0}
\newcommand{\Vo}{\bm{V}_0}

\newcommand{\vpi}{\boldsymbol{\pi}} 
\newcommand{\bpi}{\bar{\vpi}} 
\newcommand{\cpi}{{\vpi}_{\phi}} 
\newcommand{\ppi}{\vpi^Q} 

\newcommand{\G}{\mathcal{G}^{N}}
\newcommand{\ut}{U}

\newcommand{\dx}{\bar{\bm{x}}}

\newcommand{\vpLag}{\hat{\bm{p}}}
\newcommand{\pLag}{\hat{p}}
\newcommand{\aWait}{a^{\text{wait}}}
\newcommand{\aFB}{\hat{a}}

\newcommand{\mAdj}{\bm{A}_{\text{adj}}}
\newcommand{\Adj}{A_{\text{adj}}}

\newcommand{\midd}{\,\big|\,}
\newcommand{\bigd}{\,\Big|\,}
\newcommand{\mybigtimes}{\mathop{\vcenter{\hbox{\LARGE$\times$}}}}
\newcommand{\mymbox}[1]{\mbox{\scriptsize #1}}

\author[1,3]{Xiaohui Tu \thanks{Corresponding author. Email addresses: 
    \href{mailto:xiaohui.tu@hec.ca}{\textit{xiaohui.tu@hec.ca}},
    \href{mailto:yossiri.adulyasak@hec.ca}{\textit{yossiri.adulyasak@hec.ca}}, and  
    \href{mailto:erick.delage@hec.ca}{\textit{erick.delage@hec.ca}}}}
\author[2]{Yossiri Adulyasak}
\author[1,3]{Erick Delage}

\affil[1]{GERAD \& Department of Decision Sciences, HEC Montréal, Montr\'eal, H3T2A7, Canada}
\affil[2]{GERAD \& Department of Logistics and Operations Management, HEC Montréal, Montr\'eal, H3T2A7, Canada}
\affil[3]{MILA - Quebec AI Institute, Montréal, H2S3H1, Canada}

\title{Fair Policy Optimization in Major-Minor Weakly Coupled Markov Decision Processes}

\date{}

\begin{document}
\maketitle

\begin{abstract} \normalsize
We consider fair resource allocation in sequential decision-making environments modeled as major-minor weakly coupled Markov decision processes (M2WCMDP). In this framework, resource constraints couple the action spaces of a major sub-Markov decision process (sub-MDP) and a population of minor sub-MDPs that would otherwise operate independently. Instead of using the traditional utilitarian (total-sum) objective, we optimize a general class of monotone, concave, permutation-invariant, normalized fairness functions. With homogeneous minor sub-MDPs, we prove that the problem under symmetry reduces to optimizing the platform-plus-mean-participant utilitarian objective over the class of \textit{permutation-invariant} policies, which allows us to exploit efficient algorithms that optimize the utilitarian-based objective to solve this fairness-aware problem. For more general settings, we introduce a count-proportion-based deep reinforcement learning approach with a priority-based sampler that generates feasible count actions. The generality of our framework means that the proposed algorithms and theoretical guarantees transfer to any domain with a symmetric M2WCMDP structure. We consider two applications: the machine replacement problem and the joint control of pricing and taxi relocation problem on a New York City-calibrated dataset. We validate our theoretical findings with comprehensive experiments, confirming the effectiveness of our proposed method in achieving strong fairness-aware performance while remaining scalable.
\end{abstract}

\InlineVersiontrue
\section{Introduction}\label{sec:Intro}
The dynamic allocation of shared resources in large-scale stochastic systems is a fundamental problem in various domains, including inventory pricing \citep{gallego1994optimal, adelman2004price}, marketing \citep{bertsimas2007learning, caro2007dynamic}, and preemptive asset maintenance \citep{nadarajah2025self}. These systems are often modeled using weakly coupled Markov decision processes (WCMDPs) \citep{hawkins2003langrangian, adelman2008relaxations}, where each component evolves according to a local sub-Markov decision process (sub-MDP), and coupling arises through shared resource constraints on the joint action space. 

However, a critical challenge arising in such systems is that system-level endogenous information and coordination decisions cannot be represented only by local components. In ride-hailing platforms, for example, joint decisions on pricing, dispatching, and relocation not only determine current matching outcomes, but also influence future revenue through subsequent vehicle availability and the spatial distribution of demands \citep{Chen2024ABC}. The platform observes {system-level information such as outstanding customer queues and controls coordination decisions such as prices}, while individual drivers follow local state transitions determined by their own locations and assigned routes. Similar structures arise in asset maintenance and public infrastructure systems \citep{nadarajah2025self}, where a central operator observes system-level resource availability and allocates limited budgets across assets. These applications motivate a major-minor modeling structure, in which a major sub-MDP represents system-level dynamics, while a large number of minor sub-MDPs represent participants.

At the same time, maximizing total system utility alone is often insufficient when multiple stakeholders are affected. A purely total-sum objective may achieve high overall performance by assigning more profitable opportunities or scarce resources to certain subsets of participants. However, in public systems, operational decisions are often expected to provide equitable service quality, balanced workload allocation, or fair access to opportunities across participants. Examples include equitable driver earnings in mobility platforms \citep{lesmana2019balancing, sun2022optimizing}, fair access to wireless networks \citep{huaizhou2013fairness}, and balanced service levels in healthcare \citep{bertsimas2013fairness} and public infrastructure systems \citep{michailidis2023balancing}. These considerations have led to growing interest in fairness-aware sequential decision-making to promote equitable outcomes among participants \citep{hassanzadeh2023sequential, bhattacharya2024active, aminian2026markovian}.

These two considerations motivate the need for new WCMDP formulations that model both endogenous system dynamics and fairness-aware coordination. Unfortunately, the important computational difficulties arise from the non-separable problem structure and the curse of dimensionality. While classical WCMDP formulations allow tractable analysis through decomposition techniques \citep{adelman2008relaxations} or fluid approximation \citep{brown2025fluid}, they typically assume that sub-MDPs have additive objectives and separable dynamics. These assumptions break down when fairness is optimized as part of the objective. Additionally, existing formulations fail to capture system-level decisions that evolve separately from the individual behaviour, while affecting the transition and reward structure of the participants.

We thus introduce a fairness-aware major-minor WCMDP (M2WCMDP) framework, which consists of one major sub-MDP and a collection of minor sub-MDPs. Conditional on the major state and action, and on each participant's own state-action pair, each minor sub-MDP evolves independently, while their actions remain coupled through shared resource constraints. Within this framework, the objective balances the expected total discounted reward of the major sub-MDP with a fairness measure evaluated over the vector of participant expected total discounted rewards, consistent with social welfare theory \citep{weymark1981generalized, xinying2023guide}. {This fairness measure accommodates many commonly used monotone, concave, permutation-invariant, and normalized welfare measures}, such as generalized Gini functions \citep{weymark1981generalized} and max-min welfare \citep{rawls1971egalitarian}. 

Our first contribution is a structural reduction result for symmetric M2WCMDPs. We prove that, under symmetry conditions, optimizing a fairness-aware objective is equivalent to maximizing a trade-off between the platform's expected total discounted rewards and the participants' average expected total discounted rewards over the class of permutation-invariant policies without loss of {optimality for the fairness-aware objective}. Additionally, motivated by symmetry reduction of \cite{gast2022reoptimization}, the problem complexity is further reduced by obtaining an equivalent problem on a count-based representation.

The second contribution concerns efficient and scalable algorithms. We develop a deep reinforcement learning (DRL) approach that uses count-proportional representations and a stochastic policy network to generate feasible count actions. We show mathematically the feasibility-preserving and expressiveness properties of the proposed priority-based sampling procedure. We also provide an approach to translate aggregate decisions into permutation-invariant participant-level execution that prevents systematic discrimination among participants. In particular, a count policy specifies how many participants take each action, while fairness at the individual level requires a permutation-invariant randomized disaggregation rule that does not systematically favour specific participant labels.

Finally, we evaluate the proposed framework in two numerical settings. The first is a machine replacement problem (MRP), which provides a simple M2WCMDP environment where optimal solutions can be computed for small instances. The second is a taxi dispatching and pricing problem (TDPP) calibrated with trip records from \cite{nyc_tripdata}. This application evaluates the full major-minor coordination structure. The results show that the proposed methods achieve significantly improved fairness-efficiency trade-offs compared to benchmark approaches while remaining computationally scalable in large-scale settings.

\textit{Paper structure.} The remainder of this paper is organized as follows. Section \ref{sec:3problem} introduces the M2WCMDP framework and formalizes the fairness-aware optimization problem. Section \ref{sec:4reduction} establishes the utilitarian reduction under symmetry and presents the count aggregation reformulation. In Section \ref{sec:5policy_nn}, we develop scalable reinforcement learning algorithms based on count-proportion representations and stochastic policy networks. Next, in Section \ref{sec:7numerical}, we validate the proposed framework and methodologies through extensive numerical experiments for the machine replacement and taxi dispatching applications. Finally, we conclude the paper in Section \ref{sec:8conclusion}. The {electronic companions} review related literature, scope the boundary of the welfare class, provide proofs and exact count-model formulations, as well as report implementation details, benchmark formulations, and additional experiments.

\section{Fairness-Aware Major-Minor Weakly Coupled Markov Decision Processes}\label{sec:3problem}
We start by introducing the labelled infinite-horizon M2WCMDPs in Section \ref{sec:3.1-model}. In Section \ref{sec:3.2-fair-problem}, we {characterize the class of fairness measures considered in this work through a set of properties motivated by social welfare theory}, and then formulate the fairness-aware optimization problem. Finally, in Section \ref{sec:3.3-example}, we instantiate the framework with two applications that will be used in the numerical studies.

\textbf{Notation.} Let $[N] := \{1, \dots, N\}$ for any integer $N$. An indicator function $\sI\{x \in A\}$ equals 1 if $x \in A$ and 0 otherwise. For a finite set $X$, let $\Delta(X)$ denote the probability simplex over $\sR^{|X|}$. {For a vector-valued mapping $\xi: \sR^N \rightarrow \sR^{|X|}$ with components indexed by a finite set $X$, we write $[\xi(\vy)](x)$ for its $x$-indexed component, for any $\vy \in \sR^N$ and $x \in X$.} We use $\vone$ and $\vzero$ to denote all-ones and all-zero arrays, respectively, with their dimensions specified by subscripts. We define $\G$ as the set of all $N \times N$ permutation matrices, where $Q \in \G$ is a linear operator that acts on any $\vv \in \sR^N$ by reordering its components.

\subsection{Model Formulation}\label{sec:3.1-model}
{We consider a centrally coordinated system evolving over an infinite horizon $t \in \gT := \{0, 1, \dots\}$. This system is formulated as a discrete-time Markov Decision Process (MDP) that decomposes into a major-minor format, which consists of one major sub-MDP $\M^0$ (referred to as the \textit{platform}) and $N$ minor sub-MDPs} $\{\M^n\}_{n \in [N]}$ (referred to as \textit{participants}). A central decision maker jointly optimizes system-level efficiency and participant-level welfare.

Formally, the major sub-MDP $\M^0$ is characterized by the tuple $(\mathcal{S}^0, \A^0, p^0, r^0, \mu^0, \gamma)$, where $\S^0$ and $\A^0$ are the finite major state and action spaces, respectively. The major transition probability $p^0(\cdot)$ and real-valued reward function $r^0(\cdot)$ depend on the joint state and action of the entire system, and will be defined in Equation \ref{eq:major-transition-kernel} and the paragraph following it. The $n$-th minor sub-MDP $\M^n$ is defined by the tuple {$(\S^n, \A^n, \P^n, \R^n, \mu^n, \gamma)$}, where $\S^n$ is a finite set of states with cardinality $S$, and $\A^n$ is a finite set of actions with cardinality $A$. The stationary transition probability function is denoted by $p^n(\tilde{s}^n \mid s^0, s^n, a^0, a^n) = \Prob(s^n_{t+1} = \tilde{s}^n \mid s^0_t = s^0, s^n_t = s^n, a^0_t = a^0, a^n_t = a^n)$, representing the probability of reaching minor state $\tilde{s}^n \in \S^n$, after performing minor action  $a^n \in \A^n$ and major action $a^0 \in \A^0$ in minor state $s^n \in \S^n$ and major state $s^0 \in \S^0$ at time $t$. The reward function $r^n(s^0, s^n, a^0, a^n)$ denotes the immediate real-valued reward obtained by executing action $a^n$ in state $s^n$ under the coordination of the major system $(s^0, a^0)$. In the general formulation, the transition probabilities and the reward functions may vary with the sub-MDP $n$, and we assume only that they are stationary over time for simplicity. The initial state distributions are represented by $\mu^0 \in \Delta(\S^0)$ and $\mu^n \in \Delta(\S^n)$, while each sub-MDP's initial state is assumed to be drawn independently. The discount factor, common to all sub-MDPs, is denoted by $\gamma \in [0, 1)$.

The state of the system at any time is represented by the pair $(s^0, \s)$. The component $s^0 \in \S^0$ corresponds to the state of the major sub-MDP. The joint state space for the $N$ minor sub-MDPs is given by $\s = (s^1, \dots, s^N) \in \Sn$, with $\Sn$ as the Cartesian product of the individual state spaces. There is a set of $K$ resource constraints that couple the action selection of the minor sub-MDPs. Let $d_k(a^n \mid s^n)$ be the state-dependent non-negative consumption of the $k$-th resource by the $n$-th minor sub-MDP. Then, for a given system state $(s^0, \s)$, the feasible joint minor action space is a state-dependent subset of the Cartesian product of minor action spaces, defined as

\ifInlineVersion
\begin{equation} \label{eq:constraint} 
    \An(s^0, \s) := \left\{(a^1, \dots, a^N) \in \mybigtimes_{n=1}^N \A^n \bigd \sum_{n=1}^N d_k(a^n \mid s^n) \leq b_k(s^0), \forall k \in [K]\right\}, 
\end{equation}
\else
$\An(s^0, \s) := \left\{(a^1, \dots, a^N) \in \mybigtimes_{n=1}^N \A^n \bigd \sum_{n=1}^N d_k(a^n \mid s^n) \leq b_k(s^0), \forall k \in [K]\right\}$,
\fi
where $b_k: \S^0 \to \sR_{+}$ is the budget function mapping the major state to the available resource capacity of type $k$ to capture the exogenous supply of resources (e.g., passenger demands) which varies over time. The major action does not directly restrict the feasible minor action space, but serves as a coordination signal that influences the system through the transition dynamics and the reward function.

The minor sub-MDP transitions are conditionally independent. {From an operational perspective, this assumption implies that conditional on the platform's state (e.g., outstanding customer queues), the macro-level coordination signal (e.g., the price multiplier), and each participant's own state-action pair, participants' state transitions are mutually independent, as a driver's subsequent travel time and next location depend only on their own realized trajectory.} Specifically, minor sub-MDPs transition from state $\s$ to state $\tilde{\s}$ for a given major action $a^0$ and feasible minor actions $\a = (a^1, \dots, a^N)$ at time $t$ with probability 
\ifInlineVersion
\begin{equation*}
    \Pn (\tilde{\s} \mid s^0, \s, a^0, \a) := \prod^{N}_{n=1} p^n(\tilde{s}^n \mid s^0, s^{ n}, a^0, a^{n}) = \prod^{N}_{n=1} \Prob(s_{t+1}^n = \tilde{s}^n \mid s^0_t = s^0, s_{t}^n = s^n, a^0_t = a^0, a_{t}^n = a^n),
\end{equation*}
\else
$\Pn (\tilde{\s} \mid s^0, \s, a^0, \a) := \prod^{N}_{n=1} p^n(\tilde{s}^n \mid s^0, s^{ n}, a^0, a^{n}) = \prod^{N}_{n=1} \Prob(s_{t+1}^n = \tilde{s}^n \mid s^0_t = s^0, s_{t}^n = s^n, a^0_t = a^0, a_{t}^n = a^n)$, 
\fi
while the major sub-MDP has the transition kernel
\begin{equation} \label{eq:major-transition-kernel}
\ifInlineVersion
    p^0 (\tilde{s}^0 \mid s^0, \s, a^0, \a) := \Prob(s_{t+1}^0 = \tilde{s}^0 \mid s^0_t = s^0, \s_{t} = \s, a_t^0 = a^0, \a_t = \a).
\else
\begin{aligned}
    & p^0(\tilde{s}^0\mid s^0,\s,a^0,\a):= \Prob \left(s_{t+1}^0=\tilde{s}^0 \mid \right.\\
    & \hspace{3em} \left. s_t^0=s^0, \s_t=\s, a_t^0=a^0,\ \a_t=\a \right).
\end{aligned}
\fi
\end{equation}

After choosing an action $(a^0, \a) \in \A^0 \times \An(s^0, \s)$ when in state $(s^0, \s) \in \S^0 \times \Sn$, the decision maker receives rewards for the minor sub-MDPs defined as {$\Rn(s^0, \s, a^0, \a) = (r^{1}(s^0, s^1, a^0, a^1), \dots, r^{N}(s^0, s^N, a^0, a^N))$} with each component representing the reward associated with the respective minor sub-MDP $\M^n$. We employ a vector form for the rewards to later offer the flexibility for formulating fairness objectives on individual expected total discounted rewards. The major sub-MDP receives a real-valued reward, denoted by {$r^0(s^0, \s, a^0, \a)$}, which represents the platform-level performance metric. 

Collectively, the infinite-horizon M2WCMDP model is completely specified by the collection $(\M^0, \Mn)$, where $\M^0 :=(\mathcal{S}^0, \A^0, p^0, r^0, \mu^0, \gamma)$ and $\Mn:= (\Sn, \An, \Pn, \r, \vmu, \gamma)$. Under centralized coordination, the system observes the full joint state and jointly selects the major action and a resource-feasible minor action for each participant. Let $\Pi$ denote the set of all feasible stationary randomized Markov policies. A joint policy $\vpi \in \Pi$ is formally defined by a mapping from the joint state space $\S^0 \times \Sn$ to a probability distribution over the state-dependent joint action space $\A^0 \times \An(s^0, \s)$. Thus, $\vpi(a^0, \a \mid s^0, \s)$ represents the probability of taking the joint action $(a^0, \a)$ given the joint state $(s^0, \s)$. We define this feasible policy space as $\Pi := \left\{\vpi \midd \vpi(\cdot, \cdot \mid s^0, \s) \in \Delta(\A^0\times \An(s^0,\s)), \, \forall (s^0, \s) \right\}$.

The initial system state $(s^0_0, \s_0)$ is sampled from the joint initial distribution $\bm{\mu} \in \Delta(\S^0 \times \Sn)$. Using the discounted-reward criterion, the performance of any joint policy $\vpi$ is evaluated by its expected total discounted rewards. For the major sub-MDP $\M^0$ under the initial distribution $\bm{\mu}$, this is given by:
\begin{equation}\label{eq:value}
\ifInlineVersion
    V^0_0(\vpi):= \mathbb{E}_{\vpi}\left[\sum_{t=0}^\infty \gamma^t r^0(s^0_t, \s_t, a^0_t, \a_t) \, \Big|(s^0_0, \s_0)\sim \bm{\mu} \right].
\else
\begin{aligned}
    & V^0_0(\vpi):= \\
    & \hspace{1em} \mathbb{E}_{\vpi}\left[\sum_{t=0}^\infty \gamma^t r^0(s^0_t, \s_t, a^0_t, \a_t) \, \Big|(s^0_0, \s_0)\sim \bm{\mu} \right].
\end{aligned}
\fi
\end{equation}

Similarly, the expected total discounted rewards $V^n_0(\vpi)$ specific to the $n$-th sub-MDP $\gM^n$, under the initial distribution $\bm{\mu}$ and policy $\vpi$, is defined as 
\begin{equation}\label{eq:vector-component}
\ifInlineVersion
    {V}^n_0(\vpi) := \mathbb{E}_{\vpi}\left[\sum_{t=0}^\infty \gamma^t r^n(s^0_t,s^n_t, a^0_t, a^n_t)  \, \Big| \,  (s_0^0,\s_0)\sim\bm{\mu}\right], \forall n \in [N],
\end{equation}
\else
\begin{aligned}
    & {V}^n_0(\vpi) := \\
    & \hspace{1em} \mathbb{E}_{\vpi}\left[\sum_{t=0}^\infty \gamma^t r^n(s^0_t,s^n_t, a^0_t, a^n_t)  \, \Big|\, (s_0^0,\s_0)\sim\bm{\mu}\right],
\end{aligned}
\end{equation}
for all $n \in [N]$, 
\fi
where $(a^0_t, \a_t) \sim \vpi(\cdot, \cdot \mid s^0_t, \s_t)$. The vectorial expected total discounted rewards $\V_0(\vpi)$ for all minor sub-MDPs under policy $\vpi$ is defined as 
\begin{equation}\label{eq:vector}
    \V_0({\vpi}) := [{V}_0^1({\vpi}), \dots, {V}_0^N({\vpi})]^\top.
\end{equation}

Traditionally, the performance of such a centrally coordinated system is evaluated by maximizing a purely utilitarian objective. Expressed through our formulated value functions (\ref{eq:value}) and (\ref{eq:vector}), this typically involves maximizing the expected infinite-horizon discounted total reward of the platform and all participants, $V^0_0(\vpi) + \vone^\top \V_0(\vpi)$ \citep{zhu2021mean}. When no separate platform reward is modelled, this is reduced to $\vone^\top \V_0(\vpi)$ \citep{Chen2024ABC}. However, under a strictly efficiency-driven objective, the utilities of marginalized participants might be systematically sacrificed to achieve a higher system performance. Fairness concerns therefore arise naturally in various decision-making contexts and application domains, including healthcare management \citep{bertsimas2013fairness}, transportation \citep{gutjahr2018equity, Chen2024ABC}, facility location \citep{filippi2021single}, and resource allocation \citep{hassanzadeh2023sequential, bhattacharya2024active}. Fairness should be explicitly modeled to reflect equity in the distribution of outcomes across the participants in such systems \citep{aziz2024best}. This concern maps directly to our vector-valued welfare framework, where we seek to evaluate the joint performance vector $\V_0(\vpi)$ rather than only the total sum. This motivates the fairness-aware optimization problem.

\subsection{Fair Optimization Problem}\label{sec:3.2-fair-problem}
In modern welfare economics, ensuring equity among multiple participants fundamentally reduces to formulating a socially meaningful scalarization of vector-valued outcomes. Grounded in the social choice and utility theory literature \citep{weymark1981generalized, moulin1991axioms, tsang2025unified}, we resort to a social welfare function that aggregates the vectorial expected discounted rewards of all minor sub-MDPs into a scalar representation of system-level welfare criteria. This approach provides a unifying mathematical framework for representing welfare objectives suitable for optimization contexts \citep{xinying2023guide, fan2023welfare}.

Rather than prescriptively defining what constitutes a fair measure, we identify the mathematical properties required to enable tractable fair policy optimization. We formally define $\rho$-fairness, which serves as the theoretical foundation for our subsequent utilitarian reduction theorem in Section \ref{sec:4reduction}.

\begin{definition}[$\rho$-fairness] \label{def:rho-fairness}
Let $\rho:\gV\rightarrow\sR$, with $\gV\subseteq\sR^N$ as the domain of $\rho$, be a fairness measure  that satisfies:
\begin{itemize}
    \item {\textbf{Monotonicity (MN)}: For any $\vv,\vw\in\gV$ such that $\vv\geq\vw$ componentwise, $\rho[\vv]\geq\rho[\vw]$,}
    \item \textbf{Concavity (CV)}: The set $\gV$ is convex and $\forall \vv,\vw\in\gV$, and $\tau\in[0,1]$, $\rho[\tau \vv + (1-\tau)\vw]\geq \tau\rho[\vv]+(1-\tau)\rho[\vw]$,
    \item \textbf{Permutation invariance (PI)}: $\forall \vv\in\gV$ and all $Q\in\G$, both $Q\vv\in\gV$ and $\rho[\vv]=\rho[Q \vv]$,
    \item \textbf{Constant vector invariance (CVI)}: {$\forall c \in \sR$ such that $c\vone \in \gV$, $\rho[c\vone]=c$.}
\end{itemize}
\end{definition}

{MN ensures that improving participant outcomes componentwise cannot reduce social welfare.} CV reflects a preference for equality through decreasing marginal gains and guarantees that redistributing reward among participants cannot decrease welfare. CV is not only a desirable property in optimization contexts, but also a property satisfied by many commonly used fairness measures as if a fairness measure is non-concave, its value could paradoxically increase by partitioning the groups of interests into subgroups \citep{kolm1976unequal, williamson2019fairness, tsang2025unified}. Next, PI ensures that welfare does not depend on participant identities \citep{moulin1991axioms, xinying2023guide}. CVI ensures the measure is normalized to standard individual reward units, which provides a normalized economic scale when trading off against the platform's utilitarian reward.

We provide two examples on fairness measures satisfying these conditions that will be used in numerical experiments. The comparison of additional welfare measures and Definition \ref{def:rho-fairness} compliance is provided in \ref{apx:welfare-examples} to scope the boundary of the proposed welfare class.

\begin{example}[Generalized Gini Function]\label{exmp:ggf}
A typical instance is the Generalized Gini Function (GGF) \citep{weymark1981generalized, siddique2020learning}, defined by assigning ordered weights to the components of a utility vector that $\GGF_{\vw}[\vv] := \min_{\sigma \in \sS^N} \sum_{n=1}^N w_n v_{\sigma(n)}$, where $\sS^N$ is the set of all permutations of $N$ elements, and $\vw \in \Delta([N])$ satisfies $w_1 \geq w_2 \geq \dots \geq w_N$. The minimizer $\sigma^*$ effectively sorts $\vv$ in ascending order, prioritizing the lowest earners. Another prominent special case is the Rawlsian max-min fairness \citep{rawls1971egalitarian}, $\rho[\vv] = \min_{n \in [N]} \{v_n\}$, which corresponds to $w_1 = 1$ and $w_{i} = 0, \forall i >1$. Both natively satisfy Definition \ref{def:rho-fairness}. \hfill $\square$
\end{example}

\begin{example}[Certainty-Equivalent Representations of $\alpha$-Fairness]\label{exmp:ce-alpha}
The classical $\alpha$-fair welfare family \citep{mo2000fair, ju2023achieving} is parameterized by $\alpha > 0$ and takes the additive form $W_\alpha[\vv] = \sum_{n=1}^N \frac{v_n^{1-\alpha}}{1-\alpha}$, with $\sum_{n=1}^N \log(v_n)$ for $\alpha=1$ as a special instance of Nash Social Welfare \citep{fan2023welfare, mandal2022socially}. While it satisfies MN, CV and PI, it violates CVI because it is not normalized. For instance, evaluating a uniform outcome $\V = \bar{v}\vone$ yields $N \frac{\bar{v}^{1-\alpha}}{1-\alpha}$, explicitly violating CVI that requires $\rho(\bar{v}\vone) = \bar{v}$. To restore normalization, we apply the certainty-equivalent alternatives via the inverse utility function $u^{-1}$ by setting the expected utility model $\rho[\vv] = u^{-1}\left({\frac{1}{N}}\sum_{n=1}^N u(v_n)\right)$, where $u(\cdot)$ is a monotone and concave function. This leads to the alternative ordinally equivalent welfare class $\rho_\alpha[\vv] = \left(\frac{1}{N}\sum_{n=1}^N v_n^{1-\alpha}\right)^{\frac{1}{1-\alpha}}$ for $\alpha > 0, \alpha \neq 1$, and the geometric mean $\rho_1[\vv] = \left(\prod_{n=1}^N v_n\right)^{1/N}$ for Nash Social Welfare. As limits, this family recovers the utilitarian mean ($\alpha = 0$), the normalized Nash Social Welfare \citep{fan2023welfare} ($\alpha \to 1$), and the Rawlsian max-min \citep{rawls1971egalitarian} ($\alpha \to \infty$), all of which satisfy Definition \ref{def:rho-fairness}. \hfill $\square$
\end{example}

Given an initial joint state distribution $\bm{\mu} \in \Delta(\S^0 \times \Sn)$, we are seeking a policy that achieves the best tradeoff between the efficiency of the platform and the fairness of participants \citep{tsang2025unified}. 
Formally, the objective of our fairness-aware optimization problem ($\rho$-M2WCMDPs) under policy $\vpi$ is given by
\begin{equation}\label{eq:obj-rho-M2WCMDP}
    \max_{\vpi \in \Pi}G_\rho(\vpi) := V^0_0(\vpi) + \lambda \rho[\V_0(\vpi)],
\end{equation}
where $V^0_0(\vpi)$ denotes the expected total discounted rewards of the major sub-MDP, $\V_0(\vpi)$ is the vectorial counterpart for the minor sub-MDPs, both evaluated under the initial distribution $\bm{\mu}$, and $\lambda \geq 0$ is the weighting factor. We adopt this weighted-sum scalarization because it provides a transparent and widely used mechanism for balancing platform efficiency and participant welfare through a single preference parameter $\lambda$ \citep{bertsimas2012efficiency, hooker2012combining, tsang2025unified}.
{We further note that when $\rho$ has a restricted domain, i.e., $\mathcal{V}\neq \sR$, one needs to impose the following assumption to make the $\rho$-M2WCMDP well defined. 
\begin{assumption}\label{assum:domain}
    The expected discounted total reward vector generated by any feasible policy lies entirely within the domain of $\rho$. That is, $\V_0(\vpi) \in \gV, \forall \vpi \in \Pi$.
\end{assumption}
This assumption is required to ensure that $G_\rho(\vpi)$ is well defined for all $\pi\in\Pi$. Alternatively, one should consider using a different $\rho$ measure or formulating a version of $\rho$-M2WCMDP that constrains $\V_0(\vpi) \in \gV$, which unfortunately falls outside the scope of our study.}

\subsection{Two Illustrative Examples}\label{sec:3.3-example}
To contextualize the M2WCMDP framework and the tension between system efficiency and equity, we introduce two representative applications that will be extensively studied in the subsequent numerical experiments.

{\begin{example}[Machine Replacement Problem] \label{exmp:mrp} This problem is adapted from \cite{delage2010percentile, akbarzadeh2019restless}, and represents a natural degenerate case where no active major sub-MDP is present (i.e., $\S^0$ and $\A^0$ are singletons). Consider a maintenance system consisting only of a large collection of deteriorating machines (minor sub-MDPs $\Mn$), reducing the system to a classical WCMDP. For each machine $n$, the minor state $s^n \in \S^n$ captures its current deterioration level, and the binary action space $\A^n := \{0, 1\}$ corresponds to either continuing normal operation ($a^n=0$) or performing an active replacement ($a^n=1$) to reset the machine state to the newest condition. The sub-MDPs are weakly coupled through the maintenance budget $b_k$ (e.g., available repair crews or parts) that implicitly restricts the decision space for maintenance resources of type $k \in [K]$ at each time step. Formally, the feasible joint minor action space is coupled as $\An(\s) := \left\{\a \in \mybigtimes_{n=1}^N \A^n \bigd \sum_{n=1}^N d_k(a^n \mid s^n) \leq b_k, \, \forall k \in [K] \right\}$. \hfill $\square$
\end{example}}

In the context of Example \ref{exmp:mrp}, if equipment is regionally distributed in cases like electricity or telecommunication networks \citep{nadarajah2025self}, a fair policy guarantees equitable operations and replacements, thereby preventing frequent failures in specific areas that lead to unsatisfactory and unfair results for certain customers. A similar issue occurs in taxi dispatching systems (as detailed in Example \ref{exmp:taxi-dispatching}), purely efficiency-driven ride assignments may lead to the situation where some drivers consistently get profitable trips, while others systematically receive fewer opportunities \citep{dai2017balanced}. 

\begin{example}[Taxi Dispatching and Joint Pricing Control Problem]\label{exmp:taxi-dispatching}
A mobility-on-demand platform coordinates $N$ homogeneous taxis over zones $\gM := \{1, \dots, M\}$. The network connectivity is represented by a graph with an adjacency matrix $\mAdj \in \{0,1\}^{M \times M}$, where $\Adj(i, j) = 1$ indicates the presence of a {direct} movement from zone $i$ to $j$. The travel times between any two nodes $i, j \in \gM$ are given by a deterministic duration matrix $D \in \sN^{M \times M}$, which is proportional to the shortest geodesic distance between two nodes. The maximum travel duration across the network is defined as $D_{\max} := \max_{i,j \in \gM} D(i,j)$, and let $\gD := \{0, \dots, D_{\max}\}$.

The fleet current state is given by $\s = \{s^n\}_{n \in [N]}$, where the state of each taxi $n$ is a tuple $s^n := (i, d) \in \M \times \gD$. This indicates that taxi $n$ is en route to destination $i$ with $d$ remaining time steps. A new decision is made when the driver is available at the destination zone ($d= 0$). The platform monitors the {origin-destination (OD)} queue of unassigned ride requests {$\vq_t \in \sN^{M \times M}$ and offers a price-multiplier matrix $\vp_t \in \sR_{+}^{M \times M}$ at time $t$}\endnote{Here, finiteness of  state and action spaces can be obtained by assuming bounded OD queues and discretizing the support of the price multipliers. In our implementation, we let price multipliers to be continuous with the assumption that our theory from sections \ref{sec:3problem} and \ref{sec:4reduction} still holds in continuous spaces.}. To guarantee price consistency for attracted customers \citep{Chen2024ABC}, the platform uses a lagged price $\vp_{t-1}$ for order matching and the price-induced demands are non-stationary. Accordingly, we incorporate the pricing multiplier into the augmented major state $s^0_t := (\vq_t, \vp_{t-1}, t)$ at time $t$, also written as $s^0=(\vq,\vpLag,t)$ when the time index is suppressed.

Based on the joint augmented states $(s^0, \s)$, the platform chooses the major action $a^0 :=\vp$ and a routing action $a^n := (j,o) \in \M \times \{0,1\}$ for each idle taxi, where $j$ is the destination zone and $o$ indicates taxi occupancy ($o=1$ for a passenger match and $0$ for relocation). The resource budgets are of $K = M^2$ types, each referring to the number of unmatched orders $b_{ij}(s^0) := q(i, j)$ for each OD pair $(i,j)$. For taxi $n$, we define the indicator function of accepting a ride from $i$ to $j$ as $d_{ij}(a^{n}\mid s^{n}) := \sI\{s^n = (i,0), a^{n} = (j, 1) \}$. The feasible joint minor action set is $\An(s^0,\s):=\left\{\a\in\mybigtimes_{n=1}^N\A^n(s^n) \bigd \sum_{n \in [N]} \sI\{s^n = (i,0), a^{n} = (j, 1)\} \leq q(i, j), \, \forall i, j \in \gM \right\}.$ Thus, taxis have local routing actions, but their passenger-matching decisions cannot be selected independently because they compete for the same finite set of ride requests.

The state transitions involve the evolution of both the fleet and the demand arrivals. Each idle taxi ($d=0$) takes a routing action, while in-transit taxis decrement their timers deterministically, i.e.,
\begin{equation}
    s^n_{t+1} :=
    \begin{cases}
        (j, D(i, j)) & \text{if } d = 0, a^n = (j ,\cdot)\\
        (i, d - 1) & \text{if } d \geq 1 \\
    \end{cases}, \quad \forall n.
\end{equation}

The customer request queue evolves stochastically according to $q_{t+1}(i,j) = n_{t+1}(i,j)$ that initializes new arrivals $n_{t+1}(i,j)$ drawn from independent Poisson distributions with per-OD rates $\theta_{t+1}(i,j)=\theta^{\text{base}}_{t+1}(i,j) \psi(p_t(i,j))$ for all $i,j \in \gM$, where $\psi(\cdot)$ is a monotonic price-response function and the base rate $\bm{\theta}^{\text{base}}_t$ is estimated from historical data.

The system objective balances platform service fees, driver income, and operational travel costs. Let $c_f \geq0$ the fixed fare of a served trip, $c_r \geq0$ denote the variable per-unit ride price, $c_o \geq0$ denote the variable per-unit operational cost for driving, and $\beta \in (0, 1)$ be the platform service rate. The immediate reward to the driver $n$ in state $(i, d)$ when taking action $a^n$ is given by
\begin{equation}
\ifInlineVersion
    r(s^0, s^n, a^n) :=
    \begin{cases}
    (1-\beta) \left[c_f+c_r\pLag(i,j)D(i,j)\right] - c_o D(i, j) & \text{if } d = 0, a^n = (j, 1),\\
    - c_o D(i, j) & \text{if } d = 0, a^n = (j, 0),\\
    0 & \text{if } d \geq 1.\\
    \end{cases}
\else
\begin{aligned}
    & r(s^0, s^n, a^n) :=\\
    & \begin{cases}
    \begin{aligned}
    & (1-\beta) \left[c_f+c_r\pLag(i,j)D(i,j)\right] - c_o D(i, j) & \\
    & \hspace{9.4em}\text{if } d = 0, a^n = (j, 1), &\\
    & - c_o D(i, j) \hspace{4.65em} \text{if } d = 0, a^n = (j, 0), &\\
    & 0 \hspace{8.85em} \text{if } d \geq 1. &
    \end{aligned}
    \end{cases}
\end{aligned}
\fi
\end{equation}

The platform’s net reward is the total service-fee revenue collected from all matched rides, i.e., $r^0(s^0, \s, \a) := \, \beta \sum_{i,j \in \M}  \sum_{n \in [N]} \sI\{s^n = (i,0), a^{n} = (j, 1)\} \left[c_f+c_r\pLag(i,j)D(i,j)\right]$. The objective is to maximize the trade-off between the efficiency of the platform and the fairness of the participants following (\ref{eq:obj-rho-M2WCMDP}). \hfill $\square$
\end{example}

\section{Utilitarian Reduction under Symmetry}\label{sec:4reduction}
This section develops the structural reduction for symmetric M2WCMDPs. In Section \ref{sec:rho_problem}, we start by formally defining the symmetric M2WCMDPs (Definition \ref{def:symmetry-M2WCMDP}) and prove that an optimal policy of the symmetric $\rho$-M2WCMDP problem can be obtained by solving the utilitarian M2WCMDP using “permutation-invariant” policies. Section \ref{sec4.2:count_mdp} proposes an exact count aggregation method to further simplify the model. Finally, Section \ref{sec:3.3count-examples} instantiates the count formulation for the machine replacement problem of Example \ref{exmp:mrp} and taxi application of Example \ref{exmp:taxi-dispatching}. The resulting count aggregation MDP provides the optimization model addressed in Section \ref{sec:5policy_nn}.

\subsection{Symmetric \texorpdfstring{$\rho$}{rho}-M2WCMDPs Problem Reduction} \label{sec:rho_problem}
We first establish the formal conditions for a M2WCMDP to be considered symmetric.

\begin{definition}[Symmetric M2WCMDP]\label{def:symmetry-M2WCMDP}
A M2WCMDP is symmetric if
\begin{enumerate}[leftmargin=*]
    \setlength{\parskip}{3pt}
    \item \textbf{(Identical Minor Sub-MDPs)} Each minor sub-MDP is identical, i.e.,  $\S^n = \S$, $\A^n = \A$, $\P^n = \P$, $\R^n = \R$, $\mu^n = \mu$, for all $n\in [N]$, and for some $(\S, \A, \P, \R, \mu, \gamma)$ tuple.
    \item \textbf{(Permutation-Invariant Initial Distribution)}  For any permutation operator $Q \in \G$, the probability of selecting the permuted initial minor state $Q \so$ at major state $\bar{s}^0_0$ is equal to that of selecting $\so$, i.e., $\bm{\mu}(\bar{s}^0_0, \so) = \bm{\mu}(\bar{s}^0_0, Q \so), \forall (\bar{s}^0_0, \so) \in \S^0 \times \Sn, \forall Q \in \G$.
    \item \textbf{(Permutation-Invariant Participants Influence on Major sub-MDP)}  The probability of transitioning from $s^0$ to $\tilde{s}^0$ under action $a^0$ when the participants apply $(\s,\a)$ satisfies $p^0(\tilde{\s}^0 \mid s^0, \s, a^0, \a) = p^0(\tilde{\s}^0 \mid s^0, Q \s, a^0, Q \a)$ for all $Q \in \G$. The reward function of the major sub-MDP satisfies $r^0(s^0,\s,a^0,\a)=r^0(s^0,Q\s,a^0,Q\a)$ for all $Q \in \G$.
\end{enumerate}
\end{definition}

The conditions for a symmetric M2WCMDP define a class of problems that is invariant under any permutation of sub-MDP indexing. This invariance naturally leads to the concept of \textit{permutation-invariant} policies (see Definition 1 in \cite{cai2021efficient}).

\begin{definition}[Permutation Invariant Policy]\label{def:permutation-invariant-policy}
    A Markov stationary policy $\vpi$ is said to be permutation-invariant if the probability of selecting action $(a^0,\va)$ in state $(s^0,\vs)$ is equal to that of selecting the permuted action $Q\va$ in the permuted state $Q\vs$, for all $Q \in \G$. Formally, this can be expressed as $\vpi(a^0,\a\mid s^0,\s) = \vpi(a^0,Q\a\mid s^0,Q\s)$, for all $Q \in \G$, $\vs \in \Sn$ and $\va \in \An$. 
\end{definition}

This naturally brings us to the question of  whether this class of policies is optimal for symmetric M2WCMDPs. The key observation is that symmetrizing a policy preserves the platform value and the mean participant value while equalizing the participant-value vector. The properties PI and CV of $\rho$ imply that this equalization cannot decrease welfare, and CVI identifies the welfare of the resulting constant vector with its common component. This argument connects the fairness-aware problem to the utilitarian mean in Definition \ref{def:utilitarian}.

\begin{definition}[Utilitarian Approach]\label{def:utilitarian}
    The utilitarian approach is the unique $\rho$-fair (Definition \ref{def:rho-fairness}) linear mapping on $\sR^N$, given by $$\rho_U[\vv] := \frac{1}{N}\sum_{n=1}^N v_n =: \bar{v}.$$
\end{definition}

\begin{remark}
    The uniqueness of $\rho_U$ follows from the fact that any $\rho$-fair linear mapping on $\sR^N$ can be represented as $\vw^\top \vv$. Permutation invariance restricts $\vw$ to the set of constant vectors ($\vw = c\mathbf{1}$ for some $c \in \sR$), while the constant vector invariance property forces $c = 1/N$.
\end{remark}

Additionally, the PI property implies that averaging the occupancy measure over all label permutations equalizes participant values. Using the stationary-policy construction associated with a feasible discounted occupancy measure (Theorem 6.9.1 from \cite{puterman2014markov}), we construct a permutation-invariant policy from any policy, resulting in uniform state-value representation (Lemma \ref{thm:bar-policy}).

\begin{lemma}[Uniform State-Value Representation]\label{thm:bar-policy}
    If a M2WCMDP is symmetric (Definition \ref{def:symmetry-M2WCMDP}), then for any policy $\vpi$, there exists a corresponding permutation-invariant policy $\bpi$ such that the vector of expected total discounted rewards for all sub-MDPs under $\bpi$ is equal to the average of the expected total discounted rewards for each sub-MDP, i.e., $\V_0(\bpi) = \bar{V}_0(\vpi) \vone$, where $\bar{V}_0(\vpi) := \frac{1}{N} \sum_{n=1}^N V_0^n(\vpi)$, and $V^0_0(\bpi) = V^0_0(\vpi)$.
\end{lemma}

The proof is detailed in \ref{apx:bar-policy}. Furthermore, one can use the above lemma to show that the optimal policy for the $\rho$-M2WCMDP problem (\ref{eq:obj-rho-M2WCMDP}) under symmetry can be recovered from solving the problem with the utilitarian approach. Our main result is presented in the following theorem. See \ref{apx:thm:eqv} for a detailed proof.

\begin{theorem}[Utilitarian Reduction] \label{thm:eqv}
    Under Assumption \ref{assum:domain}, for a symmetric M2WCMDP, let $\Pi^*_{U,\mymbox{PI}}$ be the set of optimal policies for the utilitarian approach that is permutation-invariant, then $\Pi^*_{U,\mymbox{PI}}$ is necessarily non-empty and all $\vpi^*_{U,\mymbox{PI}}\in\Pi^*_{U,\mymbox{PI}}$ satisfy $G_\rho(\vpi_{U,\mbox{\tiny{PI}}}^{*})=\max \limits_{\vpi \in \Pi}  G_\rho(\vpi).$
\end{theorem}

This theorem simplifies solving the $\rho$-M2WCMDP problem by reducing it to an equivalent utilitarian problem, showing that any permutation-invariant policy that is utilitarian optimal is also optimal for the original $\rho$-M2WCMDP problem under the utilitarian reduction that takes the form
\begin{equation}\label{eq:obj-utilitarian-reduced}
\ifInlineVersion
    \mbox{(Utilitarian approach)} \quad\max_{\vpi \in \Pi}G_{\ut}(\vpi) := V^0_0(\vpi) + \lambda \bar{V}_0(\vpi),
\else
\begin{aligned}
    &\mbox{(Utilitarian approach)} \\
    & \hspace{2.5em} \quad\max_{\vpi \in \Pi}G_{\ut}(\vpi) := V^0_0(\vpi) + \lambda \bar{V}_0(\vpi),
\end{aligned}
\fi
\end{equation}
where $\bar{V}_0(\vpi) := (1/N)\sum_{n=1}^N V^n_0(\vpi)$. Therefore, we can restrict the search for optimal policies to the class of permutation-invariant policies.

\subsection{The Count Aggregation MDP} \label{sec4.2:count_mdp}
Under Definition \ref{def:symmetry-M2WCMDP}, the labelled M2WCMDP is invariant under permutations of the minor indices. This invariance property allows us to further simplify the problem using a more compact count-based representation motivated by the symmetry simplification representation in \cite{gast2022reoptimization} for the reduced utilitarian objective.

Let $\phi=(f, g)$ denote the aggregation mapping. The mapping $f: \Sn \rightarrow \{0, \dots, N\}^{|\gS|}$ maps from the labelled minor state $\s$ to a count state $\x =f(\s)$, where $[f(\s)](s):=\sum_{n=1}^N \sI\{s^n=s\}, \forall s\in \gS$. The corresponding count state space is
\begin{equation*}
    \cS := \left\{\x\in \{0, \dots, N\}^{|\gS|} \midd \sum_{s\in\gS}x(s)=N \right\}.
\end{equation*}

Thus, $x(s)$ denotes the number of minor sub-MDPs in the $s$-th state. Similarly, for each feasible state-action pair $(s,a) \in \S \times \A$, the mapping  $g: \Sn \times \An \rightarrow \{0, \dots, N\}^{|\S| \times |\A|}$ is defined by $[g(\s,\a)](s, a) := \sum_{n=1}^N \sI\{s_n = s, a_n = a\}$. For a major state $s^0 \in \S^0$ and a count state $\x \in \cS$, the feasible count action set is
\begin{equation}\label{eq:feasible-action-action}
\ifInlineVersion
    \cA(s^0, \x) := \left\{ \u \in \{0, \dots, N\}^{|\S| \times |\A|} \midd
    \begin{aligned}
        & \sum_{a\in\A} u(s,a) = x(s), \forall s \in \S\\
        & \sum_{s\in\S}\sum_{a\in\A} d_k(a \mid s) u(s,a) \le b_k(s^0), \forall k \in [K]
    \end{aligned}
    \right\}.
\else
\begin{aligned}
    & \cA(s^0, \x) := \bigl\{ \u \in \{0, \dots, N\}^{|\S| \times |\A|} \mid \\
    & \hspace{1.5em} \sum_{a\in\A} u(s,a) = x(s), \forall s \in \S, \\
    & \hspace{1.5em} \sum_{s\in\S}\sum_{a\in\A} d_k(a \mid s) u(s,a) \le b_k(s^0), \forall k \in [K] \bigr\}.
\end{aligned}
\fi
\end{equation}

Here, $u{(s, a)}$ indicates the number of minor MDPs at $s$-th state that perform $a$-th action. In particular, for any labelled state $\s$ and any feasible labelled action $\a\in\An(s^0,\s)$, if $f(\s)=\x$, then $g(\s,\a)\in\cA(s^0,\x)$. We can then formulate the count aggregation MDP (Definition \ref{def:count-MDP}).

\begin{definition}[Count Aggregation MDP]\label{def:count-MDP}
The count aggregation MDP $\cM := (\gM^0_\phi, \gM^{(N)}_\phi)$ derived from M2WCMDP $(\gM^0, \gM^{(N)})$ is composed of:
    \begin{enumerate}
        \item major aggregation sub-MDP $\gM^0_\phi$ with elements $(\S^0, \A^0, p_\phi^0, r_\phi^0, \mu^0, \gamma)$ derived from major sub-MDPs $(\S^0, \A^0, p^0, r^0, \mu^0, \gamma)$. \label{def:majorMDP}
        \item minor aggregation sub-MDPs $\gM^{(N)}_\phi$ with elements $(\cS, \cA, \cP, \bR, \cmu, \gamma)$ derived from minor sub-MDPs $(\Sn, \An, \Pn, \Rn, \vmu, \gamma)$.
    \end{enumerate}
\end{definition}

Both representations lead to the same optimization problem as established in \cite{gast2022reoptimization} when the objective is utilitarian. Using the count representation, the mean expected total discounted reward $\bar{V}_0({\vpi})$ for minor sub-MDPs $\Mn$ with a permutation-invariant distribution $\bm{\mu}$ under utilitarian reduction (Theorem \ref{thm:eqv}) is then equivalent to the expected total discounted mean reward $\bar{V}_0({\vpi_\phi})$ for the count aggregation MDP $\cM$ given the policy $\cpi: (\S^0, \cS) \rightarrow \Delta(\A^0, \cA(s^0, \x))$ under aggregation mapping with initial distribution $\bm{\mu}_f \in \Delta(\S^0 \times \Sn_f)$, i.e., $\bV(\vpi) = \frac{1}{N}\sum_{n=1}^N V^{n}_{0}(\vpi) = \bV{(\cpi})$. Similarly, for the major sub-MDP, $V_0^0(\vpi) = V_0^0{(\cpi})$. The objective in Equation (\ref{eq:obj-utilitarian-reduced}) is therefore reformulated as
\begin{equation}\label{eq:obj-utilitarian-count-reduced} 
    \max_{\vpi_\phi}G^{\phi}_{\ut}(\vpi_\phi) := V^0_0(\vpi_\phi) + \lambda \bar{V}_0(\vpi_\phi),
\end{equation}
\ifInlineVersion
where 
\begin{equation*}
    V^0_0(\vpi_\phi):=\mathbb{E}_{\vpi_\phi}\left[\sum_{t=0}^\infty \gamma^t r^0_\phi(s^0_t, \x_t, a^0_t, \u_t)\bigd (s^0_0, \x_0)\sim \bm{\mu}_f \right],
\end{equation*}
and
\begin{equation*}
    \bar{V}_0(\vpi_\phi) := \E_{\cpi}\left[\sum_{t=0}^{\infty} \gamma^t \bar{r}_{\phi}(s^0_t, \x_t, a_t^0, \u_t) \bigd (s^0_0, \x_0) \sim \bm{\mu}_{f} \right].
\end{equation*}
\else
where $V^0_0(\vpi_\phi):=\mathbb{E}_{\vpi_\phi}\left[\sum_{t=0}^\infty \gamma^t r^0_\phi(s^0_t, \x_t, a^0_t, \u_t)\bigd (s^0_0, \x_0)\sim \bm{\mu}_f \right]$, and $\bar{V}_0(\vpi_\phi) := \E_{\cpi}\left[\sum_{t=0}^{\infty} \gamma^t \bar{r}_{\phi}(s^0_t, \x_t, a_t^0, \u_t) \bigd (s^0_0, \x_0) \sim \bm{\mu}_{f} \right]$.
\fi

Here, the utilitarian reward aggregation mapping for minor sub-MDPs is $\bar{r}_{\phi}(s^0, \x, a^0, \u) := \frac{1}{N} \sum_{s \in \S, a \in \A} u(s, a) r(s^0,s,a^0,a)$. Similarly, the reward aggregation mapping for the major reward is $r^0_{\phi}(s^0, \x, a^0, \u) = r^0(s^0, \bar{\s}, a^0, \bar{\a})$ for any arbitrary labelled state-action pair $(\bar{\s}, \bar{\a})$ satisfying $f(\bar{\s}) = \x$ and $g(\bar{\s}, \bar{\a}) = \u$. Definition \ref{def:symmetry-M2WCMDP} ensures that this value is independent of the representative labelled pair. The count transition kernel and initial distribution of $\gM_{\phi}$ are derived in \ref{apx:exact-form-count-M2WCMDP}. 

By Theorem \ref{thm:eqv}, an optimal policy for the count aggregation MDP, after permutation-invariant lifting (see Equation (\ref{eq:lift-to-labelled-pi-policy}) in \ref{apx:dual-lp}), is optimal for the original symmetric $\rho$-M2WCMDP. An exact linear programming (LP) method is provided to solve the utilitarian-reduced count aggregation M2WCMDP in \ref{apx:dual-lp}. However, due to the curse of dimensionality, solving this LP is computationally prohibitive despite the reduction. We thus design a count-proportion-based learning approach in Section \ref{sec:5policy_nn}.

\subsection{Examples of Count Reformulation}\label{sec:3.3count-examples}
{We now detail the count reformulation to the MRP and TDPP applications in Section \ref{sec:3.3-example}. 

\begin{example}[Count Reformulation of the MRP Example \ref{exmp:mrp}]\label{exmp:mrp-count}
For the binary-action, single-resource MRP, let $x(s)$ denote the number of machines in deterioration state $s$, and let $u(s)$ denote the number replaced. The count state and feasible count action are $\cS :=\{\x\in\sN^{S} \mid \sum_{s\in\gS}x(s)=N\}$, and $ \cA(\x) := \{\bm u\in\sN^{S}\mid 0\le u(s)\le x(s),\ \sum_{s\in\gS}u(s)\le b\}.$ \hfill $\square$
\end{example}}

\begin{example}[Count Reformulation of the Taxi Example \ref{exmp:taxi-dispatching}]\label{exmp:taxi-dispatching-count}
For the taxi application, the count action space is $\cS:=\left\{\x\in\sN^{M\times|\gD|}\midd \sum_{i\in\M}\sum_{d\in\gD}x(i,d)=N\right\}$, where $x(i,d)$ is the number of taxis that are en route to zone $i$ and will be available after $d$ time steps. 

Let $\u :=(\u^1, \u^2)$, where $u^1_t(i,j)$ and $u^2_t(i,j)$ denote, respectively, the number of idle taxis relocated and matched from origin zone $i \in \gM$ to target zone $j \in \gM$. For a major state $s^0=(\vq,\vpLag,t)$, the feasible count action set is $\cA(s^0,\x) :=\left\{(\u^1,\u^2) \in\sN^{M\times M} \times\sN^{M\times M}\midd \right.$ $\left. \sum_{j\in\M}\bigl(u^1(i,j)+u^2(i,j)\bigr)=x(i,0), \forall i\in\M; \, \, u^2(i,j)\leq s^0(i,j), \forall i,j\in\M \right\}$. For $d \geq 1$, the only feasible individual action is the wait-in-place action $\aWait$. The corresponding count therefore must satisfy $u((i,d),\aWait)=x(i,d)$ automatically.

With the convention $\x_t(i, D_{\max}+1):= 0$ and $D(j, i) = 0$ if and only if $j =i$, the in-transit counts evolve according to $\x_{t+1}(i, d) := \x_{t}(i, {d+1}) +  {\sum_{j \in \M} \left(u^1_{t}(j, i) + u^2_t(j ,i) \right)\mathbb{I}\{D(j, i) = d\}},$ for all $d \in \gD$ and all $ i \in \gM$.

Under the utilitarian reduction (Theorem \ref{thm:eqv}) and count aggregation $\phi$, the platform's reward is a fraction of all ride rewards $r^0_{\phi}(s^0, \x, \u) := \beta \sum_{i \in \M, j \in \M} \left[c_f+c_r\pLag(i,j)D(i,j)\right] u^2(i,j)$. The mean participant reward is determined by the total successful ride assignments to {locally available taxis} and then normalized by $N$ to get the average per driver reward, i.e., $\bar{r}_{\phi}\left(s^0, \x, \u\right) := \frac{1}{N} \left[(1-\beta) \sum_{i,j \in \M} \left[c_f+c_r\pLag(i,j)D(i,j)\right] u^2(i,j) - c_o \sum_{i,j \in \M} D(i,j) \left(u^1(i,j) + u^2(i,j)\right)\right]$. The objective is formulated by (\ref{eq:obj-utilitarian-count-reduced}). \hfill $\square$
\end{example}

\section{Count-Proportion-Based Deep Reinforcement Learning} \label{sec:5policy_nn}
For large-scale optimization and when the model $(\gM^0, \Mn)$ is unknown, we develop a count-proportion-based deep reinforcement learning (CP-DRL) method for approximately solving the utilitarian-reduced count M2WCMDP in (\ref{eq:obj-utilitarian-count-reduced}). In Section \ref{sec:policy-NN}, we introduce the architecture of our stochastic policy network, which is explicitly designed to search within {a scalable parameterized subclass of} the permutation-invariant policies. The fundamental challenge in this setting is bridging the gap between continuous deep neural network outputs and the combinatorial constraints of the feasible discrete count action space. Therefore, in Section \ref{sec:sampling}, we introduce a parametrized priority-based sampling procedure. This ensures operational feasibility while maintaining sufficient stochasticity for exploration.

\subsection{Stochastic Policy Neural Network} \label{sec:policy-NN}
One key property of the count aggregation MDP is that the dimensions of the state space $\cS$ and the action space $\cA$ are independent of the number $N$ of minor sub-MDPs, although the cardinalities of the feasible count-action sets still depend on $N$. For $\x\in\cS$, we define the count minor state proportion as $\dx := \x/N \in [0, 1]^{|\gS|}$ to further simplify the analysis and eliminate the explicit dependence on $N$. The role of the stochastic policy network is to convert the state $(s_0, \bar{\x})$ to a distribution over $(a^0,\u)$ from which a sample is drawn.  We represent the stochastic policy through a chain-rule decomposition, i.e., $\vpi(a^0,\u\mid s^0,\x)=\vpi^{0}(a^0\mid s^0,\x)\vpi^{(N)}(\u\mid s^0,\x,a^0)$. The conditional relationship between $a^0$ and $\u$ is implicitly represented within the joint neural-network architecture. 

To deal with the exponential size of the support of $\u \in \cA$ and with the challenges associated to the state dependent constraints (\ref{eq:feasible-action-action}), we represent the conditional minor policy $\vpi^{(N)}(\u\mid s^0,\x,a^0)$ using an intermediate parametrization $(\tilde{\vp}, \widetilde{\bm{W}}, \bm{U})$, where $(\tilde{\vp},\widetilde{\bm W},\bm U)$ are the state-dependent network outputs jointly generated with the major-action output $a^0$ and used for a priority-based sampling procedure $\vpi_{\tilde{\vp}, \widetilde{\bm{W}}, \bm{U}}(\u\mid s^0,\x)$ to be described next. The neural networks for producing $a_0$ and $(\tilde{\vp}, \widetilde{\bm{W}}, \bm{U})$ are combined to reduce the number of learned parameters and share a common representation of the state $(s^0,\x)$. {Since count-state proportions alone can erase scale information needed to evaluate resource capacities, the raw count tuple $(s^0,\x)$ is passed explicitly to the final integer sampler. Figure \ref{fig:stochastic-policy-network} presents a diagram representing our proposed implementation.

\begin{figure}\FIGURE
    {\includegraphics[width=0.8\textwidth]{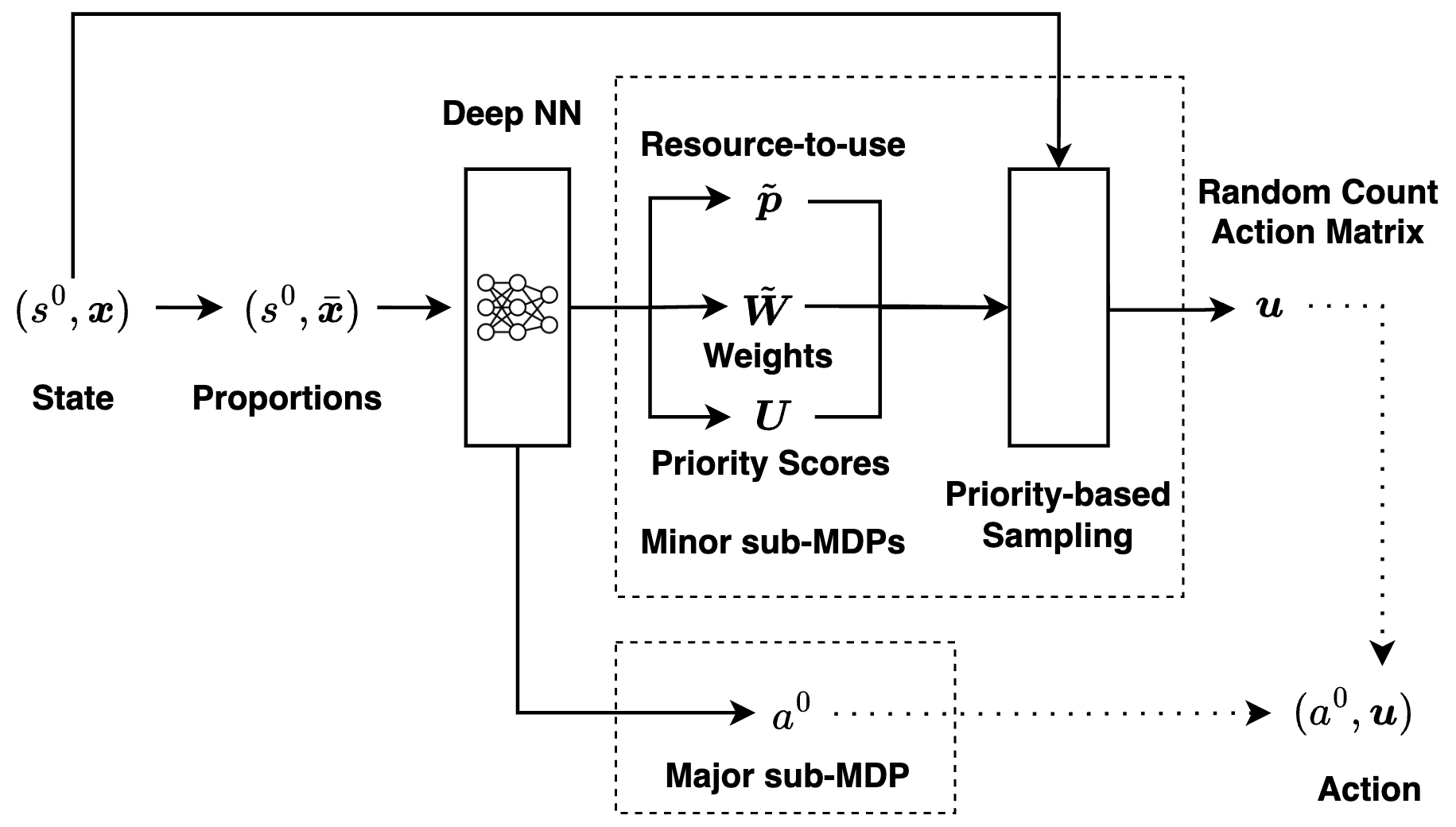}}
    {CP-based Stochastic Policy Neural Network \label{fig:stochastic-policy-network}} 
    {} 
\end{figure}

\subsection{Parameterized Priority-based Sampling Procedure} \label{sec:sampling}
The priority-based sampling procedure $\vpi_{\tilde{\vp}, \widetilde{\bm{W}}, \bm{U}}(\u\mid s^0,\x)$ first employs a parameter {$\tilde{\vp}\in [0, 1]^K$} controlling the proportion of available resources used by the state-action count plan. It then exploits a predefined partition of the action space $\{\A_c\}_{c\in\gC}$ that identifies $|\gC|$ categories of actions, in order to control the {targeted} count of planned actions within each category $c$, for each state $s$. This is done using a {category-allocation weight} matrix $\widetilde{\bm W}\in\Delta(\gC)^{|\S|}$ that {determines integer category budgets $B(s, c) \approx \widetilde{\bm{W}}(s,c)x(s)$} for every $s\in\S$, satisfying $\sum_{c\in \gC} B(s, c) = x(s)$. Finally, a priority score matrix $\mU \in [0, 1]^{|\S|\times|\A|}$ promotes that each state-action pair be sampled proportionally to $U(s,a)$ within each category\endnote{If all eligible priority scores are zero, a uniform value of $1/(|\S|\times |\A|)$ is used.}. Algorithm \ref{alg:sampling} presents the details of the state-action count sampling procedure, ensuring a feasible realization of the count action matrix $\u$.

\begin{algorithm}[htbp]
\caption{Priority-Score based State-Action Count Sampling Procedure ($\u\sim\vpi_{\tilde{\vp}, \widetilde{\bm{W}}, \bm{U}}(\cdot \mid s^0,\x)$)}
\label{alg:sampling}
\begin{algorithmic}
\STATE {\bfseries Input:} Major state $s^0$, count state $\x$
\STATE {\bfseries Parameters:} Resource-to-use proportion $\pb$, category allocation weights $\widetilde{\bm{W}}$, priority matrix $\mU$
\STATE {\bfseries Assumption:} There exists $\aFB\in\A$ such that $d_k(\aFB \mid s) = 0$ for all $s,k$
\STATE \textbf{Initialize}: $\tilde{\vb} \leftarrow \vb(s^0) \cdot \pb$, $\u \leftarrow \bm{0}_{|\S|\times|\A|}$
\STATE \#Identify categories' count budgets closest to fractional targets
\STATE $B(s,c_i)\leftarrow \left[\argmin_{\mathfrak{b}
\in\mathbb{N}^{|\mathcal{C}|}:\sum_{c\in\mathcal{C}} \mathfrak{b}(c)=x(s)} \sum_{c\in\mathcal{C}} |\mathfrak{b}(c)-x(s)\widetilde W(s,c)| \right](c_i)$, $\forall c_i\in\mathcal{C}$, $\forall s\in\gS$
\FOR{$i\in\{1,\dots,|\gC|\}$}
\STATE $\gF \leftarrow \{(s, a)\in\gS\times\gA_{c_i} \mid {B}(s, c_i) = 0\}$ \;\;\#Create set of exhausted state-action pairs in $\gS\times\gA_{c_i}$ 
\WHILE{$|\gF| < |\S|\times |\gA_{c_i}|$}
\STATE Sample $(s,a) \propto \mU(s,a) \sI\{(s,a) \in \gS\times\gA_{c_i} \setminus \gF\}$ \;\;\#Sample unexhausted pair based on priority
    \IF{$d_{k}(a \mid s) \leq \tilde{b}_k$ for all $k$}
    \STATE $B(s, c_i) \leftarrow B(s, c_i) - 1$ \;\;\#Update action count budget
    \STATE $\tilde{b}_k \leftarrow \tilde{b}_k - d_k(a\mid s)$ for all $k$ \;\;\#Update left-over resources
    \STATE $u(s, a) \leftarrow u({s, a}) + 1$ \;\;\#Add $(s,a)$ to state-action count plan
    \STATE $\gF \leftarrow \gF \cup \{(s, a) \in \gS\times\gA_{c_i}\mid B(s,c_i)=0\}$ \;\;\#Update $\gF$
    \ELSE 
    \STATE $\gF \leftarrow \gF \cup \{(s, a)\}$ \;\;\#Add $(s,a)$ to exhausted pairs
    \ENDIF
\ENDWHILE
\ENDFOR
\STATE $u(s, \aFB) \leftarrow u(s,\aFB) + x(s)-\sum_a u(s,a)$ for all $s$\;\;\#Complete plan with free action counts
\STATE {\bf Return:} Count action matrix $\u$
\end{algorithmic}
\end{algorithm}

We next establish two properties of the sampler. \ref{apx:proof-sampler-properties} provides the proof.

\begin{lemma}[Feasibility and deterministic coverage]\label{lemma:sampler-properties}
    Given that there exists a fallback action $\aFB\in\A$ such that $d_k(\aFB \mid s) = 0$ for all $s\in\S$ and all $k\in[K]$, the sampling procedure described in Algorithm \ref{alg:sampling},  has the following properties:
    \begin{enumerate}
        \item $\Prob(\u\in{\cA}(s^0,\x))=1$ for all $s^0\in\S^0$ and $\x\in\cS$.
        \item If $\{\gA_c\}_{c\in\gC}=\gA$, then for any deterministic policy $\cpi^D:\S^0\times{\cS}\rightarrow${$\{0, \dots, N\}^{|\gS| \times |\gA|}$ satisfying $\cpi^D(s^0, \x) \in\cA(s^0,\x)$}, there exists a mapping $h:\S^0\times{\cS}\rightarrow [0, 1]^K\times\Delta(\gC)^{|\S|}\times [0, 1]^{|\S|\times|\A|}$ such that a sample $\u\sim\vpi_{h(s^0,\x)}(\cdot\mid s^0,\x)= \cpi^D(s^0,\x)$ with probability one.
    \end{enumerate}
\end{lemma}

The assumption of a fallback action $\aFB$ ensures that the system always has a valid default operation. Consequently, the sampling procedure strictly guarantees operational feasibility (Property 1) by reverting to $\aFB$, so both the count-conservation constraints and the resource constraints are satisfied for every realization of the sampler.

Property 2 establishes that the sampling parameterization does not exclude any deterministic feasible count policy when each action is represented by a separate category. In this case, the category weights can encode the desired action proportions directly, and the sampler can reproduce any feasible integer count action with probability one. This indicates that, under singleton categories, an optimal deterministic count policy is contained in the sampler-induced policy class.

There is, however, a trade-off in the number of action categories. Using more categories increases representational flexibility, and Property 2 provides exact deterministic representability in the limiting singleton partition. On the other hand, a larger number of categories increases the number of policy-network outputs, the number of sequential sampling stages, and the computational burden of action generation and learning. Coarser partitions exploit application-specific structure, but they generally restrict the determinism of the policy class, because actions within the same category are controlled only through their relative priority scores.

\section{Numerical Results}\label{sec:7numerical}
This section evaluates the proposed CP-DRL framework (Section \ref{sec:5policy_nn}) in two count-reformulated problem settings of Section \ref{sec:3.3count-examples}. In all experiments, CP-DRL is trained to maximize the utilitarian-reduced count objective $G^{\phi}_{\ut}(\vpi_\phi)$ in (\ref{eq:obj-utilitarian-count-reduced}) and evaluated on the original fairness-aware objective $G_\rho(\vpi)$ in (\ref{eq:obj-rho-M2WCMDP}) using a permutation-invariant disaggregation rule $\phi^{-1}$. We first consider the machine replacement problem (see Example \ref{exmp:mrp-count}), where exact $\rho$-optimal solutions can be calculated for small instances. This setting provides a controlled environment for showing the empirical consistency and operational implications of the utilitarian reduction in Theorem \ref{thm:eqv}. The taxi dispatching problem (see Example \ref{exmp:taxi-dispatching-count}) provides a complementary setting to evaluate the full major-minor coordination with endogenous pricing and routing decisions.

The MRP does not include a major sub-MDP and is therefore a classical WCMDP. This makes it a useful diagnostic experiment to isolate the approximation quality of CP-DRL and the computational benefit of the count-proportion representation without confounding effects from major-state dynamics. Among the DRL algorithms considered in preliminary experiments, the \textit{Proximal Policy Optimization} (PPO) algorithm \citep{schulman2017proximal} consistently delivers the most stable and high-quality performance, and is thus used as the main algorithm for our CP-DRL approach.

\subsection{The Machine Replacement Problem} \label{sec:mrp-expt}
The MRP provides a scalable framework for evaluating the CP-DRL approach as problem size and complexity increase. We focus on a single resource ($K = 1$) and binary action ($A = 2$) for each machine, which places the problem within the restless multi-armed bandits (RMABs) subclass of WCMDPs \citep{whittle1988restless, zhang2022near}. For selected indexable instances, this structure enables direct comparison with the Whittle index policy, a strong structure-exploiting benchmark for RMABs.

\subsubsection{Implementation Details}  For each state-action pair, resource consumption $d(a \mid s \,)$ is 1 for replacements and 0 for operations, with up to $b$ replacements per time step. We convert the normalized cost function $c(s, a) \in [0, 1]$ into rewards by applying the positive shift affine transformation $r(s,a):=1-c(s,a)$. Machines degrade if not replaced and remain in the most deteriorated state $S$ until replaced. We choose operational and replacement costs across two presets to capture different cost structures, characterized by different replacement cost constant coefficients (RCCCs): \textit{i) Exponential-RCCC} and \textit{{ii)} Quadratic-RCCC}. Refer to \ref{apx:instance-generation-mrp} for cost structures and transition probabilities.

\paragraph{Objectives and Evaluation Protocol} We use the GGF in Example \ref{exmp:ggf} as the primary fairness criterion for testing the symmetry-based reduction. GGF weights decay exponentially with a factor of 2, defined as $w_n = 1/2^n$, and normalized to sum to 1. The goal is to find a fair policy that maximizes the GGF score over the expected total discounted mean rewards under the count aggregation MDP. We use a uniform distribution $\vmu$ over $\Sn$ and set the discount factor $\gamma = 0.95$. We use Monte Carlo simulations to evaluate policies over $N_{MC}$ trajectories truncated at time length $T$. We choose $N_{MC}$ = 1,000 and $T$ = 300 across all experiments.

\paragraph{CP-DRL and Benchmarks} We evaluate two CP-DRL training regimes. The single-task CP-DRL is trained separately for each machine count $N$. CP-DRL(MT) is a multi-task (MT) extension trained jointly over $N\in\{2,3,4,5\}$ with $S=3$, randomly switching the configuration at the end of each episode over 2,000 training episodes. Both variants use the same count-proportion representation and priority-based sampler described in Section \ref{sec:5policy_nn}. Hyperparameters for the CP-DRL algorithm are in \ref{apx:hyperparameter-mrp}. We compare CP-DRL against seven benchmarks, including optimal solutions (OPT) from the utilitarian-reduced dual LP model (\ref{eq:utilitarian-reduced-lp}) for small instances solved with Gurobi 10.0.3, the Whittle index policy (WIP) for RMABs, and a random (RDM) agent that selects actions randomly at each time step and averages the results over 10 independent runs. {Additionally,  we implemented a simple DRL baseline, Vanilla-DRL (V-DRL), with a utilitarian objective. The stochastic policy network employs a fully connected neural network that maps the vector $\s$ to a $N$-dimensional probability vector. We also implemented two heuristics to complement the random agent approach. The oldest-first (OFT) approach selects the machine in the worst state, while the myopic (MYP) selects the machine that maximizes immediate reward. We finally implemented an equal-resources (EQR) approach based on \cite{li2022efficient}, which imposes that each machine be replaced once every $N$ steps to ensure an equal distribution of resources.}

\subsubsection{Results} We designed a series of experiments to test the GGF-optimality and fairness of our CP-DRL algorithm. Additional experiments on scalability and efficiency are provided in \ref{apx:additional-expt-mrp}.

\begin{table}[!htb]
\TABLE{GGF Scores (Exponential-RCCC) \label{tab:performance1}}
{
    \begin{tabular}{ccccccccc>{\centering\arraybackslash}c}
    \hline
    $N$ & OPT & WIP & CP-DRL & CP-DRL(MT)  & V-DRL& OFT & MYP & EQR & RDM \\ \hline
    2 & 14.19 & 14.07 & \textbf{14.12} $\pm$ 0.01 & 14.11 $\pm$ 0.01  & 13.56 
    $\pm$ 0.00&5.84& 12.59&10.05& 9.67 \\
    3 & 14.08 & 13.75 & \textbf{13.95} $\pm$ 0.02 & 13.89 $\pm$ 0.14  &  13.39 $\pm$ 0.00&7.92& 12.32&	11.67& 10.13 \\
    4 & 13.94 & 13.27 & \textbf{13.64} $\pm$ 0.05 & 13.59 $\pm$ 0.10  &  13.04 $\pm$ 0.01&9.02& 12.86&12.03& 9.74 \\
    5 & 13.77 & 12.47 & 12.96 $\pm$ 0.01 & \textbf{13.28} $\pm$ 0.03  &  12.83 $\pm$ 0.00&10.01& 12.08&11.87& 8.95 \\ \hline
    \end{tabular}
}
{\textit{Note.} CP-DRL and CP-DRL(MT) entries are mean $\pm$ standard deviation over five random seeds.}
\end{table}

\begin{table}[!htb]
\TABLE{GGF Scores (Quadratic-RCCC) \label{tab:performance2}}
{
    \begin{tabular}{ccccccccc>{\centering\arraybackslash}c}
    \hline
    $N$ & OPT & WIP & CP-DRL & CP-DRL(MT)  & V-DRL& OFT & MYP & EQR & RDM \\ \hline
    2 & 16.17 & \textbf{16.17} & 16.14 $\pm$ 0.00 & 16.14 $\pm$ 0.00  & 15.36 
    $\pm$ 0.00& 3.11& 6.61&9.73& 10.15 \\
    3 & 16.10 & \textbf{16.09} & 16.05 $\pm$ 0.00 & 16.05 $\pm$ 0.00  & 15.17 $\pm$ 0.01& 6.16& 6.63&12.73& 11.83 \\
    4 & 16.01 & \textbf{16.01} & 15.94 $\pm$ 0.00 & 15.94 $\pm$ 0.00  & 15.01 $\pm$ 0.00	& 7.92& 6.85&14.02& 12.17 \\
    5 & 15.91 & 15.86 & \textbf{15.87} $\pm$ 0.02 & 15.86 $\pm$ 0.02  & 14.73 $\pm$ 0.00& 9.25& 6.70&14.64& 11.98 \\ \hline
    \end{tabular}
}
{\textit{Note.} CP-DRL and CP-DRL(MT) entries are mean $\pm$ standard deviation over five random seeds.}
\end{table}

Tables \ref{tab:performance1} and \ref{tab:performance2} report the GGF scores for the Exponential-RCCC and Quadratic-RCCC instances, respectively. WIP and RDM are evaluated with 1,000 Monte Carlo runs, and their standard deviations are negligible and omitted. Bold font indicating the best GGF scores at each row excluding optimal values. As shown in Table \ref{tab:performance1}, CP-DRL(MT) consistently achieves scores very close to the OPT values as the number of machines increases from 2 to 4. For the 5-machine case, CP-DRL(MT) shows slightly better performance than the single-task CP-DRL. In Table \ref{tab:performance2}, the single- and multi-task CP-DRL agents show slight variations in performance across different machine numbers. For $N=5$, CP-DRL achieves the best GGF score, slightly outperforming WIP. {In both cases, the CP-DRL approach outperforms Vanilla-DRL, the three heuristic methods, and the random benchmark.}

\subsection{Taxi Dispatching and Dynamic Pricing}\label{sec:taxi-expt}
We next evaluate the proposed CP-DRL methods in the taxi dispatching application. The New York City (NYC) calibrated dataset is used to evaluate real-world applicability and scalability through a calibrated simulator. In the first setup (Section \ref{sec:nyc}), we focus on performance evaluation by comparing fairness objective values and operational metrics on a 15-node case. In the second setup (Section \ref{sec:mechanism}), we conduct mechanism analysis by inspecting the adjusted price distribution and a pricing trajectory. In the final setup (Section \ref{sec:fairness-exec}), we validate the fairness execution on a smaller 4-zone Midtown network by comparing disaggregation rules and different fairness measures. Additional details can be found in \ref{apx:details}.

\subsubsection{Experimental Setup} The simulation environment is built on the NYC network and Yellow Taxi trip records provided by \cite{nyc_tripdata}. The environment operates in discrete time steps of length $\Delta t = 3$ minutes, which balances responsiveness and computational tractability, similar to prior mobility system studies \citep{sun2022optimizing}. The spatial domain is discretized into a graph where nodes represent taxi zones and edges represent connectivity based on road network distances. All operational and economic parameters, including per-unit ride reward, per-unit travel cost, and price-response coefficients, are calibrated to the same setting to ensure cross-experiment comparability (see \ref{apx:nyc-data} for full details of calibration on base arrival rates and other parameters).

\paragraph{Pricing Mechanism} We consider four levels of pricing granularity. OD pricing assigns a distinct multiplier to each OD pair per time step, which offers maximal theoretical flexibility but may complicate learning. Origin-only (O) uses one multiplier per-origin, reducing the decision space while still accounting for spatial demand heterogeneity. Uniform (U) pricing implements a single multiplier per time step, focusing mainly on temporal fluctuations. Static (S) pricing maintains a constant price throughout the episode. The optimal multiplier for this static baseline is determined by an exhaustive search and serves as the primary benchmark to quantify the performance gains of dynamic intervention strategies. {To reduce the price-output dimension and regularize spatial price variation, low-rank parameterizations are used, in which two factor matrices $\mL^{\text{O}},\mL^{\text{D}}\in[0,1]^{M\times r_p}$ are mapped to price multipliers as specified in \ref{apx:low-rank-pricing}, and $r_p$ is the chosen pricing rank.}

\paragraph{Objectives and Metrics} All policies are trained with $\lambda=1$. We then evaluate the resulting policies from both welfare and operational perspectives. Participant welfare is assessed using the GGF (with the same weight configuration as in Section \ref{sec:mrp-expt}), the $\alpha$-fairness, and the utilitarian mean. We also report metrics to characterize the mechanisms underlying policy performance. Occupancy measures productive fleet utilization, admitted orders measure passenger service, and the average and time-varying OD multipliers are used to interpret the learned pricing policy.

\paragraph{CP-DRL and Benchmarks} We use PPO and a standard Graph Neural Network (GNN) as the feature extractor to capture the underlying spatial dependencies (see structural details in \ref{apx:gnn}), compared with the following benchmarks: 1) \textit{count-proportional assignment linear program} (CP-ALP) uses a myopic per-step optimizer to replace the trained priority-based sampling procedure to generate count actions. 2) \textit{Random} (RDM) samples actions uniformly from the feasible parameter space, which isolates the contribution of structured decision-making. 3) A \textit{lowest income first} (LIF) heuristic gives assignment priority to drivers with lower accumulated income. LIF uses the fixed price multiplier 5.10 reported in Table \ref{tab:c_r_od_pricing_mechanism} and does not optimize price. LIF is a local execution rule and is included to test the performance of income-priority matching alone compared with the welfare achieved by coordinated pricing and dispatch. 4) We additionally report the \textit{multilayer perceptron} (MLP) benchmark with fully connected layers to show how the spatial feature GNN improves performance. For every fleet-size and pricing-granularity configuration, CP-DRL and MLP are trained separately over 10 independent seeds. CP-ALP, RDM, and LIF are evaluated over 10 independent simulation seeds. All methods are evaluated with a demand synchronization to CP-DRL over the corresponding 10 seeds with 50 episodes per seed. Complete benchmark definitions and implementation details are provided in \ref{apx:benchmark}.

\subsubsection{Performance on the 15-Zone Configuration} \label{sec:nyc} We first evaluate CP-DRL on a 15-zone NYC-calibrated network under OD pricing. The fleet size varies over $\{60,120,180,240,300\}$, which changes the degree of supply scarcity while keeping the calibrated base demand fixed. The key managerial question is whether coordinated dynamic pricing and dispatch can simultaneously improve driver-side welfare, passenger matching, and online decision speed as fleet size changes.

\paragraph{Fairness-Aware Performance} We first present the main performance on a 15-zone network (see details in \ref{apx:graph}) under OD pricing in Table \ref{tab:c_r_objective_by_ggf}, with $\rho$ set to GGF and $\lambda=1$. The number of taxis governs the scarcity of supply. CP-DRL obtains the highest mean value at every fleet size. {Relative to CP-ALP, its mean objective is higher by 18.9\% at 60 taxis, 18.9\% at 120 taxis, 25.7\% at 180 taxis, 23.5\% at 240 taxis, and 19.7\% at 300 taxis.} Since CP-ALP also learns pricing, this gap should not be attributed merely to the existence of a learned price policy. Instead, it indicates that myopic per-step dispatch optimization is insufficient when current assignment decisions shape future taxi availability, future matching opportunities, and the distribution of driver rewards.

\begin{table}[!ht]
\TABLE{Fairness-Aware Objective Values (OD Pricing)\label{tab:c_r_objective_by_ggf}}{
    \begin{tabular}{cccccc}
    \hline
    \# Taxis & CP-DRL & CP-ALP & RDM & LIF & MLP \\
    \hline
    60 & $420.58 \pm 3.00$ & $353.83 \pm 1.92$ & $65.79 \pm 1.83$ & $259.62 \pm 2.51$ & $384.74 \pm 9.02$ \\
    120 & $544.12 \pm 6.49$ & $457.53 \pm 4.26$ & $30.80 \pm 2.76$ & $291.80 \pm 3.85$ & $458.34 \pm 15.66$ \\
    180 & $597.00 \pm 9.34$ & $474.98 \pm 2.37$ & $-4.08 \pm 2.37$ & $307.78 \pm 3.10$ & $469.51 \pm 26.41$ \\
    240 & $606.48 \pm 10.43$ & $490.99 \pm 3.32$ & $-39.18 \pm 2.81$ & $314.78 \pm 3.17$ & $475.92 \pm 40.12$ \\
    300 & $602.08 \pm 11.11$ & $503.14 \pm 2.76$ & $-74.00 \pm 2.90$ & $317.84 \pm 2.31$ & $457.67 \pm 41.47$ \\
    \hline
    \end{tabular}
}{}
\end{table}

LIF remains well below CP-DRL despite explicitly favoring lower-income drivers, showing that fairness-motivated assignment at the execution layer cannot compensate for uncoordinated pricing and future fleet positioning. RDM deteriorates rapidly as the fleet grows, confirming that feasibility alone does not generate effective coordination. MLP is competitive for the smaller instances but becomes increasingly variable as fleet size grows.

\paragraph{Operational Performance} Table \ref{tab:c_r_od_pricing_mechanism} links the welfare differences to service and fleet utilization. CP-DRL sustains substantially higher occupancy and admits many more orders than CP-ALP, RDM, and LIF. The gap of admitted orders relative to CP-ALP grows with fleet size, where CP-ALP admits 14.35\% fewer orders than CP-DRL at 60 taxis, and the gap reaches 56.51\% at 300 taxis. At the same time, CP-DRL lowers its average multiplier as supply expands, from 3.04 to 1.38. The combined pattern suggests that CP-DRL does not obtain its welfare gain by suppressing demand through high prices, but uses additional fleet capacity more productively while relaxing prices as supply scarcity falls.

\begin{table}[!ht]
\TABLE{Operational Performance Under OD Pricing \label{tab:c_r_od_pricing_mechanism}}{
    \begin{tabular}{cccccccc}
    \hline
    \multirow{2}{*}{\# Taxis} & \multirow{2}{*}{Method} & \multicolumn{2}{c}{Occupancy Rate} & \multicolumn{2}{c}{Admitted Orders} & \multicolumn{2}{c}{Average Price} \\
    \cline{3-8}
    & & Mean & Gap vs. DRL & Mean & Gap vs. DRL & Mean & Gap vs. DRL \\
    \hline
    \multirow{5}{*}{60} & CP-DRL & $0.83 \pm 0.01$ & --- & $12.61 \pm 0.22$ & --- & $3.04 \pm 0.04$ & --- \\
     & CP-ALP & $0.75 \pm 0.00$ & $-9.64\%$ & $10.80 \pm 0.02$ & $-14.35\%$ & $3.06 \pm 0.00$ & $+0.66\%$ \\
     & RDM & $0.19 \pm 0.00$ & $-77.11\%$ & $2.80 \pm 0.03$ & $-77.80\%$ & $4.93 \pm 0.01$ & $+62.17\%$ \\
     & LIF & $0.36 \pm 0.00$ & $-56.63\%$ & $5.69 \pm 0.05$ & $-54.88\%$ & $5.10 \pm 0.00$ & $+67.76\%$ \\
     & MLP & $0.88 \pm 0.01$ & $+6.02\%$ & $12.96 \pm 0.13$ & $+2.78\%$ & $2.53 \pm 0.09$ & $-16.78\%$ \\
    \hline
    \multirow{5}{*}{120} & CP-DRL & $0.81 \pm 0.02$ & --- & $24.75 \pm 0.48$ & --- & $2.31 \pm 0.09$ & --- \\
     & CP-ALP & $0.61 \pm 0.00$ & $-24.69\%$ & $18.68 \pm 0.03$ & $-24.53\%$ & $2.99 \pm 0.00$ & $+29.44\%$ \\
     & RDM & $0.10 \pm 0.00$ & $-87.65\%$ & $3.16 \pm 0.05$ & $-87.23\%$ & $4.93 \pm 0.00$ & $+113.42\%$ \\
     & LIF & $0.22 \pm 0.00$ & $-72.84\%$ & $7.09 \pm 0.03$ & $-71.35\%$ & $5.10 \pm 0.00$ & $+120.78\%$ \\
     & MLP & $0.88 \pm 0.02$ & $+8.64\%$ & $25.58 \pm 0.45$ & $+3.35\%$ & $1.67 \pm 0.11$ & $-27.71\%$ \\
    \hline
    \multirow{5}{*}{180} & CP-DRL & $0.76 \pm 0.01$ & --- & $36.02 \pm 0.86$ & --- & $1.89 \pm 0.06$ & --- \\
     & CP-ALP & $0.45 \pm 0.00$ & $-40.79\%$ & $20.87 \pm 0.05$ & $-42.06\%$ & $2.97 \pm 0.00$ & $+57.14\%$ \\
     & RDM & $0.07 \pm 0.00$ & $-90.79\%$ & $3.28 \pm 0.04$ & $-90.89\%$ & $4.92 \pm 0.00$ & $+160.32\%$ \\
     & LIF & $0.16 \pm 0.00$ & $-78.95\%$ & $7.79 \pm 0.05$ & $-78.37\%$ & $5.10 \pm 0.00$ & $+169.84\%$ \\
     & MLP & $0.86 \pm 0.02$ & $+13.16\%$ & $37.55 \pm 0.80$ & $+4.25\%$ & $1.21 \pm 0.09$ & $-35.98\%$ \\
    \hline
    \multirow{5}{*}{240} & CP-DRL & $0.72 \pm 0.02$ & --- & $46.01 \pm 1.13$ & --- & $1.50 \pm 0.05$ & --- \\
     & CP-ALP & $0.35 \pm 0.00$ & $-51.39\%$ & $21.88 \pm 0.05$ & $-52.45\%$ & $2.96 \pm 0.00$ & $+97.33\%$ \\
     & RDM & $0.05 \pm 0.00$ & $-93.06\%$ & $3.33 \pm 0.03$ & $-92.76\%$ & $4.92 \pm 0.00$ & $+228.00\%$ \\
     & LIF & $0.13 \pm 0.00$ & $-81.94\%$ & $8.22 \pm 0.04$ & $-82.13\%$ & $5.10 \pm 0.00$ & $+240.00\%$ \\
     & MLP & $0.80 \pm 0.04$ & $+11.11\%$ & $48.01 \pm 1.79$ & $+4.35\%$ & $1.02 \pm 0.17$ & $-32.00\%$ \\
    \hline
    \multirow{5}{*}{300} & CP-DRL & $0.64 \pm 0.03$ & --- & $51.83 \pm 2.09$ & --- & $1.38 \pm 0.07$ & --- \\
     & CP-ALP & $0.28 \pm 0.00$ & $-56.25\%$ & $22.54 \pm 0.06$ & $-56.51\%$ & $2.95 \pm 0.00$ & $+113.77\%$ \\
     & RDM & $0.04 \pm 0.00$ & $-93.75\%$ & $3.39 \pm 0.02$ & $-93.46\%$ & $4.92 \pm 0.00$ & $+256.52\%$ \\
     & LIF & $0.10 \pm 0.00$ & $-84.38\%$ & $8.49 \pm 0.05$ & $-83.62\%$ & $5.10 \pm 0.00$ & $+269.57\%$ \\
     & MLP & $0.78 \pm 0.04$ & $+21.88\%$ & $59.20 \pm 2.09$ & $+14.22\%$ & $0.76 \pm 0.14$ & $-44.93\%$ \\
    \hline
    \end{tabular}}
{\textit{Note.} Gap vs. DRL is 100({Method}-{DRL})/{DRL}\%. A negative value gap indicates a lower value than CP-DRL, and vice versa.  For RDM, average price is the sample mean of randomized multipliers. LIF uses the fixed multiplier 5.10.}
\end{table}

\paragraph{Computational Performance} Table \ref{tab:computational_time} measures online count-action generation during evaluation. For CP-DRL and MLP, the recorded time covers the the count action sampling procedure, whereas CP-ALP solves its per-period LP and RDM/LIF construct actions using their respective random or heuristic rules. Under OD pricing, CP-DRL's mean action-generation time increases linearly with fleet size, from approximately $2.9$ milliseconds at 60 taxis to $6.7$ milliseconds at 300 taxis, while CP-ALP requires approximately 26 to 28 milliseconds.

\begin{table}[!ht]
\TABLE{Computational Time on Count-Action Generation\label{tab:computational_time}}{
    \begin{tabular}{ccccccc}
    \hline
    \# Taxis & Granularity & CP-DRL & CP-ALP & RDM & LIF & MLP \\
    \hline
    \multirow{4}{*}{60} & OD & $2.92 \pm 0.12$ & $26.50 \pm 0.92$ & $2.97 \pm 0.12$ & $0.39 \pm 0.02$ & $2.44 \pm 0.10$ \\
     & O & $2.97 \pm 0.13$ & $26.44 \pm 1.09$ & $2.98 \pm 0.12$ & $0.42 \pm 0.03$ & $2.65 \pm 0.07$ \\
     & U & $3.38 \pm 1.26$ & $62.40 \pm 10.47$ & $3.20 \pm 0.92$ & $0.45 \pm 0.02$ & $2.94 \pm 0.16$ \\
    \hline
    \multirow{4}{*}{120} & OD & $3.56 \pm 0.08$ & $27.41 \pm 0.91$ & $3.84 \pm 0.17$ & $0.59 \pm 0.02$ & $3.03 \pm 0.19$ \\
     & O & $3.88 \pm 0.16$ & $27.08 \pm 0.88$ & $3.87 \pm 0.12$ & $0.63 \pm 0.03$ & $3.23 \pm 0.23$ \\
     & U & $3.78 \pm 0.10$ & $43.66 \pm 1.26$ & $3.56 \pm 0.06$ & $0.67 \pm 0.03$ & $3.54 \pm 0.19$ \\
    \hline
    \multirow{4}{*}{180} & OD & $4.19 \pm 0.21$ & $26.43 \pm 0.78$ & $4.36 \pm 0.11$ & $0.77 \pm 0.02$ & $3.27 \pm 0.15$ \\
     & O & $4.23 \pm 0.20$ & $26.22 \pm 0.64$ & $4.15 \pm 0.15$ & $0.85 \pm 0.03$ & $3.93 \pm 0.35$ \\
     & U & $4.09 \pm 0.05$ & $33.27 \pm 0.74$ & $4.09 \pm 0.09$ & $0.87 \pm 0.02$ & $4.11 \pm 0.20$ \\
    \hline
    \multirow{4}{*}{240} & OD & $4.85 \pm 0.42$ & $26.79 \pm 3.30$ & $4.86 \pm 0.35$ & $0.94 \pm 0.02$ & $3.73 \pm 0.31$ \\
     & O & $5.57 \pm 0.44$ & $26.23 \pm 1.16$ & $4.90 \pm 0.27$ & $1.00 \pm 0.02$ & $4.60 \pm 0.35$ \\
     & U & $5.77 \pm 0.32$ & $30.39 \pm 1.26$ & $4.63 \pm 0.20$ & $1.07 \pm 0.04$ & $5.65 \pm 0.36$ \\
    \hline
    \multirow{4}{*}{300} & OD & $6.67 \pm 0.72$ & $28.49 \pm 2.67$ & $5.85 \pm 0.50$ & $1.11 \pm 0.02$ & $4.63 \pm 0.59$ \\
     & O & $6.60 \pm 0.49$ & $26.51 \pm 1.70$ & $5.48 \pm 0.17$ & $1.18 \pm 0.02$ & $5.40 \pm 0.40$ \\
     & U & $8.56 \pm 0.77$ & $31.31 \pm 0.94$ & $5.73 \pm 0.11$ & $1.25 \pm 0.03$ & $7.40 \pm 0.87$ \\
    \hline
    \end{tabular}}
{\textit{Note.} Computation time is measured in milliseconds during policy evaluation.}
\end{table}

\subsubsection{Pricing Mechanism Under Non-Stationary Demand} \label{sec:mechanism} The calibrated base arrival rates vary over time, and thus not directly comparable across decision steps. We next report adjusted price multipliers $\psi^{-1}\,\left(\theta^{\mathrm{base}}_{t}(i,j)\psi(p_t(i,j))\right)$ according to the base demands to show how the major pricing decision changes with supply scarcity.

\paragraph{Price Adjustment to Supply Scarcity} Figure \ref{fig:pricing} reports the uniform (U) pricing policy. The adjusted learned price distributions provide a demand-equivalent comparison. As fleet size increases from 60 to 300, the CP-DRL's mean price multipliers are lowered to admit additional demands and allow the platform to use the expanded fleet. The optimized static multiplier changes across fleet sizes in the same direction as indicated by the dashed line.

\begin{figure}\FIGURE
    {\includegraphics[width=\textwidth]{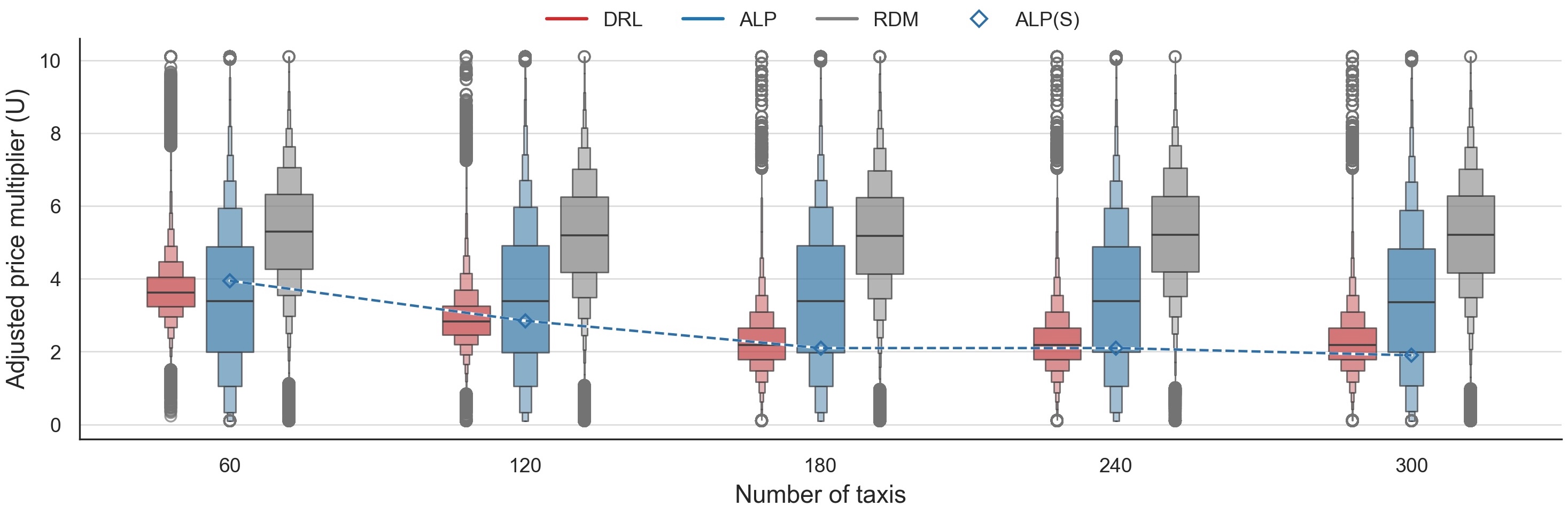}}
    {(Color online) Distribution of Adjusted Learned Uniform Price Multipliers
    \label{fig:pricing}} 
    {The letter-value (boxen) plots show the empirical distribution of adjusted price multipliers. The center line denotes the median. The nested boxes represent successively wider quantile intervals, from the interquartile range (25th and 75th percentiles) to the distribution tails. The dashed reference line denotes the optimized static multiplier for the corresponding fleet size.} 
\end{figure}

\begin{figure}\FIGURE
    {\includegraphics[width=0.96\textwidth]{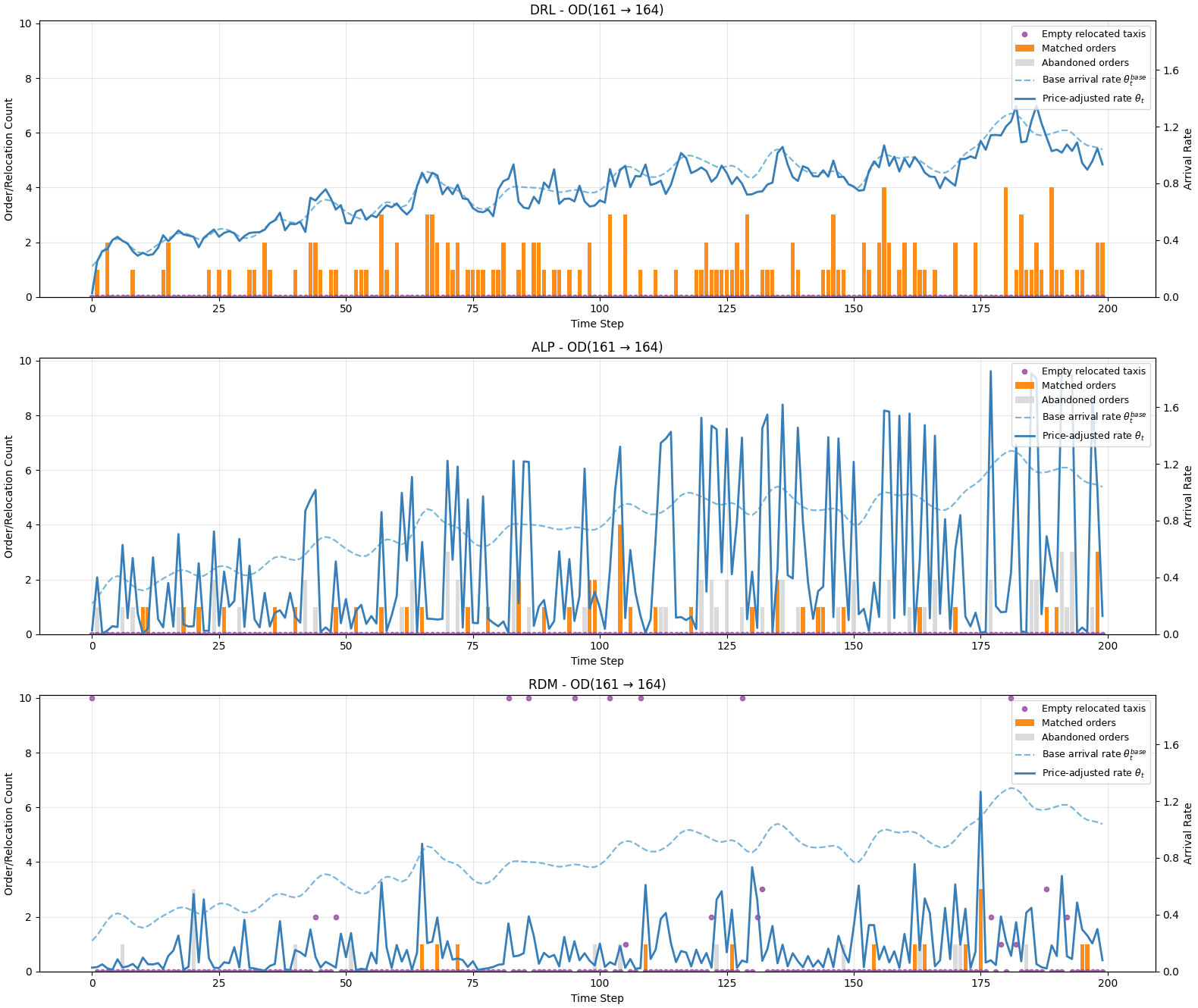}}
    {(Color online) Adjusted Uniform Pricing and Dispatching on a Representative OD Pair \label{fig:pricing-trajectory}} 
    {Lines report price multipliers, bars report matched and abandoned orders, and markers report empty relocations over time.} 
\end{figure}

\paragraph{Joint Evolution of Pricing and Matching} Further, in Figure \ref{fig:pricing-trajectory}, we show a representative OD trajectory and how the aggregate patterns arise. CP-DRL maintains comparatively stable prices (blue lines) with more matched orders on the selected OD pair (orange bars) and less abandoned orders (grey bars). CP-ALP changes prices more sharply and is accompanied by more abandoned demand, whereas RDM combines highly variable prices with few matches and more taxis relocated to other places while there are demands (purple dots). This indicates that CP-DRL's advantage comes from effectively coordinating pricing with dispatch over time.

\subsubsection{Fairness Comparison} \label{sec:fairness-exec} The previous experiments evaluate count policies at the aggregate level. However, a count policy specifies how many taxis take each action, not which labelled taxis receive those actions. This distinction is central for fairness. If a count action is always assigned to the same subset of taxi labels, individual rewards may become systematically unequal.

In this section, we focus on small 4-Node Midtown area (see details in \ref{apx:graph}) with varying fleet size $N \in \{10, 20, 30, \dots, 60\}$ to compare the permutation-invariant randomized disaggregation used by CP-DRL with a biased (B) disaggregation model CP-DRL-B that repeatedly selects taxis according to the lowest fixed indices.

As shown in Figure \ref{fig:fairness-by-execution}, the two execution rules can produce similar mean-objective values while generating a clear gap in the GGF value. A permutation-invariant disaggregation is therefore an operational requirement for transferring the fairness properties of the count policy to individual-level execution.

\begin{figure}\FIGURE
    {\includegraphics[width=\textwidth]{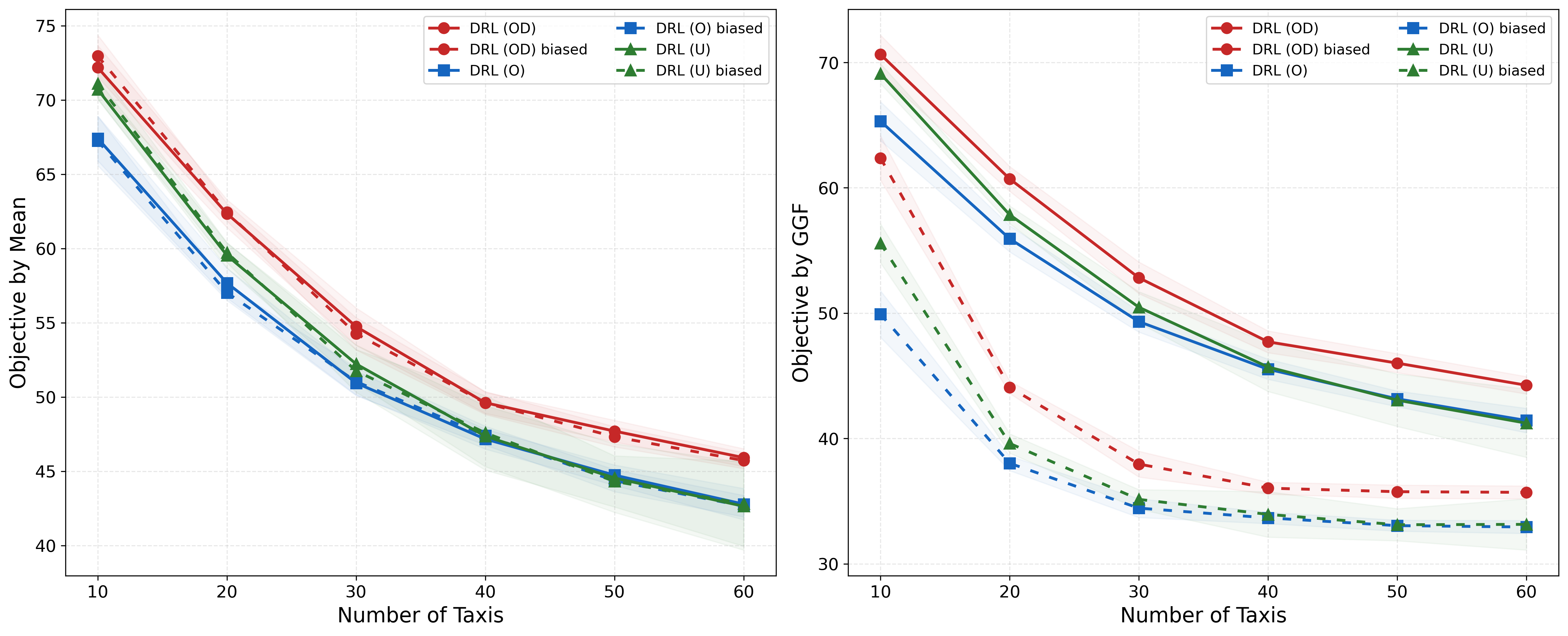}}
    {(Color online) Effect of Participant-Level Disaggregation
    \label{fig:fairness-by-execution}} 
    {Objectives are evaluated by the utilitarian mean and GGF under permutation-invariant and biased execution rules across the OD/O/U pricing granularity.  The same learned count policy is executed using either the permutation-invariant randomized disaggregation rule or the biased low-index first rule, isolating participant-level execution from aggregate count decisions.} 
\end{figure}

Table \ref{tab:nyc_objectives} reports mean, GGF, and $\alpha$-fairness evaluations with $\alpha=1$. The gap between mean and GGF remains small under permutation-invariant disaggregation, especially relative to the biased execution experiment. The table also shows that dynamic OD pricing generally outperforms origin-only and uniform pricing under CP-DRL, whereas static pricing can remain competitive in smaller cases because it avoids learning a high-dimensional policy space while demand variation is limited. This indicates that OD pricing has a better performance when the platform has sufficient data and computational capacity to learn spatial-temporal controls.

\begin{table}[!htb]
\TABLE{Evaluation Under Alternative Fairness Measures \label{tab:nyc_objectives}}{
    \begin{tabular}{cccccccc}
    \hline
    \multirow{2}{*}{\# taxis} & \multirow{2}{*}{Granularity} & \multicolumn{2}{c}{By Mean} & \multicolumn{2}{c}{By GGF} & \multicolumn{2}{c}{By $\alpha$-fairness} \\
    \cline{3-8}
    & & CP-DRL & CP-ALP & CP-DRL & CP-ALP & CP-DRL & CP-ALP \\
    \hline
    \multirow{4}{*}{10} & OD & $72.18 \pm 1.48$ & $52.45 \pm 1.02$ & $70.64 \pm 1.54$ & $50.54 \pm 1.05$ & $72.16 \pm 1.48$ & $52.41 \pm 1.02$ \\
     & O & $67.40 \pm 1.51$ & $45.38 \pm 0.64$ & $65.31 \pm 1.55$ & $43.59 \pm 0.87$ & $67.37 \pm 1.51$ & $45.34 \pm 0.65$ \\
     & U & $70.70 \pm 0.62$ & $31.22 \pm 0.92$ & $69.10 \pm 0.80$ & $29.56 \pm 1.14$ & $70.69 \pm 0.63$ & $31.17 \pm 0.93$ \\
     & S & $71.42 \pm 1.09$ & $68.40 \pm 0.61$ & $69.40 \pm 1.00$ & $66.80 \pm 0.69$ & $71.39 \pm 1.09$ & $68.38 \pm 0.61$ \\
    \hline
    \multirow{4}{*}{20} & OD & $62.34 \pm 0.94$ & $41.17 \pm 0.72$ & $60.71 \pm 0.99$ & $38.99 \pm 0.91$ & $62.33 \pm 0.94$ & $41.12 \pm 0.73$ \\
     & O & $57.67 \pm 1.00$ & $36.38 \pm 0.92$ & $55.93 \pm 1.05$ & $33.92 \pm 0.98$ & $57.65 \pm 1.00$ & $36.31 \pm 0.92$ \\
     & U & $59.53 \pm 0.85$ & $25.21 \pm 0.65$ & $57.85 \pm 0.70$ & $23.55 \pm 0.82$ & $59.51 \pm 0.84$ & $25.15 \pm 0.65$ \\
     & S & $59.89 \pm 0.64$ & $57.98 \pm 0.58$ & $58.35 \pm 0.60$ & $56.28 \pm 0.69$ & $59.87 \pm 0.64$ & $57.95 \pm 0.58$ \\
    \hline
    \multirow{4}{*}{30} & OD & $54.74 \pm 1.26$ & $35.87 \pm 0.58$ & $52.82 \pm 1.26$ & $33.49 \pm 0.73$ & $54.71 \pm 1.26$ & $35.80 \pm 0.58$ \\
     & O & $50.94 \pm 0.80$ & $31.47 \pm 0.63$ & $49.32 \pm 0.83$ & $29.28 \pm 0.67$ & $50.92 \pm 0.80$ & $31.39 \pm 0.63$ \\
     & U & $52.22 \pm 1.28$ & $21.53 \pm 0.56$ & $50.47 \pm 1.25$ & $19.93 \pm 0.52$ & $52.19 \pm 1.28$ & $21.47 \pm 0.56$ \\
     & S & $52.68 \pm 0.53$ & $51.90 \pm 0.54$ & $50.99 \pm 0.58$ & $50.02 \pm 0.44$ & $52.65 \pm 0.53$ & $51.88 \pm 0.54$ \\
    \hline
    \multirow{4}{*}{40} & OD & $49.61 \pm 0.71$ & $32.46 \pm 0.70$ & $47.71 \pm 0.87$ & $30.42 \pm 0.91$ & $49.58 \pm 0.71$ & $32.40 \pm 0.71$ \\
     & O & $47.16 \pm 0.67$ & $29.09 \pm 0.48$ & $45.52 \pm 0.80$ & $27.20 \pm 0.44$ & $47.14 \pm 0.67$ & $29.03 \pm 0.48$ \\
     & U & $47.36 \pm 2.05$ & $19.53 \pm 0.50$ & $45.70 \pm 1.95$ & $17.97 \pm 0.42$ & $47.33 \pm 2.04$ & $19.46 \pm 0.49$ \\
     & S & $48.56 \pm 0.38$ & $48.26 \pm 0.69$ & $46.72 \pm 0.67$ & $46.20 \pm 0.53$ & $48.53 \pm 0.38$ & $48.22 \pm 0.68$ \\
    \hline
    \multirow{4}{*}{50} & OD & $47.70 \pm 0.73$ & $30.87 \pm 0.59$ & $46.01 \pm 0.76$ & $28.96 \pm 0.59$ & $47.67 \pm 0.73$ & $30.81 \pm 0.59$ \\
     & O & $44.73 \pm 0.69$ & $27.42 \pm 0.36$ & $43.16 \pm 0.62$ & $25.44 \pm 0.52$ & $44.71 \pm 0.69$ & $27.35 \pm 0.36$ \\
     & U & $44.54 \pm 2.31$ & $18.75 \pm 0.44$ & $43.06 \pm 2.10$ & $17.23 \pm 0.45$ & $44.52 \pm 2.31$ & $18.68 \pm 0.44$ \\
     & S & $45.67 \pm 0.51$ & $45.55 \pm 0.41$ & $44.15 \pm 0.61$ & $43.88 \pm 0.48$ & $45.64 \pm 0.51$ & $45.52 \pm 0.41$ \\
    \hline
    \multirow{4}{*}{60} & OD & $45.91 \pm 0.60$ & $29.89 \pm 0.49$ & $44.24 \pm 0.69$ & $27.98 \pm 0.56$ & $45.88 \pm 0.60$ & $29.83 \pm 0.48$ \\
     & O & $42.78 \pm 1.05$ & $26.02 \pm 0.46$ & $41.43 \pm 0.97$ & $24.37 \pm 0.43$ & $42.76 \pm 1.05$ & $25.96 \pm 0.46$ \\
     & U & $42.64 \pm 2.95$ & $17.70 \pm 0.36$ & $41.22 \pm 2.73$ & $16.37 \pm 0.41$ & $42.61 \pm 2.94$ & $17.64 \pm 0.36$ \\
     & S & $43.69 \pm 0.29$ & $43.50 \pm 0.37$ & $42.21 \pm 0.41$ & $42.07 \pm 0.37$ & $43.67 \pm 0.29$ & $43.48 \pm 0.37$ \\
    \hline
    \end{tabular}}
{\textit{Note.} Entries are reported as mean $\pm$ standard deviation over 10 experiments. The $\alpha$-fairness column uses $\alpha=1$ and is computed as the geometric mean $\rho_1[\vv]=(\prod_{n=1}^N v_n)^{1/N}$.}
\end{table}

\section{Conclusion}\label{sec:8conclusion}
We studied fairness-aware sequential resource allocation in major-minor weakly coupled Markov decision processes. The proposed M2WCMDP framework models a centrally coordinated system with one major sub-MDP representing endogenous information and platform-level decisions, and a collection of minor sub-MDPs representing homogeneous participants with joint actions coupled through shared resource constraints. The objective balances the platform's expected total discounted reward with a welfare function $\rho$ defined over the vector of participant expected total discounted rewards.

The main theoretical result establishes a utilitarian reduction for the symmetric $\rho$-M2WCMDP optimization problem. We showed that, under symmetry conditions on the minor sub-MDPs, this fairness-aware problem can be solved by optimizing the utilitarian-based objective over permutation-invariant policies. We further use this structure to obtain a count-aggregation reformulation, which removes dependence on participant labels and provides a compact representation for large populations. To address large-scale and model-free settings, we developed a count-proportion deep reinforcement learning approach. The proposed CP-DRL method learns joint count actions using count-proportion state representations and stochastic policy networks. Numerical results on the machine replacement problem validate empirically the symmetry-based reduction in a controlled setting, while experiments on taxi dispatching and pricing show that coordinated pricing, relocation, and matching can improve fairness-aware welfare relative to various benchmarks. 

{Several extensions remain open. First, action generation remains a computational bottleneck as the current priority sampler is sequential and its representational flexibility depends on the chosen action-category partition. A natural next step is to design more efficient feasibility-preserving samplers and to characterize the approximation loss induced by coarse categories. Second, the exact utilitarian reduction relies on symmetry of homogeneous minor sub-MDPs. An important extension is to heterogeneous populations, for example through finitely many participant types. This would replace full symmetry by within-type symmetry, lead to a type-by-state count representation, and raise a new fairness question of how welfare should be balanced both within and across heterogeneous groups.}

\begingroup \parindent 0pt \parskip 0.0ex \def\enotesize{\normalsize} \theendnotes \endgroup

\section*{Acknowledgements}
Xiaohui Tu was partially funded by GERAD and IVADO. Yossiri Adulyasak was partially supported by the Canadian Natural Sciences and Engineering Research Council [Grant RGPIN-2021-03264] and by the Canada Research Chair program [CRC-2022-00087]. Erick Delage was partially supported by the Canadian Natural Sciences and Engineering Research Council [Grant RGPIN-2022-05261] and by the Canada Research Chair program [950-230057].

\bibliographystyle{apalike}
\bibliography{reference}

@article{aziz2024best,
  title={Best of both worlds: Ex ante and ex post fairness in resource allocation},
  author={Aziz, Haris and Freeman, Rupert and Shah, Nisarg and Vaish, Rohit},
  journal={Operations Research},
  volume={72},
  number={4},
  pages={1674--1688},
  year={2024},
  publisher={INFORMS}
}

@article{aminian2026markovian,
  title={Markovian Search with Ex Ante Constraints: Theory and Applications to Socially Aware Algorithmic Hiring},
  author={Aminian, Mohammad Reza and Manshadi, Vahideh and Niazadeh, Rad},
  journal={Management Science},
  year={2026},
  publisher={INFORMS}
}

@techreport{parrott2024tlc,
  author      = {Parrott, James A.},
  title       = {Revised Expense Model for the NYC Taxi and Limousine Commission's High-Volume For-Hire Vehicle Minimum Pay Standard},
  institution = {Center for New York City Affairs, The New School},
  year        = {2024},
  month       = dec,
  note        = {Report prepared for the New York City Taxi and Limousine Commission},
  url         = {https://www.nyc.gov/assets/tlc/downloads/pdf/driver_expense_report.pdf}
}

@article{caro2007dynamic,
  title={Dynamic assortment with demand learning for seasonal consumer goods},
  author={Caro, Felipe and Gallien, J{\'e}r{\'e}mie},
  journal={Management science},
  volume={53},
  number={2},
  pages={276--292},
  year={2007},
  publisher={INFORMS}
}

@book{zhang2022near,
  title={Near-optimality for multi-action multi-resource restless bandits with many arms},
  author={Zhang, Xiangyu},
  year={2022},
  publisher={Cornell University}
}

@misc{nyc_tripdata,
  author       = {{New York City Taxi and Limousine Commission}},
  title        = {TLC Trip Record Data},
  year         = {2026},
  publisher    = {New York City Taxi and Limousine Commission},
  url          = {https://www.nyc.gov/site/tlc/about/tlc-trip-record-data.page}
}

@article{moyano2021traffic,
  title={Traffic congestion and economic context: changes of spatiotemporal patterns of traffic travel times during crisis and post-crisis periods},
  author={Moyano, Amparo and St{\k e}pniak, Marcin and Moya-G{\'o}mez, Borja and Garc{\'\i}a-Palomares, Juan Carlos},
  journal={Transportation},
  volume={48},
  number={6},
  pages={3301--3324},
  year={2021},
  publisher={Springer}
}

@article{lehky2010decoding,
  title={Decoding poisson spike trains by gaussian filtering},
  author={Lehky, Sidney R},
  journal={Neural computation},
  volume={22},
  number={5},
  pages={1245--1271},
  year={2010},
  publisher={MIT Press One Rogers Street, Cambridge, MA 02142-1209, USA journals-info~…}
}

@inproceedings{michailidis2023balancing,
  title={Balancing fairness and efficiency in transport network design through reinforcement learning},
  author={Michailidis, Dimitris and Ghebreab, Sennay and Santos, Fernando P},
  booktitle={Proceedings of the 2023 International Conference on Autonomous Agents and Multiagent Systems},
  pages={2532--2534},
  year={2023}
}

@article{huaizhou2013fairness,
  title={Fairness in wireless networks: Issues, measures and challenges},
  author={Huaizhou, SHI and Prasad, R Venkatesha and Onur, Ertan and Niemegeers, IGMM},
  journal={IEEE Communications Surveys \& Tutorials},
  volume={16},
  number={1},
  pages={5--24},
  year={2013},
  publisher={IEEE}
}

@article{lesmana2019balancing,
  title={Balancing efficiency and fairness in on-demand ridesourcing},
  author={Lesmana, Nixie S and Zhang, Xuan and Bei, Xiaohui},
  journal={Advances in neural information processing systems},
  volume={32},
  year={2019}
}

@article{brown2025fluid,
  title={Fluid policies, reoptimization, and performance guarantees in dynamic resource allocation},
  author={Brown, David B and Zhang, Jingwei},
  journal={Operations Research},
  volume={73},
  number={2},
  pages={1029--1045},
  year={2025},
  publisher={INFORMS}
}

@article{delage2010percentile,
  title = {Percentile {{Optimization}} for {{Markov Decision Processes}} with {{Parameter Uncertainty}}},
  author = {Delage, Erick and Mannor, Shie},
  journal = {Operations Research},
  volume  = {58},
  number  = {1},
  pages   = {203--213},
  year    = {2010},
  doi     = {10.1287/opre.1080.0685}
}

@inproceedings{akbarzadeh2019restless,
  title={Restless bandits with controlled restarts: Indexability and computation of Whittle index},
  author={Akbarzadeh, Nima and Mahajan, Aditya},
  booktitle={2019 IEEE 58th conference on decision and control (CDC)},
  pages={7294--7300},
  year={2019},
  organization={IEEE}
}

@inproceedings{li2022efficient,
  title={Efficient resource allocation with fairness constraints in restless multi-armed bandits},
  author={Li, Dexun and Varakantham, Pradeep},
  booktitle={Uncertainty in Artificial Intelligence},
  pages={1158--1167},
  year={2022},
  organization={PMLR}
}

@article{gutjahr2018equity,
  title={Equity and deprivation costs in humanitarian logistics},
  author={Gutjahr, Walter J and Fischer, Sophie},
  journal={European Journal of Operational Research},
  volume={270},
  number={1},
  pages={185--197},
  year={2018},
  publisher={Elsevier}
}

@article{filippi2021single,
  title={On single-source capacitated facility location with cost and fairness objectives},
  author={Filippi, Carlo and Guastaroba, Gianfranco and Speranza, M Grazia},
  journal={European Journal of Operational Research},
  volume={289},
  number={3},
  pages={959--974},
  year={2021},
  publisher={Elsevier}
}

@article{bertsimas2013fairness,
  title={Fairness, efficiency, and flexibility in organ allocation for kidney transplantation},
  author={Bertsimas, Dimitris and Farias, Vivek F and Trichakis, Nikolaos},
  journal={Operations research},
  volume={61},
  number={1},
  pages={73--87},
  year={2013},
  publisher={INFORMS}
}

@article{hooker2012combining,
  title={Combining equity and utilitarianism in a mathematical programming model},
  author={Hooker, John N and Williams, H Paul},
  journal={Management Science},
  volume={58},
  number={9},
  pages={1682--1693},
  year={2012},
  publisher={INFORMS}
}

@article{bertsimas2012efficiency,
  title={On the efficiency-fairness trade-off},
  author={Bertsimas, Dimitris and Farias, Vivek F and Trichakis, Nikolaos},
  journal={Management Science},
  volume={58},
  number={12},
  pages={2234--2250},
  year={2012},
  publisher={INFORMS}
}

@article{sen1976real,
  title={Real national income},
  author={Sen, Amartya},
  journal={The Review of Economic Studies},
  volume={43},
  number={1},
  pages={19--39},
  year={1976},
  publisher={Wiley-Blackwell}
}

@inproceedings{siddique2020learning,
  title={Learning fair policies in multi-objective (deep) reinforcement learning with average and discounted rewards},
  author={Siddique, Umer and Weng, Paul and Zimmer, Matthieu},
  booktitle={International Conference on Machine Learning},
  pages={8905--8915},
  year={2020},
  organization={PMLR}
}

@article{cousins2021axiomatic,
  title={An axiomatic theory of provably-fair welfare-centric machine learning},
  author={Cousins, Cyrus},
  journal={Advances in Neural Information Processing Systems},
  volume={34},
  pages={16610--16621},
  year={2021}
}

@inproceedings{speicher2018unified,
  title={A unified approach to quantifying algorithmic unfairness: Measuring individual \&group unfairness via inequality indices},
  author={Speicher, Till and Heidari, Hoda and Grgic-Hlaca, Nina and Gummadi, Krishna P and Singla, Adish and Weller, Adrian and Zafar, Muhammad Bilal},
  booktitle={Proceedings of the 24th ACM SIGKDD international conference on knowledge discovery \& data mining},
  pages={2239--2248},
  year={2018}
}

@inproceedings{fan2023welfare,
  title={Welfare and Fairness in Multi-objective Reinforcement Learning},
  author={Fan, Ziming and Peng, Nianli and Tian, Muhang and Fain, Brandon},
  booktitle={Proceedings of the 2023 International Conference on Autonomous Agents and Multiagent Systems},
  pages={1991--1999},
  year={2023}
}

@article{xinying2023guide,
  title={A guide to formulating fairness in an optimization model},
  author={Chen, Violet Xinying and Hooker, John N},
  journal={Annals of Operations Research},
  volume={326},
  number={1},
  pages={581--619},
  year={2023},
  publisher={Springer}
}

@article{kolm1976unequal,
  title={Unequal inequalities. I},
  author={Kolm, Serge-Christophe},
  journal={Journal of economic Theory},
  volume={12},
  number={3},
  pages={416--442},
  year={1976},
  publisher={Academic Press}
}

@inproceedings{williamson2019fairness,
  title={Fairness risk measures},
  author={Williamson, Robert and Menon, Aditya},
  booktitle={International conference on machine learning},
  pages={6786--6797},
  year={2019},
  organization={PMLR}
}

@article{ju2023achieving,
  title={Achieving fairness in multi-agent markov decision processes using reinforcement learning},
  author={Ju, Peizhong and Ghosh, Arnob and Shroff, Ness B},
  journal={arXiv preprint arXiv:2306.00324},
  year={2023}
}

@article{rawls1971egalitarian,
  title={An egalitarian theory of justice},
  author={Rawls, John},
  journal={Philosophical ethics: An introduction to moral philosophy},
  pages={365--370},
  year={1971}
}

@article{zhu2021mean,
  title={A mean-field Markov decision process model for spatial-temporal subsidies in ride-sourcing markets},
  author={Zhu, Zheng and Ke, Jintao and Wang, Hai},
  journal={Transportation Research Part B: Methodological},
  volume={150},
  pages={540--565},
  year={2021},
  publisher={Elsevier}
}

@article{gallego1994optimal,
  title={Optimal dynamic pricing of inventories with stochastic demand over finite horizons},
  author={Gallego, Guillermo and Van Ryzin, Garrett},
  journal={Management science},
  volume={40},
  number={8},
  pages={999--1020},
  year={1994},
  publisher={INFORMS}
}

@article{adelman2004price,
  title={A price-directed approach to stochastic inventory/routing},
  author={Adelman, Daniel},
  journal={Operations Research},
  volume={52},
  number={4},
  pages={499--514},
  year={2004},
  publisher={INFORMS}
}

@article{bertsimas2007learning,
  title={A learning approach for interactive marketing to a customer segment},
  author={Bertsimas, Dimitris and Mersereau, Adam J},
  journal={Operations Research},
  volume={55},
  number={6},
  pages={1120--1135},
  year={2007},
  publisher={INFORMS}
}

@article{kazemi2019time2vec,
  title={Time2vec: Learning a vector representation of time},
  author={Kazemi, Seyed Mehran and Goel, Rishab and Eghbali, Sepehr and Ramanan, Janahan and Sahota, Jaspreet and Thakur, Sanjay and Wu, Stella and Smyth, Cathal and Poupart, Pascal and Brubaker, Marcus},
  journal={arXiv preprint arXiv:1907.05321},
  year={2019}
}

@article{schulman2017proximal,
  title={Proximal policy optimization algorithms},
  author={Schulman, John and Wolski, Filip and Dhariwal, Prafulla and Radford, Alec and Klimov, Oleg},
  journal={arXiv preprint arXiv:1707.06347},
  year={2017}
}

@article{mo2000fair,
  title={Fair end-to-end window-based congestion control},
  author={Mo, Jeonghoon and Walrand, Jean},
  journal={IEEE/ACM Transactions on networking},
  volume={8},
  number={5},
  pages={556--567},
  year={2000},
  publisher={IEEE}
}

@inproceedings{sun2022optimizing,
  title={Optimizing long-term efficiency and fairness in ride-hailing via joint order dispatching and driver repositioning},
  author={Sun, Jiahui and Jin, Haiming and Yang, Zhaoxing and Su, Lu and Wang, Xinbing},
  booktitle={Proceedings of the 28th ACM SIGKDD conference on knowledge discovery and data mining},
  pages={3950--3960},
  year={2022}
}

@article{Chen2024ABC,
title={Real-time spatial--intertemporal pricing and relocation in a ride-hailing network: Near-optimal policies and the value of dynamic pricing},
author={Chen, Qi and Lei, Yanzhe and Jasin, Stefanus},
journal={Operations Research},
volume={72},
number={5},
pages={2097--2118},
year={2024},
publisher={INFORMS}
}

@inproceedings{cai2021efficient,
  title={Efficient reinforcement learning in resource allocation problems through permutation invariant multi-task learning},
  author={Cai, Desmond and Lim, Shiau Hong and Wynter, Laura},
  booktitle={2021 60th IEEE Conference on Decision and Control (CDC)},
  pages={2270--2275},
  year={2021},
  organization={IEEE}
}

@techreport{cohen2016using,
  title={Using big data to estimate consumer surplus: The case of uber},
  author={Cohen, Peter and Hahn, Robert and Hall, Jonathan and Levitt, Steven and Metcalfe, Robert},
  year={2016},
  institution={National Bureau of Economic Research}
}

@article{huang2018taxi,
  title={Taxi driver speeding: Who, when, where and how? A comparative study between Shanghai and New York City},
  author={Huang, Yizhe and Sun, Daniel and Tang, Juanyu},
  journal={Traffic injury prevention},
  volume={19},
  number={3},
  pages={311--316},
  year={2018},
  publisher={Taylor \& Francis}
}

@article{gast2022reoptimization,
  title={Reoptimization nearly solves weakly coupled markov decision processes},
  author={Gast, Nicolas and Gaujal, Bruno and Yan, Chen},
  journal={arXiv preprint arXiv:2211.01961},
  year={2022}
}

@inproceedings{dai2017balanced,
  title={A balanced assignment mechanism for online taxi recommendation},
  author={Dai, Guang and Huang, Jianbin and Wambura, Stephen Manko and Sun, Heli},
  booktitle={2017 18th IEEE international conference on mobile data management (MDM)},
  pages={102--111},
  year={2017},
  organization={IEEE}
}

@article{mandal2022socially,
  title={Socially fair reinforcement learning},
  author={Mandal, Debmalya and Gan, Jiarui},
  journal={arXiv preprint arXiv:2208.12584},
  year={2022}
}

@book{moulin1991axioms,
  title={Axioms of cooperative decision making},
  author={Moulin, Herv{\'e}},
  number={15},
  series={1},
  year={1991},
  publisher={Cambridge university press}
}

@article{weymark1981generalized,
  title={Generalized Gini inequality indices},
  author={Weymark, John A},
  journal={Mathematical social sciences},
  volume={1},
  number={4},
  pages={409--430},
  year={1981},
  publisher={Elsevier}
}

@article{tsang2025unified,
  title={A Unified Framework for Analyzing and Optimizing a Class of Convex Fairness Measures},
  author={Tsang, Man Yiu and Shehadeh, Karmel S},
  journal={Operations Research},
  year={2025},
  publisher={INFORMS}
}

@inproceedings{bhattacharya2024active,
  title={Active Learning for Fair and Stable Online Allocations},
  author={Bhattacharya, Riddhiman and Nguyen, Thanh and Sun, Will Wei and Tawarmalani, Mohit},
  booktitle={Proceedings of the 25th ACM Conference on Economics and Computation},
  pages={196--197},
  year={2024}
}

@inproceedings{chen2020fair,
  title={Fair contextual multi-armed bandits: Theory and experiments},
  author={Chen, Yifang and Cuellar, Alex and Luo, Haipeng and Modi, Jignesh and Nemlekar, Heramb and Nikolaidis, Stefanos},
  booktitle={Conference on Uncertainty in Artificial Intelligence},
  pages={181--190},
  year={2020},
  organization={PMLR}
}

@inproceedings{huang2024statistical,
  title={On the statistical efficiency of mean-field reinforcement learning with general function approximation},
  author={Huang, Jiawei and Yardim, Batuhan and He, Niao},
  booktitle={International Conference on Artificial Intelligence and Statistics},
  pages={289--297},
  year={2024},
  organization={PMLR}
}

@inproceedings{hassanzadeh2023sequential,
  title={Sequential fair resource allocation under a markov decision process framework},
  author={Hassanzadeh, Parisa and Kreacic, Eleonora and Zeng, Sihan and Xiao, Yuchen and Ganesh, Sumitra},
  booktitle={Proceedings of the Fourth ACM International Conference on AI in Finance},
  pages={673--680},
  year={2023}
}

@article{sinha2023no,
  title={No-regret algorithms for fair resource allocation},
  author={Sinha, Abhishek and Joshi, Ativ and Bhattacharjee, Rajarshi and Musco, Cameron and Hajiesmaili, Mohammad},
  journal={Advances in Neural Information Processing Systems},
  volume={36},
  pages={48083--48109},
  year={2023}
}

@article{ghalme2021long,
  title={Long-term resource allocation fairness in average markov decision process (amdp) environment},
  author={Ghalme, Ganesh and Nair, Vineet and Patil, Vishakha and Zhou, Yilun},
  journal={arXiv preprint arXiv:2102.07120},
  year={2021}
}

@inproceedings{wen2021algorithms,
  title={Algorithms for fairness in sequential decision making},
  author={Wen, Min and Bastani, Osbert and Topcu, Ufuk},
  booktitle={International Conference on Artificial Intelligence and Statistics},
  pages={1144--1152},
  year={2021},
  organization={PMLR}
}

@article{kelly1998rate,
  title={Rate control for communication networks: shadow prices, proportional fairness and stability},
  author={Kelly, Frank P and Maulloo, Aman K and Tan, David Kim Hong},
  journal={Journal of the Operational Research society},
  volume={49},
  number={3},
  pages={237--252},
  year={1998},
  publisher={Taylor \& Francis}
}

@inproceedings{bistritz2020my,
  title={My fair bandit: Distributed learning of max-min fairness with multi-player bandits},
  author={Bistritz, Ilai and Baharav, Tavor and Leshem, Amir and Bambos, Nicholas},
  booktitle={International Conference on Machine Learning},
  pages={930--940},
  year={2020},
  organization={PMLR}
}

@article{meherrem2024stochastic,
  title={A stochastic maximum principle for general mean-field system with constraints},
  author={Meherrem, Shahlar and Hafayed, Mokhtar},
  journal={Numerical Algebra, Control and Optimization},
  pages={0--0},
  year={2024},
  publisher={Numerical Algebra, Control and Optimization}
}

@inproceedings{yang2018mean,
  title={Mean field multi-agent reinforcement learning},
  author={Yang, Yaodong and Luo, Rui and Li, Minne and Zhou, Ming and Zhang, Weinan and Wang, Jun},
  booktitle={International conference on machine learning},
  pages={5571--5580},
  year={2018},
  organization={PMLR}
}

@inproceedings{li2019efficient,
  title={Efficient ridesharing order dispatching with mean field multi-agent reinforcement learning},
  author={Li, Minne and Qin, Zhiwei and Jiao, Yan and Yang, Yaodong and Wang, Jun and Wang, Chenxi and Wu, Guobin and Ye, Jieping},
  booktitle={The world wide web conference},
  pages={983--994},
  year={2019}
}

@book{bensoussan2013mean,
  title={Mean field games and mean field type control theory},
  author={Bensoussan, Alain and Frehse, Jens and Yam, Phillip and others},
  volume={101},
  year={2013},
  publisher={Springer}
}

@article{andersson2011maximum,
  title={A maximum principle for SDEs of mean-field type},
  author={Andersson, Daniel and Djehiche, Boualem},
  journal={Applied Mathematics \& Optimization},
  volume={63},
  pages={341--356},
  year={2011},
  publisher={Springer}
}

@article{jusup2023safe,
  title={Safe model-based multi-agent mean-field reinforcement learning},
  author={Jusup, Matej and P{\'a}sztor, Barna and Janik, Tadeusz and Zhang, Kenan and Corman, Francesco and Krause, Andreas and Bogunovic, Ilija},
  journal={International conference on autonomous agents and multiagent systems},
  year={2024}
}

@article{cui2023major,
  title={Major-minor mean field multi-agent reinforcement learning},
  author={Cui, Kai and Fabian, Christian and Tahir, Anam and Koeppl, Heinz},
  journal={International Conference on Machine Learning},
  year={2024}
}

@inproceedings{cui2024learning,
  title={Learning discrete-time major-minor mean field games},
  author={Cui, Kai and Dayan{\i}kl{\i}, G{\"o}k{\c{c}}e and Lauri{\`e}re, Mathieu and Geist, Matthieu and Pietquin, Olivier and Koeppl, Heinz},
  booktitle={Proceedings of the AAAI Conference on Artificial Intelligence},
  volume={38},
  pages={9616--9625},
  year={2024}
}

@article{bauerle2023mean,
  title={Mean field Markov decision processes},
  author={B{\"a}uerle, Nicole},
  journal={Applied Mathematics \& Optimization},
  volume={88},
  number={1},
  pages={12},
  year={2023},
  publisher={Springer}
}

@article{chen2025primal,
  title={A primal-dual approach to constrained markov decision processes with applications to queue scheduling and inventory management},
  author={Chen, Yi and Dong, Jing and Wang, Zhaoran and Zhang, Chuheng},
  journal={Management Science},
  year={2025},
  publisher={INFORMS}
}

@article{nadarajah2025self,
  title={Self-adapting network relaxations for weakly coupled markov decision processes},
  author={Nadarajah, Selvaprabu and Cire, Andre A},
  journal={Management Science},
  volume={71},
  number={2},
  pages={1779--1802},
  year={2025},
  publisher={INFORMS}
}

@inproceedings{robledo2024deep,
  title={Deep reinforcement learning for weakly coupled MDP’s with continuous actions},
  author={Robledo, Francisco and Ayesta, Urtzi and Avrachenkov, Konstantin},
  booktitle={International Conference on Analytical and Stochastic Modeling Techniques and Applications},
  pages={67--80},
  year={2024},
  organization={Springer}
}

@article{brown2023strength,
  title={On the strength of relaxations of weakly coupled stochastic dynamic programs},
  author={Brown, David B and Zhang, Jingwei},
  journal={Operations Research},
  volume={71},
  number={6},
  pages={2374--2389},
  year={2023},
  publisher={INFORMS}
}

@article{el2023weakly,
  title={Weakly coupled deep Q-Networks},
  author={El Shar, Ibrahim and Jiang, Daniel},
  journal={Advances in Neural Information Processing Systems},
  volume={36},
  pages={43931--43950},
  year={2023}
}

@article{meuleau1998solving,
  title={Solving very large weakly coupled Markov decision processes},
  author={Meuleau, Nicolas and Hauskrecht, Milos and Kim, Kee-Eung and Peshkin, Leonid and Kaelbling, Leslie Pack and Dean, Thomas L and Boutilier, Craig},
  journal={AAAI/IAAI},
  volume={8},
  pages={2},
  year={1998}
}

@phdthesis{hawkins2003langrangian,
  title={A Langrangian decomposition approach to weakly coupled dynamic optimization problems and its applications},
  author={Hawkins, Jeffrey Thomas},
  year={2003},
  school={Massachusetts Institute of Technology}
}

@article{adelman2008relaxations,
  title={Relaxations of weakly coupled stochastic dynamic programs},
  author={Adelman, Daniel and Mersereau, Adam J},
  journal={Operations Research},
  volume={56},
  number={3},
  pages={712--727},
  year={2008},
  publisher={INFORMS}
}

@article{whittle1988restless,
  title={Restless bandits: Activity allocation in a changing world},
  author={Whittle, Peter},
  journal={Journal of applied probability},
  volume={25},
  number={A},
  pages={287--298},
  year={1988},
  publisher={Cambridge University Press}
}

@book{puterman2014markov,
  title={Markov decision processes: discrete stochastic dynamic programming},
  author={Puterman, Martin L},
  year={2005},
  publisher={John Wiley \& Sons}
}

\newpage
\appendix

\section*{\Large Electronic Companions of Paper ``Fair Policy Optimization in Major-Minor Weakly Coupled Markov Decision Processes''}
\setcounter{page}{1}

\vspace{22pt}

\section{Literature Review} \label{sec:2lr}
Motivated to address the large-scale coordination challenges and endogenous system-level information inherent in taxi dispatching systems, our work extends the standard weakly coupled Markov decision processes (WCMDP) framework to a major-minor formulation. This formulation is closely connected to the literature on traditional weakly coupled models and modern mean-field approaches, while explicitly integrating fairness notions.

\paragraph{Weakly Coupled Markov Decision Processes}
The foundational work by \cite{meuleau1998solving} recognized the potential for decomposition in resource-constrained systems, but the proposed heuristics were task-specific and without optimality guarantees. Subsequently, \cite{hawkins2003langrangian} introduced a Lagrangian relaxation approach for constrained stochastic dynamic programming problems to allow tractable approximations. \cite{adelman2008relaxations} extended these ideas through performance bounds and policy design using linear programming (LP) relaxations, thus providing theoretical justification for decomposition-based methods.

Recent advances continue to build on relaxation and decomposition approaches. For example, \cite{gast2022reoptimization} proposed LP-based policies with sublinear regret and bounded duality gaps as the size of sub-MDPs grows, while their practicability is based on repeatedly solving relaxed LPs in real time. \cite{brown2023strength} investigated the tightness of various relaxations for weakly coupled stochastic dynamic programs, quantifying the suboptimality introduced by standard linear relaxations and providing conditions under which tighter formulations significantly improve solution quality. \cite{chen2025primal} proposed a unified primal-dual framework that integrates policy gradient methods with dual variable updates to address resource-constrained planning problems more robustly, offering convergence guarantees under mild assumptions.

In parallel, several methods have attempted to address the limitations of classical decomposition techniques by introducing learning-based or more adaptive architectures. For example, \cite{el2023weakly} introduced a weakly coupled variant of deep Q-learning that partitions the value function estimation across sub-tasks and uses dual bounds to coordinate global feasibility. However, the choice and coverage of multiplier sets are unclear and potentially computationally expensive. Similarly, \cite{robledo2024deep} developed a deep reinforcement learning (RL) approach tailored to WCMDPs with continuous action spaces, which showed how neural function approximators can be integrated with soft constraint penalties to achieve efficient learning in high-dimensional problems.  \cite{nadarajah2025self} proposed self-adapting network relaxations that dynamically adjust relaxation tightness to the structure of linking constraints. These studies provide a promising direction for learning-based and structure-aware optimization in weakly coupled systems.

\paragraph{Mean-Field Control (MFC)} Mean-field approximations provide another powerful tool for modeling and solving large-population stochastic control problems. Initiated by \cite{andersson2011maximum} and formalized in \cite{bensoussan2013mean}, MFC framed control problems by enabling optimization directly on population-level distributions instead of individual interactions.

Recent work has extended MFC to more general and realistic settings. For example, \cite{bauerle2023mean} showed that, under weak coupling and ergodicity assumptions, the optimal control converges to the solution of a static optimization over stationary distributions as the population size grows to infinity. \cite{meherrem2024stochastic} established stochastic maximum principles for general mean-field systems with constraints, where coefficients depend nonlinearly on both the state process and its probability law. However, classical MFC typically assumes full knowledge of system dynamics and often relies on structural conditions such as regularity and convexity, limiting practical applicability. In contrast, our symmetry reduction is formulated for a finite population, and the count representation removes participant labels without taking an infinite-population limit.

\paragraph{Mean-Field Reinforcement Learning (MFRL)} MFRL directly extends MFC by integrating data-driven learning into large-population control, allowing each agent to interact with mean-field dynamics. \cite{yang2018mean} pioneered MFRL by developing mean-field Q-learning and actor-critic, and proved its convergence to a Nash equilibrium. Their work showed empirical success in resource allocation and coordination benchmarks, establishing MFRL as a scalable alternative to traditional multi-agent RL frameworks. \cite{huang2024statistical} refined the theory of MFRL, analyzing sample efficiency and learning complexity under general function approximation, and identifying limitations of previous convergence results.

This foundational result was applied in various transportation domains. For example, \cite{li2019efficient} applied this framework to large-scale ride-sharing dispatch, showing that mean-field interactions in a multi-agent RL setting could effectively scale. \cite{jusup2023safe} further extended this approach by incorporating log-barrier constraints into mean-field RL to ensure region-level safety and coverage, and validated the proposed methods in vehicle repositioning applications. 

Recent MFRL extensions introduce a major-minor structure, where a dominant agent influences a population of minor agents through mean-field interaction. \cite{cui2024learning, cui2023major} established finite-agent approximation results, proved that stationary policies suffice in the major-minor mean-field control setting, and introduced a policy gradient algorithm that effectively solves such structured multi-agent systems. This line of research is highly relevant to our work but differs in incorporating fairness objectives into a fully centralized major-minor setting, where minor actions are constrained by shared resources.

\paragraph{Fairness in Resource Allocation} 
Fairness concerns arise when allocating limited resources among multiple stakeholders. Classical notions of fairness in resource allocation have been grounded in network economics and game-theoretic models, including max-min fairness \citep{bistritz2020my}, proportional fairness \citep{kelly1998rate}, and $\alpha$-fairness \citep{mo2000fair}.

More recently, fairness has been extended to the sequential decision-making setting. \cite{wen2021algorithms} and \cite{ghalme2021long} explored fairness-aware formulations of MDPs, introducing objectives that account for equity across states or populations over time. These efforts have broadened the scope of fairness from static allocation to dynamic and learning-based environments.

In parallel, the online learning and bandits community has developed fairness-aware algorithms to make resource allocation decisions with performance guarantees. For example, \cite{chen2020fair} introduced a fairness-aware approach to contextual bandits that enforces minimum selection rates across users while achieving provable regret. \cite{sinha2023no} developed an online proportional fairness algorithm with regret bounds, while \cite{hassanzadeh2023sequential} introduced policies that optimize the Nash Social Welfare in sequential resource allocation problems with bounded optimality gaps. \cite{bhattacharya2024active} introduced actively sampled fair allocation algorithms that solicit feedback from a subset of agents per period, while achieving logarithmic regret in both fairness and matching stability. These works show a shift toward algorithmic frameworks that are not only efficient but also equitable.

\section{Boundary of the welfare class}\label{apx:welfare-examples}
To clarify the scope of the welfare class $\rho$ in Definition \ref{def:rho-fairness}, we categorize commonly used welfare functions into two classes: (i) \textit{directly compliant} measures that natively satisfy all four properties in Definition (\ref{def:rho-fairness}), such as the generalized Gini function (GGF) in Example \ref{exmp:ggf}, (ii) \textit{transformable} measures that satisfy the definition after an appropriate normalization or reformulation, such as the certainty-equivalent representation of $\alpha$-fairness in Example \ref{exmp:ce-alpha}, and the widely used mean-minus penalty introduced next. Table \ref{tab:fairness_measures} summarizes the two categories.

\begin{example}[Mean-Minus Penalty]
Another highly structural and widely applicable class of welfare measure can be constructed by subtracting a convex inequality penalty $\varphi(\vv)$ from the utilitarian mean $\bar{v} = \frac{1}{N}\sum_{i=1}^N v_i$, formulated as $\rho[\vv] = \bar{v} - c \cdot \varphi(\vv)$. If $\varphi(\vv)$ is convex, symmetric, and vanishes under perfect equality such that $\rho[c\vone] = c$, the resulting $\rho[\vv]$ directly satisfies Definition \ref{def:rho-fairness}. This unified form broadly includes Gini dispersion \citep{weymark1981generalized}, Gini coefficient \citep{sen1976real}, mean absolute deviation \citep{xinying2023guide}, and other inequality indices \citep{speicher2018unified, cousins2021axiomatic}. \hfill $\square$
\end{example}

\begin{table}\TABLE
    {Fairness Measures and Definition \ref{def:rho-fairness} Compliance \label{tab:fairness_measures}}
    {\begin{tabular}{lcc}
    \hline
    {Measure} & {Violated Properties} & {Transformation for Definition \ref{def:rho-fairness}}\\
    \hline
    \multicolumn{3}{c}{\textit{Category I: Directly Compliant Class (Natively Satisfying Definition \ref{def:rho-fairness})}} \\
    Utilitarian (Mean) & None & N/A\\
    Generalized Gini \citep{weymark1981generalized} & None & N/A\\
    Max-Min Fairness \citep{rawls1971egalitarian} & None & N/A\\
    \multicolumn{3}{c}{\textit{Category II: Transformed Class (via Reformulation)}} \\
    Standard $\alpha$-Fairness ($\alpha \neq$ 1) \citep{mo2000fair} & CVI & \textbf{CE Form}: $\rho_\alpha[\vv] = \left(\frac{1}{N}\sum v_n^{1-\alpha}\right)^{\frac{1}{1-\alpha}}$ \\
    Nash Social Welfare ($\alpha = 1$) \citep{fan2023welfare} & CVI & \textbf{CE Form}: $\rho_1[\vv] = \left(\prod v_n\right)^{\frac{1}{N}}$ \\
    Gini Coefficient \citep{sen1976real} & MN, CV, CVI & \textbf{Mean-Minus Form}: $\rho[\vv] = \bar{v} - \bar{v} \cdot \text{Gini}(\vv)$ \\
    \hline
    \end{tabular}}
    {\textit{Note.} For compactness, MN stands for monotonicity, CV for concavity, PI for permutation invariance, and CVI for constant vector invariance, as in Definition \ref{def:rho-fairness}. CE denotes certainty-equivalent as in Example \ref{exmp:ce-alpha}.}
\end{table}

While our framework is expansive, Definition \ref{def:rho-fairness} intentionally excludes certain fairness concepts to preserve mathematical tractability. For example, \textit{threshold fairness} objectives that seek to maximize the number of participants exceeding a specific performance baseline (e.g., $|\{n : v_n \geq c\}|$) are inherently step-functions that fundamentally violate CV. The resulting integer programming formulations are combinatorial and incompatible with gradient-based policy optimization \citep{xinying2023guide}. Additionally, \textit{priority-weighted welfare} structures where the platform assigns heterogeneous importance weights to specific individuals (e.g., higher weights for important customers such that $\sum_{n=1}^N w_n v_n$ where $w_i \neq w_j$) violate PI \citep{moulin1991axioms}.

\section{Proofs of Section \ref{sec:4reduction}}
We start this section with some preliminary results regarding 1) the effect of replacing a policy with one that has permuted indices on the value function of a symmetric major-minor weakly coupled Markov decision process (M2WCMDP) in \ref{apx:value_function_under_permuted}; and 2) a well-known result from \cite{puterman2014markov} on the equivalency between stationary policies and occupancy measures in \ref{apx:mapping-pi-x}. This is followed by the proof of Lemma \ref{thm:bar-policy} on the uniform state-value representation in \ref{apx:bar-policy}, which helps establish our main result, Theorem \ref{thm:eqv} in \ref{apx:thm:eqv}.

\subsection{Value Function under Permuted Policy for Symmetric M2WCMDP}\label{apx:value_function_under_permuted}

\begin{lemma}\label{thm:piQVQ}
    If a M2WCMDP is symmetric, then for any policy $\vpi$ and permutation operator $Q$, we have $\Vo({\vpi}) = Q \Vo({\ppi})$ and $V^0_0({\vpi}) = V^0_0({\ppi})$, where the permuted policy $\ppi(a^0, \a \mid s^0, \s):=\vpi(a^0, Q\a \mid s^0, Q\s)$ for all $(s^0, \s, a^0, \a)$ tuples.
\end{lemma}

This lemma implies an important equivalency in symmetric M2WCMDPs with identical minor sub-MDPs. If we permute the states and actions of a policy, the permuted version of the resulting value function is equivalent to the original value function.

\begin{proof}{Proof.}
Fix a permutation operator $Q\in\G$. We can first show that for all $t \geq 0$,
\begin{equation}\label{eq:permuted-trajectory}
    \begin{aligned}
        &\Prob^{\ppi}(s^0_t=s^0, \s_t=\s, a^0_t=a^0, \a_t=\a|s_0^0 = \bar{s}_0^0, \s_0=\so) \\
        &\qquad = \Prob^{\vpi}(s^0_t=s^0, \s_t=Q \s, a^0_t=a^0, \a_t = Q \a|s_0^0 = \bar{s}_0^0, \s_0 = Q \so).
    \end{aligned}
\end{equation}

This can be done inductively. Starting at $t=0$, we have that, for all $(\bar{s}_0^0, \so, s^0, \s, a^0, \a)$:
\begin{align*}
    & \Prob^{\ppi}(s_0^0 = s^0, \s_0=\s,a^0_0 = a^0, \a_0=\a|s_0^0 = \bar{s}_0^0, \s_0=\so) \\
    &={\ppi(a^0, \a \mid s^0, \s)}\sI\{s_0^0 = \bar{s}_0^0, \s=\so\}\\
    &= {\vpi(a^0, Q\a \mid s^0, Q\s)\sI\{s_0^0 = \bar{s}_0^0, \s=\so\}}\\
    &= {\vpi(a^0, Q\a \mid s^0, Q\s)}\sI\{s_0^0 = \bar{s}_0^0, Q\s=Q\so\}\\
    &=\Prob^{\vpi}(s_0^0 = s^0, \s_0=Q \s, a^0_0 = a^0, \a_0=Q \a|s_0^0 = \bar{s}_0^0, \s_0=Q\so).
\end{align*}

Next, assuming that Equation (\ref{eq:permuted-trajectory}) holds for all $(\bar{s}_0^0, \so, s^0, \s, a^0, \a)$. By applying this to an arbitrary tuple $(\bar{s}_0^0, \so, \tilde{s}^0, \tilde{\s}, \tilde{a}^0, \tilde{\a})$, we can show that it is also the case for $t+1$, namely:
\begin{align*}
    &\Prob^{\ppi}(s^0_{t+1}=\tilde{s}^0, \s_{t+1}=\tilde{\s}, a^0_{t+1}=\tilde{a}^0, \a_{t+1}=\tilde{\a}|s_0^0 = \bar{s}_0^0, \s_0=\so) \\
    &= \ppi(\tilde{a}^0, \tilde{\a} \mid \tilde{s}^0, \tilde{\s})\sum_{s^0, \s, a^0, \a}  \Prob(s^0_{t+1} = \tilde{s}^0, \s_{t+1}=\tilde{\s}|s^0_t=s^0, \s_t=\s, a^0_t=a^0, \a_t=\a) \\
    & \qquad \cdot \Prob^{\ppi}(s^0_t=s^0, \s_t=\s, a^0_t=a^0, \a_t=\a|s_0^0 = \bar{s}_0^0, \s_0=\so)\\
\end{align*}
\begin{align*}
    &= \vpi(\tilde{a}^0, Q \tilde{\a} \mid \tilde{s}^0, Q \tilde{\s})\sum_{s^0, \s, a^0, \a} \Pn(\tilde{\s}|s^0, \s, a^0 ,\a) p^0(\tilde{s}^0 | s^0, \s, a^0, \a)\\
    & \qquad \cdot \Prob^{\vpi}(s^0_t=s^0, \s_t=Q \s, a^0_t=a^0, \a_t = Q \a|s_0^0 = \bar{s}_0^0, \s_0= Q \so)\\
    &= \vpi(\tilde{a}^0, Q \tilde{\a} \mid \tilde{s}^0, Q \tilde{\s})\sum_{s^0, \s, a^0, \a} \Pn(Q \tilde{\s}|s^0, Q \s, a^0, Q \a) p^0(\tilde{s}^0 | s^0, Q \s, a^0, Q \a)\\
    &\qquad \cdot \Prob^{\pi}(s^0_t = s^0, \s_t=Q \s, a^0_t = a^0, \a_t=Q \a|s_0^0 = \bar{s}_0^0, \s_0=Q \so)\\
    &=\Prob^{\vpi}(s^0_{t+1}=\tilde{s}^0, \s_{t+1}=Q \tilde{\s}, a^0_{t+1}=\tilde{a}^0, \a_{t+1}=Q \tilde{\a}|s_0^0 = \bar{s}_0^0, \s_0=Q\so),
\end{align*}
where we used the fact that the minor sub-MDPs are identical so $\Pn(\tilde{\s} \mid s^0, \s, a^0, \a)=\Pn(Q\tilde{\s} \mid s^0, Q\s, a^0, Q\a)$ {according to minor kernel equivariance (Definition \ref{def:symmetry-M2WCMDP}.1)} and $p^0(\tilde{s}^0 \mid s^0, \s, a^0, \a) = p^0(\tilde{s}^0 \mid s^0, Q \s, a^0, Q \a)$ {according to major kernel equivariance (Definition \ref{def:symmetry-M2WCMDP}.2)}.

We now have that,
\begin{equation*}
    \begin{aligned}
    &\V_0({\ppi}) = \sum_{s^0, \s,a^0, \a}\sum_{\bar{s}_0^0, \so} \bm{\mu}(\bar{s}_0^0, \so)\sum_{t=0}^\infty \gamma^t \Prob^{\ppi}(s^0_t=s^0, \s_t=\s, a^0_t=a^0, \a_t=\a \mid s_0^0 = \bar{s}_0^0, \s_0=\so) \r(s^0, \s, a^0, \a)\\
    &= \sum_{s^0, \s,a^0, \a}\sum_{\bar{s}_0^0, \so} \bm{\mu}(\bar{s}_0^0, \so)\sum_{t=0}^\infty \gamma^t \Prob^{\vpi}(s^0_t = s^0, \s_t= Q \s,a^0_t = a^0,\a_t= Q\a|s_0^0 = \bar{s}_0^0,\s_0=Q\so) \r(s^0, \s, a^0, \a)\\
    &= \sum_{s^0, \s,a^0, \a}\sum_{\bar{s}_0^0, \so} \bm{\mu}(\bar{s}_0^0, Q \so)\sum_{t=0}^\infty \gamma^t \Prob^{\vpi}(s^0_t = s^0, \s_t= Q \s,a^0_t = a^0,\a_t= Qa|s^0_0 =\bar{s}^0_0, \s_0=Q\so) \r(s^0, \s, a^0, \a)\\
    &= \sum_{s^0, \s,a^0, \a}\sum_{\bar{s}_0^0, \so} \bm{\mu}(\bar{s}_0^0, Q \so)\sum_{t=0}^\infty \gamma^t \Prob^{\vpi}(s^0_t = s^0,\s_t= Q \s,a^0_t = a^0,\a_t= Q\a|s^0_0=\bar{s}^0_0, \s_0=Q\so) Q^{-1}\r(s^0, Q\s, a^0,Q\a)\\
    &= Q^{-1}\left(\sum_{s^0, \s,a^0, \a}\sum_{\bar{s}_0^0, \so} \bm{\mu}(\bar{s}_0^0, Q \so)\sum_{t=0}^\infty \gamma^t \Prob^{\vpi}(s^0_t = s^0,\s_t= Q \s,a^0_t = a^0,\a_t= Q\a|s_0^0 = \bar{s}_0^0, \s_0=Q\so) \r(s^0, Q\s, a^0,Q\a)\right)\\
    &= Q^{-1} \left(\sum_{s^0, \s', a^0, \a'}\sum_{\bar{s}^0_0, \so'} \bm{\mu}(\bar{s}^0_0, \so')\sum_{t=0}^\infty \gamma^t \Prob^{\vpi}(s^0_t = s^0,\s_t= \s',a^0_t = a^0,\a_t= \a'|s^0_0=\bar{s}^0_0, \s_0=\so') \r(s^0, \s', a^{0}, \a')\right) \\
    & = Q^{-1} \V_0({\vpi})
    \end{aligned}
\end{equation*}
where we first use the relation between $\mathbb{P}^{\ppi}$ and $\vpi$, then exploit the permutation invariance of $\bm{\mu}$. We then exploit the permutation invariance {$\r(s^0, Q\s, a^0, Q\a)=Q^{-1} \r(s^0, \s, a^0, \a)$, maintaining consistent multiplication by $Q^{-1}$,} and reindex the summations using $\s':=Q \s$, $\a':=Q \a$, and $\so':=Q\so$. 

{Furthermore, $V^0_0({\vpi}) = V^0_0({\ppi})$ follows by the same reindexing argument using the scalar invariance $r^0(s^0,\s,a^0,\a)=r^0(s^0,Q\s,a^0,Q\a)$ as in Definition \ref{def:symmetry-M2WCMDP}.} \hfill $\square$
\end{proof}

\subsection{Mapping between stationary policies and occupancy measures}\label{apx:mapping-pi-x}

\begin{lemma} (Theorem 6.9.1 of \cite{puterman2014markov})\label{prop:x-pi-x}
    Let $\Pi$ denote the set of stationary stochastic Markov policies and $\mathcal{X}$ the set of occupancy measures.
    \begin{enumerate}[leftmargin=*]
        \item For any policy $\vpi \in \Pi$, an occupancy measure $q_{\vpi}: \S^0 \times \Sn \times \A^0 \times \An \rightarrow {[0, 1/(1-\gamma)]}\in\mathcal{X}$ is obtained as
            \begin{equation}
                q_{\vpi}(s^0,\s, a^0, \a) := \sum_{\bar{s}^0_0, \so} \bm{\mu}(\bar{s}^0_0, \so) \sum_{t=0}^{\infty} \gamma^{t} \Prob^{\vpi}\left(s_t^0 = s^0, \s_t=\s, a_t^0 = a^0, \a_t=\a | s^0_0 = \bar{s}^0_0, \s_0=\so \right),
            \label{eq:x-pi}
            \end{equation}
        for all $s^0 \in \S^0, \s \in \Sn$ and $a^0 \in \A^0, \a \in \An$.
        \item For any occupancy measure $q(s^0, \s, a^0, \a): \S^0 \times \Sn \times \A^0 \times \An \rightarrow {[0, 1/(1-\gamma)]}\in\mathcal{X}$, a policy $\vpi_{q}$ can be constructed as
            \begin{equation}
                \vpi_{q}(a^0, \a \mid s^0, \s):=\frac{q(s^0, \s, a^0, \a)}{\sum\limits_{a^0, \a} q\left(s^0, \s, a^0, \a\right)},
            \label{eq:pi-x}
            \end{equation}
        for all $s^0 \in \S^0, \s \in \Sn$ and $a^0 \in \A^0, \a \in \An$ {such that $\sum_{a^0, \a} q\left(s^0, \s, a^0, \a\right) > 0$ and letting the policy be arbitrary on unreachable states.}
    \end{enumerate}
    One necessarily has that for all $q\in\mathcal{X}$,  $q=q_{\pi_q}$.
\end{lemma}

Now, we show that the value functions can be represented using occupancy measures.

\begin{lemma} \label{prop:xr=muv}
    For any policy $\vpi \in \Pi$, and the occupancy measure $q_{\vpi}$ defined by (\ref{eq:x-pi}), the expected total discounted rewards under the policy $\vpi$ can be expressed as:
    \begin{equation}\label{eq:vpi}
        \Vo({\vpi}) = \sum_{s^0, \s, a^0, \a} q_{\vpi}(s^0, \s, a^0, \a) \r(s^0, \s, a^0, \a).
    \end{equation}
    and
    \begin{equation}\label{eq:v0pi}
        V^0_0({\vpi}) = \sum_{s^0, \s, a^0, \a} q_{\vpi}(s^0, \s, a^0, \a) r^0 (s^0, \s, a^0, \a),
    \end{equation}
\end{lemma}

\begin{proof}{Proof.}
Expanding the expected total discounted rewards $\Vo(\vpi)$ (as defined by Equation \ref{eq:vector} and \ref{eq:vector-component}), we have:
\begin{equation*}
    \Vo({\vpi}) = \sum_{\bar{s}^0_0, \so} \bm{\mu}(\bar{s}^0_0, \so) \sum_{s^0, \s, a^0, \a} \sum_{t=0}^{\infty} \gamma^t \Prob^{\vpi}\left(s_t^0 = s^0, \s_t=\s, a_t^0 = a^0, \a_t=\a | s^0_0 = \bar{s}^0_0, \s_0=\so \right) \r(s^0, \s, a^0, \a).
\end{equation*}
Rearranging the terms:
\begin{equation*}
    \resizebox{\linewidth}{!}{$
    \Vo({\vpi}) = \sum \limits_{s^0, \s, a^0, \a} \left(\sum \limits_{\bar{s}^0_0, \so} \bm{\mu}(\bar{s}^0_0, \so) \sum \limits_{t=0}^{\infty} \gamma^t \Prob^{\vpi}\left(s_t^0 = s^0, \s_t=\s, a_t^0 = a^0, \a_t=\a | s^0_0 = \bar{s}^0_0, \s_0=\so \right) \right) \r(s^0, \s, a^0, \a).
    $}
\end{equation*}
By replacing the term in parentheses by the occupancy measure $q_{\vpi}(s^0, \s, a^0, \a)$ in \eqref{eq:x-pi}, we got:
\begin{equation*}
    \Vo({\vpi}) = \sum_{s^0, \s, a^0, \a} q_{\vpi}(s^0, \s, a^0, \a) \r(s^0, \s, a^0, \a).
\end{equation*}
Similarly as above, we can prove equation (\ref{eq:v0pi}). This completes the proof. \hfill $\square$
\end{proof}

\subsection{Proof of Lemma \ref{thm:bar-policy}}\label{apx:bar-policy}
\setcounter{lemma}{0}
\renewcommand{\thelemma}{\arabic{lemma}}
\begin{lemma}[Uniform State-Value Representation]
If a M2WCMDP is symmetric (Definition \ref{def:symmetry-M2WCMDP}), then for any policy $\vpi$, there exists a corresponding permutation-invariant policy $\bpi$ such that the vector of expected total discounted rewards for all sub-MDPs under $\bpi$ is equal to the average of the expected total discounted rewards for each sub-MDP, i.e., $\V_0(\bpi) = \bar{V}_0(\vpi) \vone$, where $\bar{V}_0(\vpi) := \frac{1}{N} \sum_{n=1}^N V_0^n(\vpi)$, and $V^0_0(\bpi) = {V}^0_0(\vpi)$.
\end{lemma}

\begin{proof}{Proof by construction.}
First, for any fixed $Q$, we characterize its occupancy measure $q_{\ppi}$ as
\begin{equation}\label{eq:xq}
    q_{\ppi}(s^0,\s, a^0, \a) := \sum_{\bar{s}^0_0, \so} \bm{\mu}(\bar{s}^0_0, \so) \sum_{t=0}^{\infty} \gamma^{t} \Prob^{\ppi}\left(s_t^0 = s^0, \s_t=\s, a_t^0 = a^0, \a_t=\a \mid s^0_0 = \bar{s}^0_0, \s_0=\so \right).
\end{equation}

Next, we construct a new measure $\bar{q}$ obtained by averaging all permuted occupancy measures $q_{\ppi}$ for $Q \in \G$ on all $(s^0,\s, a^0, \a)$ pairs as
\begin{equation}
    \bar{q}(s^0,\s, a^0, \a) := \frac{1}{N!} \sum_{Q} q_{\ppi}(s^0,\s, a^0, \a).
    \label{eq:bar-pi-x}
\end{equation}

One can confirm that $\bar{q}$ is an occupancy measure, i.e., $\bar{q}\in\mathcal{X}$, since $\forall Q \in \G$, each $q_{\ppi}\in\mathcal{X}$ and $\mathcal{X}$ is {a convex polytope (Puterman 6.9.2). Crucially, one can show that the averaged policy $\bpi$ inherits pathwise feasibility. Because every $\ppi(\cdot \mid s^0,\s)$ only assigns probability mass within $\An(s^0, Q\s) = Q \An(s^0, \s)$, and the feasible set is permutation-invariant {(from Definition \ref{def:symmetry-M2WCMDP})}, the support of $\bpi(\cdot \mid s^0,\s)$ remains strictly inside the feasible joint action set $\An(s^0, \s)$.}

From Lemma \ref{prop:x-pi-x}, a stationary policy $\bpi$ can be constructed such that its occupancy measure matches $\bar{q}(s^0, \s, a^0, \a)$. Namely, for all $(s^0,\s, a^0, \a)$ pairs,
\begin{align*}
    q_{\bpi}(s^0,\s, a^0, \a)&:=\sum_{\bar{s}^0_0, \so} \bm{\mu}(\bar{s}^0_0, \so) \sum_{t=0}^{\infty} \gamma^{t} \Prob^{\bpi}\left(s_t^0 = s^0, \s_t=\s, a_t^0 = a^0, \a_t=\a | s^0_0 = \bar{s}^0_0, \s_0=\so \right)\\
    & = \bar{q}({s}^0,\s, a^0, \a).
\end{align*}

We can then derive the following steps:
{\begingroup \allowdisplaybreaks
\begin{align*}\label{eq:v-bar-one}
    \Vo({\bpi}) &=\sum_{s^0, \s, a^0, \a} q_{\bpi}(s^0, \s, a^0, \a) \r(s^0, \s, a^0, \a) \quad  \mbox{(By Lemma \ref{prop:xr=muv})}\\
    &=\sum_{s^0, \s, a^0, \a} \bar{q}(s^0, \s, a^0, \a) \r(s^0, \s, a^0, \a) \quad  \mbox{(By Lemma \ref{prop:x-pi-x})}\\
    &=\frac{1}{N!} \sum_{Q \in \G}
    \sum_{s^0, \s, a^0, \a} q_{\ppi}(s^0, \s, a^0, \a) \r(s^0, \s, a^0, \a) \quad  \mbox{(By construction in \eqref{eq:bar-pi-x}) }\\
    &= \frac{1}{N!} \sum_{Q \in \G}  \Vo({\ppi}) \quad  \mbox{(By Lemma \ref{prop:xr=muv})}\\
    &= \frac{1}{N!} \sum_{Q \in \G} Q^{-1} \Vo(\vpi) \quad \mbox{(By Lemma \ref{thm:piQVQ})}\\
    &= \frac{1}{N!} \sum_{Q \in \G} Q^{-1} \begin{bmatrix}
    V_0^{1}(\vpi) \\ \cdots \\ V_0^{N}(\vpi) \\\end{bmatrix} \quad \text{(Vector form)} \\
    &= \frac{1}{N!} \begin{bmatrix}
    (N-1)! \sum_{n=1}^N V_0^{n}(\vpi) \\ \cdots \\ (N-1)! \sum_{n=1}^N V_0^{n}(\vpi)\end{bmatrix} \quad \text{(Property of permutation group)}\\
    &= \frac{1}{N!} (N-1)! \sum_{n=1}^N V^{n}_0(\vpi) \vone = \frac{1}{N} \sum_{n=1}^N V_0^{n}(\vpi) \vone = \bar{V}_0(\vpi) \vone.
\end{align*} \endgroup}

Similarly, we have
\begin{align*}
    V^0_0(\bpi)  & = \sum_{s^0, \s, a^0, \a} q_{\bpi}(s^0, \s, a^0, \a) r^0 (s^0, \s, a^0, \a)\\
    & = \sum_{s^0, \s, a^0, \a} \bar{q}(s^0, \s, a^0, \a) r^0 (s^0, \s, a^0, \a)\\
    & = \frac{1}{N!} \sum_{Q \in \G} \sum_{s^0, \s, a^0, \a} q_{\vpi^Q}(s^0, \s, a^0, \a) r^0 (s^0, \s, a^0, \a) \\
    & = \frac{1}{N!} \sum_{Q \in \G} V^0_0(\vpi^Q) = \frac{1}{N!} \sum_{Q \in \G} V^0_0(\vpi) = V^0_0(\vpi),
\end{align*}
{relying on the scalar major identity $V^0_0({\ppi}) = V^0_0({\vpi})$ established in Lemma \ref{thm:piQVQ}.} \hfill $\square$
\end{proof}

\subsection{Proof of Theorem \ref{thm:eqv}} \label{apx:thm:eqv}
\setcounter{theorem}{0}
\renewcommand{\thetheorem}{\arabic{theorem}}
\begin{theorem}[Utilitarian Reduction]
{Under Assumption \ref{assum:domain}, for a symmetric M2WCMDP, let $\Pi^*_{U,\mymbox{PI}}$ be the set of optimal policies for the utilitarian approach that is permutation-invariant, then $\Pi^*_{U,\mymbox{PI}}$ is necessarily non-empty and all $\vpi^*_{U,\mymbox{PI}}\in\Pi^*_{U,\mymbox{PI}}$ satisfy $G_\rho(\vpi_{U,\mbox{\tiny{PI}}}^{*})=\max \limits_{\vpi \in \Pi}  G_\rho(\vpi).$}
\end{theorem}

\begin{proof}{Proof.}
We define the utilitarian fairness measure $\rho$ as $\rho_{\ut}[\vv] = \frac{1}{N}\sum_{n=1}^N v_n$. Let the optimal policy to the optimization problem (\ref{eq:obj-rho-M2WCMDP}) with utilitarian objective be $\vpi^*_{\ut} \in \arg \max\limits_{\vpi} G_{\ut}(\vpi)$.

With Lemma \ref{thm:bar-policy}, we can construct a permutation-invariant policy $\bar{\pi}^*_{\ut}$ satisfying 
\begin{equation} \label{eq:bar-policy-1}
    \bar{V}_0({\vpi^*_{\ut}})\vone = \V_0({\bpi^*_{\ut}})\in\gV, \quad \text{and} \quad V^0_0({\vpi^*_{\ut}}) = V^0_0({\bpi^*_{\ut}}),
\end{equation}
with $\bar{V}_0({\vpi^*_{\ut}}):=\frac{1}{N}\sum_{n=1}^N V_0^n({\vpi^*_{\ut}})$. Then we have that
\begin{equation*}
\begin{aligned}
    V^0_0({\vpi^*_{\ut}}) + \lambda \rho_{\ut}\left[\V_0({\vpi^*_{\ut}})\right] 
    &= V^0_0({\bpi^*_{\ut}}) + \lambda \rho_{\ut}\left[\bar{V_0}({\vpi^*_{\ut}})\vone\right] \\
    & = V^0_0({\bpi^*_{\ut}}) + \lambda \rho\left[\bar{V_0}({\vpi^*_{\ut}})\vone\right] \\
    & = V^0_0({\bpi^*_{\ut}}) + \lambda \rho\left[\V_0({\bpi^*_{\ut}})\right],
\end{aligned}
\end{equation*}
where we exploit $\rho[\bar{v}\vone]=\bar{v}$ for all {$\bar{v} \in \sR$ such that $\bar{v}\vone \in\gV$}.

Further, let $\vpi^*$ be any optimal policy to the $\rho$-M2WCMDP problem. One can establish that:    
\begin{equation}\label{ineq:ggf-extension}
\begin{aligned}
    V^0_0({\vpi^*}) + \lambda \rho\left[\V_0({{\vpi}^*})\right] &\geq V^0_0({\bpi^*_{\ut}}) + \lambda \rho\left[\V_0({\bpi^*_{\ut}})\right] \\
    & = V^0_0({\vpi^*_{\ut}}) + \lambda \rho_{\ut}\left[\V_0({\vpi^*_{\ut}})\right] \\
    & \geq V^0_0({\vpi^*}) + \lambda \rho_{\ut}\left[\V_0({{\vpi}^*})\right],
\end{aligned}
\end{equation}
{where the first inequality holds because $\vpi^*$ is optimal for $G_{\rho}$ and $\bpi^*_{\ut} \in \Pi$ (it is a stationary Markov policy by Lemma 1, and feasible because it inherits feasibility from $\vpi_U^*$).}

By Jensen's inequality and the fact that $\rho[\cdot]$ is concave, we have that 
    \[\rho[\vv] = \frac{1}{N!}\sum_{Q\in\G} \rho[Q\vv] \leq \rho[\frac{1}{N!}\sum_{Q\in\G} Q\vv] = \rho[\frac{1}{N}\sum_{n=1}^N v_n \vone] = \frac{1}{N}\sum_{n=1}^N v_n = \rho_{\ut}[\vv],\]
where we use permutation invariance of $\rho$, followed with its concavity and its constant vector invariance. Then, it holds that $V^0_0({\vpi^*}) + \lambda \rho_{\ut}\left[\V_0({{\vpi}^*})\right] \geq V^0_0({\vpi^*}) + \lambda \rho\left[\V_0({{\vpi}^*})\right]$. The inequalities in (\ref{ineq:ggf-extension}) should therefore all reach equality:
\begin{equation*}
    V^0_0({\vpi^*}) + \lambda \rho\left[\V_0({\vpi^*})\right] = V^0_0({\bpi^*_{\ut}}) + \lambda \rho\left[\V_0({\bpi^*_{\ut}})\right] = V^0_0({\vpi^*_{\ut}}) + \lambda \rho_{\ut}\left[\V_0({\vpi^*_{\ut}})\right].
\end{equation*}

Therefore, 
\begin{equation*}
    G_\rho(\bpi^*_{\ut}) = V_0({\bpi^*_{\ut}}) + \lambda \rho \left[\V_0({\bpi^*_{\ut}})\right] = V_0({\vpi^*_{\ut}}) + \lambda \rho_{\ut}\left[\V_0({\vpi^*_{\ut}})\right]  = {V_0({\vpi^*})} + \lambda \rho \left[\V_0({\vpi^*})\right] =\max_\pi  G_\rho(\vpi).
\end{equation*}
\hfill $\square$
\end{proof}

\section{Proof of Section \ref{sec:5policy_nn}}\label{apx:proof-sampler-properties}
This section proves the feasibility and deterministic-coverage properties of the priority-based sampler stated in Lemma \ref{lemma:sampler-properties}.

\begin{lemma}[Feasibility and deterministic coverage]
    Given that there exists a fallback action $\aFB\in\A$ such that $d_k(\aFB \mid s) = 0$ for all $s\in\S$ and all $k\in[K]$, the sampling procedure described in Algorithm \ref{alg:sampling},  has the following properties:
    \begin{enumerate}
        \item $\Prob(\u\in{\cA}(s^0,\x))=1$ for all $s^0\in\S^0$ and $\x\in\cS$.
        \item If $\{\gA_c\}_{c\in\gC}=\gA$, then for any deterministic policy $\cpi^D:\S^0\times{\cS}\rightarrow${$\{0, \dots, N\}^{|\gS| \times |\gA|}$ satisfying $\cpi^D(s^0, \x) \in\cA(s^0,\x)$}, there exists a mapping $h:\S^0\times{\cS}\rightarrow [0, 1]^K\times\Delta(\gC)^{|\S|}\times [0, 1]^{|\S|\times|\A|}$ such that a sample $\u\sim\vpi_{h(s^0,\x)}(\cdot\mid s^0,\x))= \cpi^D(s^0,\x)$ with probability one.
    \end{enumerate}
\end{lemma}

\begin{proof}{Proof.}
    The first property can be verified by following the updates of $\u$ made by the algorithm and confirming it never violates the resource consumption condition. In fact, throughout the procedure whenever $(\tilde{\vb},\u)$ are updated to $(\tilde{\vb}',\u')$, one can confirm that the updated values satisfy $\tilde{b}_k'+\sum_{s,a} d_k(a \mid s)u'(s,a)= \tilde{b}_k+\sum_{s,a} d_k(a \mid s)u(s,a)$ and that $\tilde{b}_k'\geq 0$. Since $(\tilde{\vb},\u)$ is initialized at $(\vb(s^0)\cdot\tilde{\vp},\bm{0}_{|\S|\times |\A|})$, this implies that when the procedure terminates:
    \[\tilde{b}_k+\sum_{s,a} d_k(a \mid s)u(s,a) = b_k(s^0)\tilde{p}_k \leq b_k(s^0).\]
    One can further confirm that, at termination, $\sum_{a} u(s,a)=x(s)$ for all $s$ due to the last step of the algorithm.

    The second property is achieved by setting $h(s^0,\x):={(\vone_{K},h_W(s^0,\x),\vone_{|\S|\times |\A|})}$, with $[h_W(s^0,\x)](s,c):=[\cpi^D(s^0,\x)](s,c)/\sum_a [\cpi^D(s^0,\x)](s,a)$ {for $\sum_a [\cpi^D(s^0,\x)](s,a)>0$. If the denominator equals 0, the corresponding row of $h_W(s^0,\x)$ can be chosen arbitrarily}. Fixing some $(s^0,\x)$, referring to $c$ as $a$ given the assumed partition, and using shorthand notations $h_W(s,c)$ and {$u^D(s,a)$} for $[h_W(s^0,\x)](s,c)$ and $[\cpi^D(s^0,\x)](s,a)$ respectively, one can follow the steps of the algorithm. For each $a\in\A$, the procedure first initializes $B(s,a):= \lfloor x(s)\cdot h_W(s,a) \rfloor= u^D(s,a)$ exactly since $\u^D\in{\cA}(s^0,\x)$ so that $\sum_a u^D(s,a)=x(s)$. Consequently, $R(s)=x(s)-\sum_a B(s,a)=0$, so the largest-remainder assignment is vacuous. Next given that this maximum planned budget on counts is feasible, namely $\sum_{s,a} d_k(a\mid s) B(s,a) = \sum_{s,a} d_k(a\mid s) u^D(s,a)\leq b_k(s^0)$ for all $k$, we can conclude that samples of type $(s,a)$ will be produced until the budget $B(s,a)$ runs out. Thus, the state-count plan $u(s,a)$ will increment up to and stop at $u^D(s,a)$ with probability one. {Consequently, ${\u}=\u^D =\cpi^D(s^0,\x)$ with probability one.} \hfill \Halmos
\end{proof}

\section{Count Aggregation Reformulation} \label{apx:exact-form-count-M2WCMDP}
Complementing Section \ref{sec4.2:count_mdp}, the transition kernel and the initial distribution of the count aggregation M2WCMDP $\gM_{\phi}$ under the mapping $\phi = (f, g)$ is obtained as follows.

\paragraph{Transition Probability} The minor transition probability $\Pn_{\phi}(\tilde \x \mid \x, s^0, a^0, \u)$ is the probability that the system moves from count state $\x$ to $\tilde \x$ given the major state-action pair $(s^0, a^0)$ and the minor action counts $\u$. We define the pre-image $f^{-1}(\tilde \x)$ as the set containing all labelled elements $\tilde \s \in \Sn$ that map to $\tilde \x$.

Given the equivalence of transitions within the pre-image set, for an arbitrary state-action pair $(\bar{\s}, \bar{\a}) \in \phi^{-1}(\x, \u)$, the probability of transitioning from $\x$ to $\tilde \x$ under action $\u$ is the sum of the probabilities of all individual transitions in the original space that correspond to this count state:
$$\Pn_{\phi}(\tilde \x \mid \x, s^0, a^0, \u) := \sum_{\tilde \s \in f^{-1}(\tilde \x)} \Pn(\tilde \s \mid s^0, \bar{\s}, a^0, \bar{\a}) = \sum_{\tilde \s \in f^{-1}(\tilde \x)} \prod_{n=1}^N p(\tilde s^n \mid s^0, \bar{s}^n, a^0, \bar{a}^n).$$
The major transition probability is similarly derived as $p^0_{\phi}(\tilde s^0 \mid s^0, \x, a^0, \u) := p^0(\tilde s^0 \mid s^0, \bar{\s}, a^0, \bar{\a})$.

\paragraph{Initial Distribution} By using a state count representation for the symmetric M2WCMDP, we know that $\sum_{s \in \S} x(s) = N$. The cardinality of the pre-image set $f^{-1}(\x)$ can be obtained through the multinomial expansion of $(s_1 + s_2 + \dots + s_{|\S|})^N$. Intuitively, distributing $N$ identical sub-MDPs into $|\S|$ distinct states such that the counts match $\x$ is given by the multinomial coefficient $|f^{-1}(\x)| = \frac{N!}{\prod_{s\in\S} x(s)!}$. Given that the joint initial distribution $\bm{\mu}(s^0, \s)$ is permutation-invariant, the pushforward probability of starting from major state $s^0$ and minor count state $\x$ is $$\bm{\mu}_f(s^0, \x) := \sum_{\s \in f^{-1}(\x)} \bm{\mu}(s^0, \s) = |f^{-1}(\x)| \cdot \bm{\mu}(s^0, \bar{\s}) = \frac{N!}{\prod_{s\in\S} x(s)!} \cdot \bm{\mu}(s^0, \bar{\s}),$$ for any arbitrary $\bar{\s}$ such that $f(\bar{\s}) = \x$.

\section{Exact Approaches for Solving Utilitarian-Reduced Count M2WCMDP} \label{apx:dual-lp}
Based on the utilitarian reduction established in Theorem \ref{thm:eqv}, the symmetric $\rho$-M2WCMDP problem can be solved through its utilitarian-reduced count aggregation M2WCMDP $\gM_\phi$ (see Definition \ref{def:count-MDP}). Using the exact form constructed in \ref{apx:exact-form-count-M2WCMDP} and following the standard dual LP methods for discounted MDPs derived in Section 6.9.1 by \cite{puterman2014markov}, we define the discounted state-action occupancy measures $q_{\phi}(s^0, \x, a^0, \u) \geq 0$ over the aggregated joint state space $\mathcal{S}^0 \times \cS$ and {the state-dependent feasible joint action space $\A^0 \times \cA(s^0, \x)$}.

Under the utilitarian reduction, the objective is to maximize the expected total discounted sum of the {platform reward $r^0_{\phi}$ and the fairness-weighted mean participant reward $\lambda \bar{r}_{\phi}$}. The exact LP is formulated as follows:
\begin{equation}\label{eq:utilitarian-reduced-lp}
    \begin{array}{rl}
    \max\limits_{q_\phi} &  \sum \limits_{s^0 \in \mathcal{S}^0, \x \in \cS} \sum \limits_{a^0 \in \A^0, \u \in \cA(s^0, \x)} \left[r^0_{\phi}(s^0,\x,a^0,\u)+\lambda\bar r_{\phi}(s^0,\x,a^0,\u)\right] q_{\phi}(s^0, \x, a^0, \u) \\
    \textit{s.t.} & \sum\limits_{a^0 \in \A^0, \u \in \cA(s^0, \x)} q_{\phi}(s^0, \x, a^0, \u) \\
    & \quad -  \gamma \hspace{-5pt} \sum \limits_{\tilde s^0 \in \mathcal{S}^0, \tilde \x \in \cS} \sum \limits_{\tilde a^0 \in \A^0, \tilde \u \in \cA(\tilde s^0, \tilde \x)} \hspace{-3pt} q_{\phi}(\tilde s^0, \tilde \x, \tilde a^0, \tilde \u) p^0_{\phi}(s^0 \mid \tilde s^0, \tilde \x, \tilde a^0, \tilde \u) \Pn_{\phi}(\x \mid \tilde s^0, \tilde \x, \tilde a^0, \tilde \u) \\ 
    &\quad = \bm{\mu}_f(s^0, \x), \qquad \forall s^0 \in \mathcal{S}^0, \x \in \cS \\
    & q_{\phi}(s^0, \x, a^0, \u) \geq 0 \qquad \forall s^0 \in \mathcal{S}^0, \x \in \cS,  \forall a^0 \in \A^0, \u \in \cA(s^0, \x)
\end{array}
\end{equation}

Because $\bm{\mu}_f$ is the pushforward of the original probability distribution $\bm{\mu}$, it satisfies $\sum_{s^0\in\S^0, \, \x\in\cS}\bm{\mu}_f(s^0,\x)=1$. For an optimal solution $q_{\phi}^*(s^0, \x, a^0, \u)$, an optimal count policy is recovered as
    $$\vpi_\phi^*(a^0, \u \mid s^0, \x) := \frac{q_\phi^*(s^0, \x, a^0, \u)}{\sum_{\tilde{a}^0, \tilde{\u}} q_\phi^*(s^0, \x, \tilde{a}^0, \tilde{\u})}, \quad \forall s^0, \x, a^0, \u, $$
for the utilitarian-reduced problem such that $\sum_{\tilde{a}^0, \tilde{\u}} q_\phi^*(s^0, \x, \tilde{a}^0, \tilde{\u}) > 0$, with the policy chosen arbitrarily on unreachable states. Let $\Gamma(s^0, \s, \u) := \left\{ \a \in \An(s^0, \s) \mid g(\s, \a) = \u \right\}$ denote the set of feasible labelled actions inducing the count action $\u$. The resulting count policy can be lifted to a permutation-invariant policy in the original labelled M2WCMDP by assigning count actions among exchangeable participants through a permutation-invariant disaggregation rule that distributes the probability mass equally
\begin{equation}\label{eq:lift-to-labelled-pi-policy}
    \vpi_{\mymbox{\textit{PI}}}(a^0, \a \mid s^0, \s) = \frac{\vpi_\phi^*(a^0, g(\s, \a) \mid s^0, f(\s))}{|\Gamma(s^0, \s, g(\s, \a))|}, \qquad \a \in \Gamma(s^0, \s, g(\s,\a)).
\end{equation}

By construction, $\vpi_{\mymbox{\textit{PI}}}$ is permutation-invariant, i.e., $\vpi_{\mymbox{\textit{PI}}}(a^0, Q\a \mid s^0, Q\s) = \vpi_{\mymbox{\textit{PI}}}(a^0, \a \mid s^0, \s)$ for any permutation matrix $Q$. This explicit exchangeable disaggregation prevents systematic discrimination during execution, addressing the inherent biases of arbitrary deterministic tie-breaking rules. By Theorem \ref{thm:eqv}, the lifted policy $\vpi_{\mymbox{\textit{PI}}}$ is therefore optimal for the original $\rho$-M2WCMDP for any $\rho$ satisfying Definition \ref{def:rho-fairness} under the transformation $\phi$.

\section{Additional Details and Experiments of Machine Replacement Problem} \label{apx:details-MRP}

\subsection{Problem Instance Generation}\label{apx:instance-generation-mrp}
This section details the construction of the components used to generate the test instances based on \cite{akbarzadeh2019restless}, including the cost function, the transition matrix, and the reset probability. This experiment uses a synthetic data generator implemented on our own that considers a system with $S$ states and binary actions ($A=2$), where the two possible actions are to operate or to replace. After generating the per-machine state-action cost matrix of size $N \times S \times A$, we normalize the costs to the range [0, 1] by dividing each entry by the maximum cost over all state-action pairs. This ensures that the discounted return always falls within the range $[0, \frac{1}{1-\gamma}]$.

\paragraph{Cost Function} The cost function $c(s)$ for $s \in [S]$ can be defined in four ways: 1) \textit{Linear}: $c(s)=s-1$, where the cost increases linearly with the state index; 2) \textit{Quadratic}: $c(s) = (s - 1)^2$ with a more severe penalty for higher states compared to the linear case; 3) \textit{Exponential}: $c(s)=e^{s-1}$, which leads to exponentially increasing costs; 4) \textit{Replacement Cost Constant Coefficient (RCCC)}: $c(s) = 1.5 (S - 1)^2$, which is based on a constant ratio of 1.5 to the maximum quadratic cost. These components define the state-action cost $c(s,a)$ in Section \ref{sec:mrp-expt}.

\paragraph{Transition Function} The transition matrix for the deterioration action is constructed as follows. Once the machine reaches the $S$-th state, it remains in that worst state indefinitely until being successfully reset by a replacement action. For the $s$-th state $s \in [S-1]$, the probability of remaining in the same state at the next step is given by a model parameter $p_m \in [0, 1]$, and the probability of transitioning to the $(s+1)$-th state is $1-p_m$.

\paragraph{Reset Probability} When a replacement occurs, there is a probability $p_s$ that the machine successfully resets to the first state, and a corresponding probability $1-p_s$ of failing to be repaired and following the deterioration rule. In our experiments, we only consider a pure reset to the first state with probability 1.

\subsection{Hyperparameters} \label{apx:hyperparameter-mrp}
In our experimental setup, we chose the Proximal Policy Optimization (PPO) algorithm to implement the count-proportion-based deep reinforcement learning (CP-DRL) architecture. The hidden layers are fully connected and the Tanh activation function is used. There are two layers, with each layer consisting of 64 units. The learning rate for the actor is set to $5 \times 10^{-4}$ and the critic is set to $3 \times 10^{-4}$. {The Vanilla-DRL baseline uses the same network architecture as PPO, with two hidden layers of 64 neurons each, but with a softmax output layer for its stochastic policy.}

\subsection{Additional Experiments} \label{apx:additional-expt-mrp}
This subsection complements the small-instance optimality comparisons in Section \ref{sec:mrp-expt}.

\paragraph{Scalability} We assess CP-DRL scalability by increasing the number of machines while keeping the resource proportion at 0.1 for the \textit{Exponential-RCCC} instances. We refer to this scaled (SC) extension as CP-DRL(SC). We vary the number of machines from 10 to 100 to evaluate CP-DRL performance as the problem size grows. We also use CP-DRL(SC), trained on 10 machines with 1 unit of resource, and scale it to tasks with 20 to 100 machines. Figure \ref{fig:combined}a shows CP-DRL and CP-DRL(SC) consistently achieve higher GGF values than Whittle index policy (WIP) as machine numbers increase. CP-DRL(SC) delivers results comparable to separately trained CP-DRL, reducing training time while maintaining similar performance. Both WIP and CP-DRL show linear growth in time consumption per episode as machine numbers scale up.

\begin{figure}[htb] \FIGURE
    {
    \subcaptionbox[size=small]{GGF values for the number of machines $N \in [10, 100]$}
    {\includegraphics[width=0.32\textwidth]{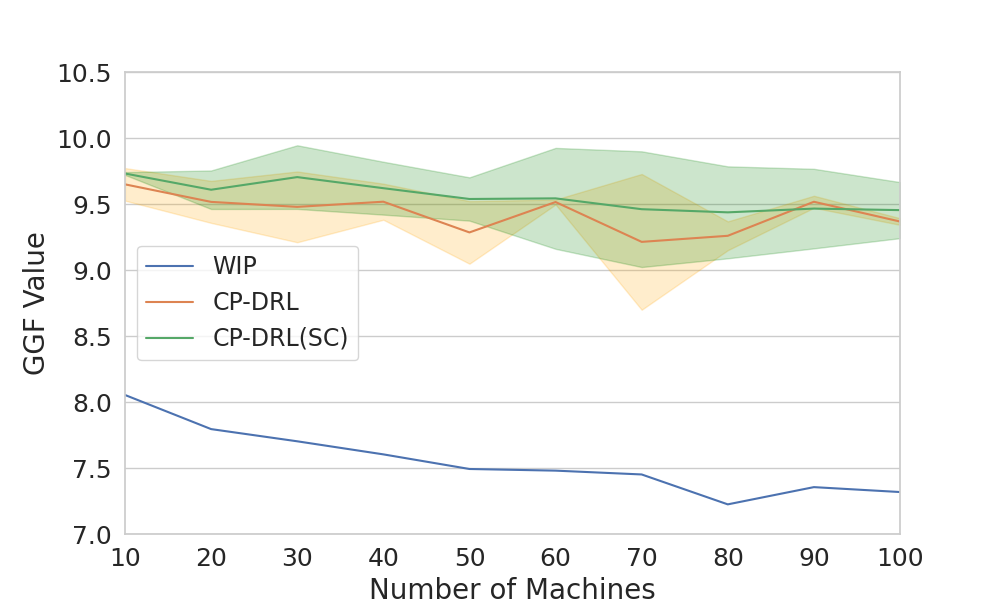}}\hfill
    \subcaptionbox[size=small]{Time per episode in seconds with a resource ratio $b/N =0.1$}
    {\includegraphics[width=0.32\textwidth]{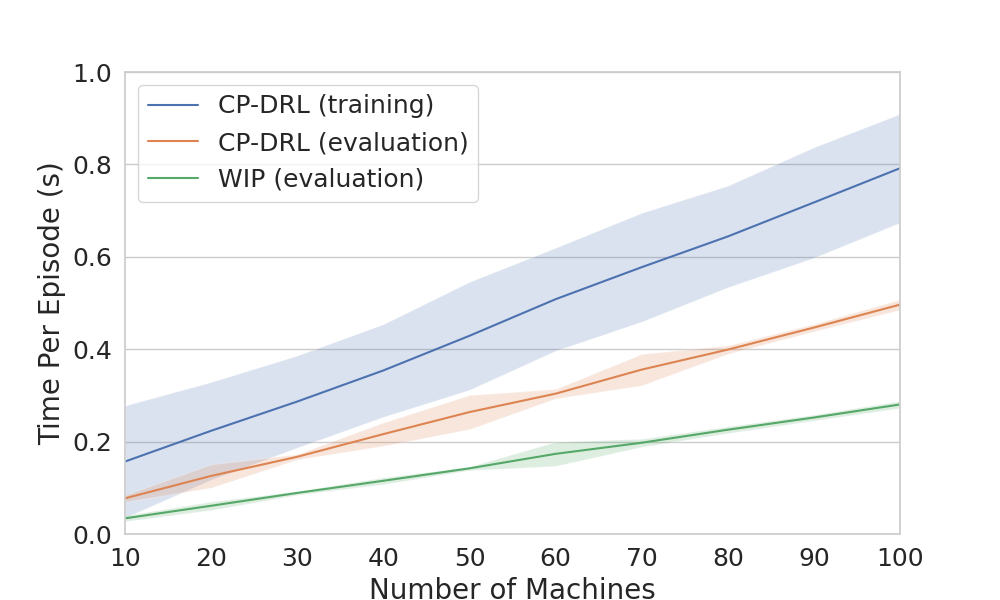}}\hfill
    \subcaptionbox[size=small]{Time per episode in seconds with a resource ratio $b/N =0.5$}
    {\includegraphics[width=0.32\textwidth]{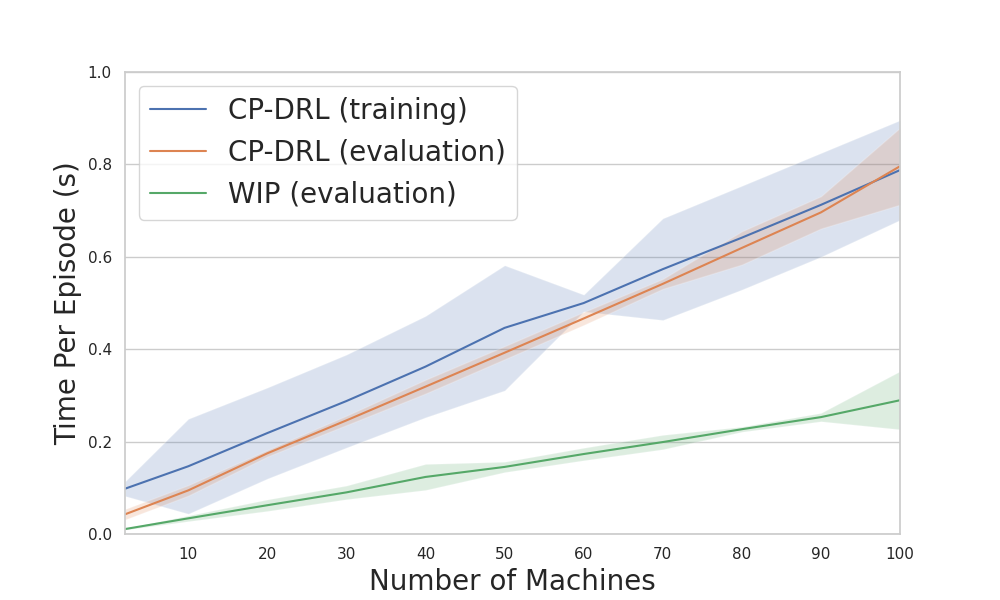}}
    }
    {(Color online) Scalability and Time Efficiency of CP-DRL \label{fig:combined}}
    {Subfigure (a) shows the scalability of CP-DRL with a fixed resource ratio of 0.1. Subfigure (a) presents GGF values across different machine counts, with intervals representing the standard deviation over 5 runs. Subfigure (b) and (c) depict time per episode in seconds for a fixed resource ratio of 0.1 and 0.5, respectively. In all time plots, the green line represents WIP during MC evaluation, the blue line shows CP-DRL during training, and the orange line represents CP-DRL during MC evaluation.}
\end{figure}

\paragraph{Efficiency} Using the count dual LP model (\ref{eq:utilitarian-reduced-lp}) reduces the model size, but constraints still grow as $\binom{N+S-1}{S-1}$ and variables increase by $\binom{N+S-1}{S-1} \cdot A$. These growth patterns create computational challenges as the problem size increases. In addition to the time per episode for a fixed ratio of 0.1 in Figure \ref{fig:combined}b, we analyze performance with a 0.5 ratio (Figure \ref{fig:combined}c) and varying machine proportions, keeping the number of machines fixed at 10. We evaluate CP-DRL over machine proportions from 0.1 to 0.9. The results show that the time per episode increases linearly with the number of machines, while training and evaluation times remain relatively stable. This indicates that the sampling procedure for legal actions is the primary bottleneck. Meanwhile, the resource ratio has minimal impact on computing times.

\section{Additional Details and Results of Taxi Dispatching and Pricing Problem}\label{apx:details}
We report additional details in the numerical studies. All experiments were conducted on the Rorqual cluster provided by the Digital Research Alliance of Canada, utilizing Python 3.10.13 and Gurobi 13.0.0. To manage the extensive parameter testing, we utilize Slurm job arrays for parallel execution. Each task was assigned 1 CPU core and 1.0 GB of RAM.

\subsection{Data Source, Preprocessing, and Calibration} \label{apx:nyc-data}
We calibrate the fundamental vehicle speed based on \cite{huang2018taxi}. The average occupancy speed is approximately set at 20 km per hour. Since that explicit empty-cruising data is not provided, we assume that the empty speed is identical to the occupancy speed. Consequently, in a single time step $\Delta t$, a vehicle travels a distance of 1 km. The time window is fixed for 10 hours per day, corresponding to $T=$ 200 time steps.

The reward function approximates the New York City (NYC) taxi fare structure with a base fare and a distance-based term. We aggregate the fixed costs (base fare \$3.00, metropolitan transportation authority surcharge \$0.50, and improvement surcharge \$1.00) into a constant intercept $c_f:= \$4.5$. Based on \cite{parrott2024tlc}, the variable reward proportional to the distance traveled is approximated as $c_r :=\$2.17$ per km. The operational cost includes fuel and maintenance costs, and calibrated at $c_o:=\$0.54$ per km. Further, to model the impact of dynamic pricing on customer decisions, we adopt a log response function $\psi_{log}(p):=e^{-\eta(p-1)}$ with an elasticity parameter $\eta$ := 0.51 following empirical studies on ride-sharing platforms \citep{cohen2016using}. The price range is set as $p \in \gP := (0, 10]$. {We write $p_{\min}:=\inf\gP$ and $p_{\max}:=\sup\gP$ for the price bounds used in \ref{apx:hyperparameter-taxi}.}

\begin{figure}[htb]\FIGURE
    {\includegraphics[width=\textwidth]{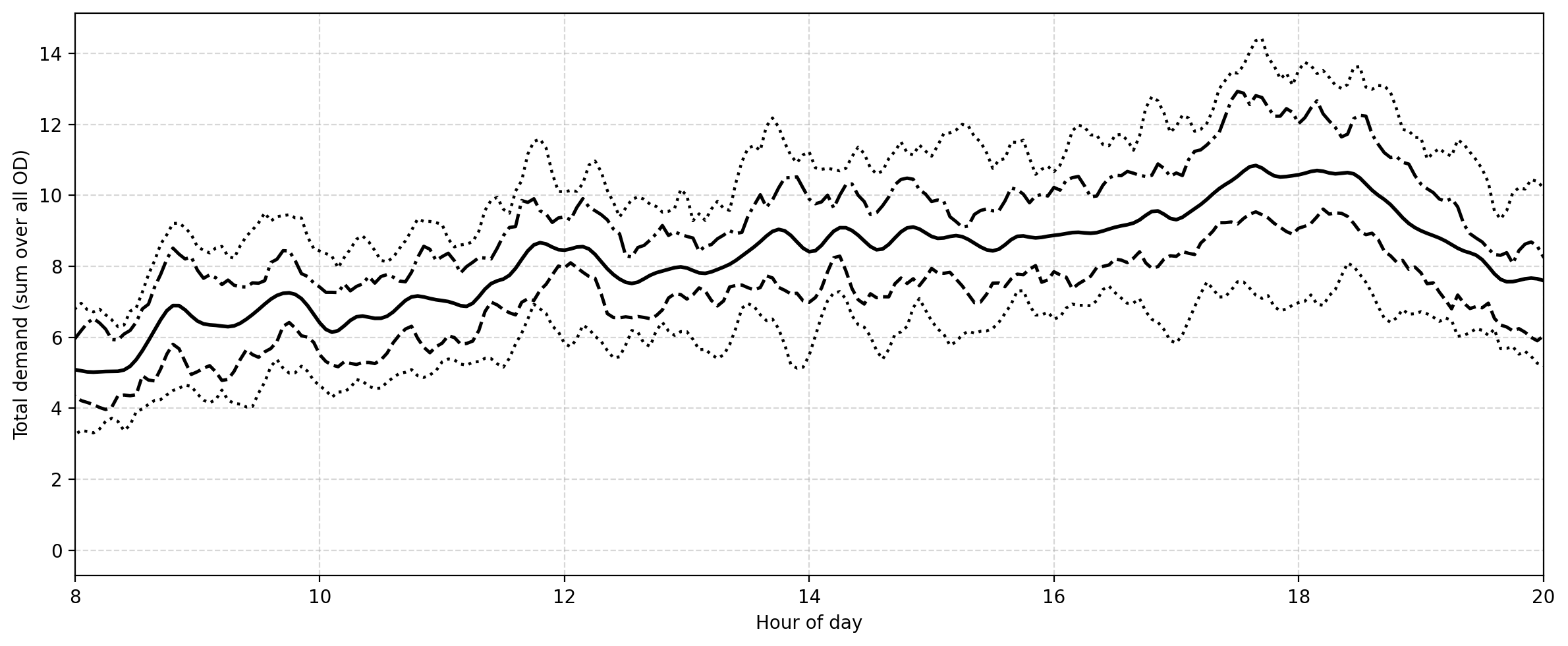}}
    {Smoothed Aggregate Trip Arrivals for the Four-Zone Midtown Network \label{fig:arrival}} 
    {The solid line represents the Gaussian-smoothed average demand across Tuesdays from July to November 2025, while the dashed lines mark the 25th and 75th percentiles, and the dotted lines mark the 10th and 90th percentiles. The horizontal axis is hour of day and each observation corresponds to a three-minute resolution.} 
\end{figure}

\paragraph{Arrival Rate} For empirical validation, we utilize historical trip data from July to November 2025, and focus on Tuesday from 8:00 a.m. to 8:00 p.m. to capture the representative transition between peak and off-peak demand patterns. We model order arrivals as an inhomogeneous Poisson process with time-varying rates $\boldsymbol{\theta}^{\text{base}}_t$. To estimate a stable empirical base-rate profile rather than fitting high-frequency sampling noise, we first compute mean arrivals across the selected Tuesdays and then apply a Gaussian smoothing filter \citep{lehky2010decoding} with $\sigma := 2.0$, approximately spanning a typical city traffic changing window of $\pm 15$ minutes to the raw count data \citep{moyano2021traffic}. This filter is an implementation choice for estimating the base demand $\hat{\boldsymbol{\theta}}^{\text{base}}_t$, while reducing variance for preprocessing. Figure \ref{fig:arrival} illustrates the summary statistics for the total demand originating from the four-zone Midtown network on Tuesday (see zone details in \ref{apx:graph}). The plot reveals a steady increase in demand intensity throughout the day, peaking during the evening commute.

\paragraph{Network Construction and Simulation Environment} \label{apx:graph} We consider two subgraphs of increasing size, each induced from the same NYC zone graph as shown in Figure \ref{fig:selected}. Sub-figure \ref{fig:selected}a comprises 4 nodes covering the densest commercial regions in Midtown Manhattan for mechanism analysis of the joint pricing-and-routing decisions. Sub-figure \ref{fig:selected}b expands to 15 nodes spanning Midtown, Upper Manhattan, and Central Park, etc., and is mainly used for scalability and real-world performance evaluation. Detailed region names are provided in Table \ref{tab:node}.

\begin{figure}[htb]\FIGURE
    {\includegraphics[width=0.8\textwidth]{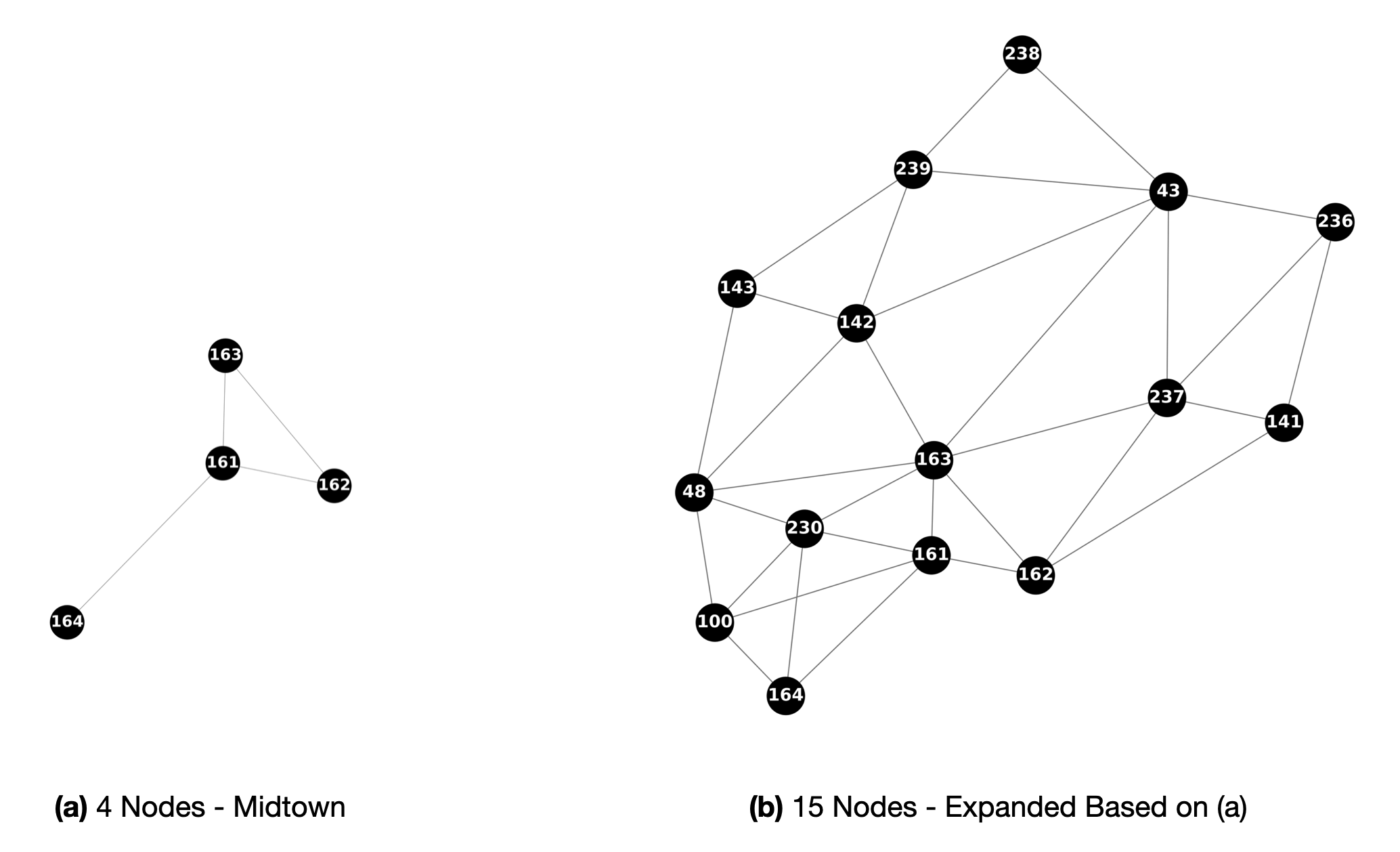}}
    {Selected NYC Sub-Networks \label{fig:selected}} 
    {Sub-figure (a) is used for mechanism analysis, and (b) for scalability and NYC-calibrated performance evaluation. Node labels are taxi-zone IDs. The corresponding zone names are listed in Table \ref{tab:node}.} 
\end{figure}

\begin{table}[htb]\TABLE
{Node Selection and Corresponding Zone Names \label{tab:node}}{
    \begin{tabular}{ccp{3.8cm}p{9.2cm}}
    \hline
    {Graph} & {Count} & {Selected Zone IDs} & {Zone Names} \\ \hline
    {a} & 4 & 161, 162, 163, 164 & 
    161: Midtown Center, 162: Midtown East, 163: Midtown North, 164: Midtown South \\
    
    {b} & 15 & 43, 48, 100, 141, 142, 143, 161, 162, 163, 164, 230, 236, 237, 238, 239 & 
    \textit{Graph a nodes, plus:} \newline 
    43: Central Park, 48: Clinton East, 100: Garment District, 141: Lenox Hill West, 142: Lincoln Square East, 143: Lincoln Square West, 230: Times Square/Theatre District, 236: Upper East Side North, 237: Upper East Side South, 238: Upper West Side North, 239: Upper West Side South \\ \hline
    \end{tabular}}
{}
\end{table}

\subsection{Graph Neural Network Architecture}\label{apx:gnn}
We encode the spatial relationships among zones using a graph neural network (GNN) that operates on the augmented major state $s^0 = (\vq, \vpLag, t)$ and the count state $\x$. Here, $\vq\in\sR^{M\times M}$ is the current origin-destination (OD) request queue, $\vpLag\in\sR^{M\times M}$ is the lagged price-multiplier matrix, and $\x\in\sR^{M\times(D_{\max}+1)}$ is the zone-level taxi-count feature.

For each edge-valued input $\bm{e} \in \{\vq, \vpLag\}$, we summarize the local edge structure at every zone $i$ by its mean outgoing and incoming edge values over the adjacency matrix $\mAdj$:
\begin{equation*}
    \resizebox{\linewidth}{!}{$
    e_{\operatorname{out}}(i) := \frac{1}{\max(1, \sum\limits_{k} \Adj(i, k))}\sum_{j=1}^{M} e(i, j) \Adj(i, j), \quad
    e_{\operatorname{in}}(i)  := \frac{1}{\max(1,\sum\limits_{k} \Adj(k, i))} \sum_{j=1}^{M} e(j, i) \Adj(j, i).$}
\end{equation*}

The pair $[e_{\operatorname{in}}(i), e_{\operatorname{out}}(i)] \in \sR^2$ is concatenated to form a compact summary of the underlying edge feature. Each channel is then independently mapped into a common $d_h$-dimensional node space via a learned linear projection $\mW$ followed by a ReLU activation layer:
\begin{equation*}
    \vh^{q}_i := \operatorname{ReLU}\left(\mW_{q}[\vq_{\operatorname{in}}(i) \| \vq_{\operatorname{out}}(i)]\right), \,
    \vh^{p}_i := \operatorname{ReLU}\left(\mW_{p}[\vp_{\operatorname{in}}(i) \| \vp_{\operatorname{out}}(i)]\right), \,
    \vh^{x}_i := \operatorname{ReLU}\left(\mW_{x}\x(i)\right),
\end{equation*}
where $\mW_{q}, \mW_{p} \in \sR^{d_h \times 2}$, $\mW_{x} \in \sR^{d_h \times (D_{\max}+1)}$ are learnable weight matrices.

The three per-node embeddings are concatenated and passed through a shared linear layer to produce a single fused node representation before any message passing occurs:
\begin{equation*}
    \vh_i^{(0)} := \operatorname{ReLU}\left(\mW_{\operatorname{fuse}}[\vh^{q}_i \| \vh^{p}_i \| \vh^{x}_i]\right), \quad \mW_{\operatorname{fuse}} \in \sR^{d_h \times 3d_h}.
\end{equation*}

Fusing the latent features at the node level prior to propagation ensures that the graph convolution layers capture interactions among supply, demand, and pricing signals jointly, rather than aggregating spatially disjoint representations. Then, the network applies $L$ graph convolution layers over the adjacency matrix $\mAdj$. At each layer $\ell = 0, \ldots, L-1$, the embedding of zone $i$ is updated by separately transforming the mean neighborhood embedding and the zone's self-embedding:
\begin{equation} \label{eq:gnn_layer}
    \vh_i^{(\ell+1)} := \operatorname{ReLU}\left(\mW^{(\ell)}{\left(\frac{1}{\max(1,\sum\limits_{k} \Adj(i, k))}\sum_{j=1}^{M} \Adj(i, j)\vh_j^{(\ell)}\right)} + \bm{B}^{(\ell)}\vh_i^{(\ell)}\right),
\end{equation}
where $\mW^{(\ell)}, \bm{B}^{(\ell)} \in \sR^{d_h \times d_h}$ are layer-specific learnable parameters. The mean aggregation in Equation \ref{eq:gnn_layer} normalizes by the degree of zone $i$, making the update invariant to network size. The temporal context $t$ is encoded through a Time2Vec representation $\bm{e}_t \in \sR^{d_t}$ \citep{kazemi2019time2vec} with both sinusoidal and cosinusoidal embeddings, and appended to the output only after spatial propagation is complete. The final feature vector passed to the downstream policy and value heads is $\bm{z} := [\vh^{(L)} \| \bm{e}_t] \in \sR^{M d_h + d_t}$, where $\vh^{(L)} := [\vh_1^{(L)}, \ldots, \vh_M^{(L)}]^\top \in \sR^{M \times d_h}$ collects the final node embeddings across all zones. Putting temporal context at the output rather than at intermediate layers prevents the time signal from interfering with the spatial message-passing dynamics, while still allowing downstream heads to condition their decisions on the current period.

\subsection{Hyperparameters and Reproducibility Protocol} \label{apx:hyperparameter-taxi}
In this section, we report the training configuration required to reproduce the taxi experiments in Section \ref{sec:taxi-expt}. Table \ref{tab:taxi-ppo-hyperparameters} summarizes the hyperparameters used for the GNN feature extractor and PPO.

\begin{table}[htb]\TABLE
    {PPO and Simulation Settings for the Taxi Experiments \label{tab:taxi-ppo-hyperparameters}}
    {\begin{tabular}{p{5.0cm}p{8.0cm}}
    \hline
    Setting & Value \\
    \hline
    Rollout length & $n_{\text{steps}}=2048$ \\
    Mini-batch size & 128; 256 for the 15-zone configuration \\
    Optimization epochs & 5 \\
    Discount factor & $\gamma$ = 0.95\\
    Fairness preference factor & $\lambda$ = 1\\
    Learning rate and optimizer & $3\times10^{-4}$; Adam \\
    Target KL divergence & 0.03; 0.05 for the 15-zone configuration \\
    Episode length & $T$ = 200 time steps \\
    Training episodes & 20,000 episodes; 40,000 for the 15-zone configuration \\
    GNN hidden dimension & $d_h$ = 32 \\
    GNN time embedding dimension & $d_t$ = 32 \\
    GNN feature dimension & 128 \\
    GNN layers & 2; 3 for the 15-zone configuration \\
    \hline
    \end{tabular}}
    {}
\end{table}

\subsubsection{Normalization}
The observation contains the taxi counts $\bm{x}$, order queue $\vq$, and normalized time $t/T$. Taxi counts are divided by the fleet size $N$, and the order queue is divided by $\max\{\sum_{i,j \in \gM}q(i,j),1\}$. The per ride reward is divided by $\max \{c_o D_{\max}, \, |c_f + c_rD_{\max} p_{\max} - c_o D_{\max}|, \, 1\}$ to approximately fall between the range $[-1, 1]$ without changing its sign.

\subsubsection{Structured Low-Rank Pricing} \label{apx:low-rank-pricing}
For OD pricing, the policy head outputs origin and destination latent factors $\mL^{\text{O}},\mL^{\text{D}}\in[0,1]^{M\times r_p}$, where $r_p$ is the chosen rank. Define their row averages by $\ell^{\text{O}}(i):=r_p^{-1}\sum_{q=1}^{r_p}L^{\text{O}}(i,q)$ and $\ell^{\text{D}}(j):=r_p^{-1}\sum_{q=1}^{r_p}L^{\text{D}}(j,q)$. The OD price multiplier is $p^{\text{OD}}(i,j):=p_{\min}+{(1/2)}(p_{\max}-p_{\min})\left(\ell^{\text{O}}(i)+\ell^{\text{D}}(j)\right)$, with $a^0:=\vp^{\text{OD}}$, i.e., the normalized OD score matrix is additively separable into origin and destination effects. Origin-only pricing uses $p^{\text{O}}(i)=p_{\min}+(p_{\max}-p_{\min})\ell^{\text{O}}(i)$, with $a^0:=\vp^{\text{O}}\vone_M^T$. Uniform pricing takes the scalars $\ell^{\text{O}}, \ell^{\text{D}}$ and averages the scores $p^{\text{U}}=p_{\min}+(p_{\max}-p_{\min})(\ell^{\text{O}} + \ell^{\text{D}})/2$ to obtain a single multiplier, with $a^0:=p^{\text{U}} \vone_{M \times M}$. We set $r_p=4$ for the four-zone experiments and $r_p=10$ for the 15-zone experiments.

\subsection{Benchmark Details}\label{apx:benchmark}
The benchmarks use the same simulator dynamics, economic parameters, episode horizon, and count-action feasibility checks as the CP-DRL.

\paragraph{Count-Proportional Assignment Linear Program (CP-ALP)} \label{apx:alp}
We solve a myopic taxi-assignment model per-step. Notations match the taxi count model in Example \ref{exmp:taxi-dispatching-count}, with the additional variable ${\nu}(i,j) \in \mathbb{Z}_+$ representing the number of taxis that are relocated to $i$, which are still in-transit but assumed to immediately serve requests to $j$.

We first consider the case (i), in which there are more taxis than orders,
\begin{equation}
    \begin{array}{rl} 
     \max \limits_{\u^1, \u^2, \bm{\nu} \geq \bm{0}} &  \sum \limits_{i\in\gM} \sum\limits_{j\in\gM} \left(\left[c_f+c_r\pLag(i,j)D(i,j)\right]u^2(i,j)- c_oD(i,j)\left[u^1(i,j)+u^2(i,j)\right]\right)\\
     \text{s.t.} & \sum\limits_{j\in\gM} u^1(i, j) + \sum\limits_{j \in \gM} u^2(i, j) = x(i, 0), \hspace{30pt} \forall i \in \gM, \\
     & \sum\limits_{k \in \gM} u^1(k, i) \geq \sum\limits_{j \in \gM} {\nu}(i, j),  \hspace{76pt} \forall i \in \gM, \\
     & {\nu}(i, j) + u^2(i, j) \geq q(i, j), \hspace{78pt} \forall i, j \in \gM, \\
     & u^2(i, j) \leq q(i, j),  \hspace{120pt} \forall i, j \in \gM.
\end{array}
\end{equation}

The first constraint is the flow conservation of idle taxis at each origin that each available taxi is either relocated or directly assigned. The second represents that the total number of orders immediately fulfilled at $i$ cannot exceed the number of taxis relocated into $i$. The third guarantees that all customer requests are satisfied, either by direct assignment $\u^2$ or by post-relocated taxis $\bm{\nu}$. The fourth  constrains the order capacity.

We next consider case (ii), in which the number of taxis is smaller than the number of orders:

\begin{equation}
    \begin{array}{rl} 
     \max \limits_{\u^1, \u^2, \bm{\nu} \geq \bm{0}} &  \sum \limits_{i\in\gM} \sum\limits_{j\in\gM} \left(\left[c_f+c_r\pLag(i,j)D(i,j)\right]u^2(i,j)- c_oD(i,j)\left[u^1(i,j)+u^2(i,j)\right]\right)\\
     \text{s.t.} & \sum\limits_{j\in\gM} u^1(i, j) + \sum\limits_{j \in \gM} u^2(i, j) = x(i, 0), \hspace{30pt} \forall i \in \gM, \\
     & \sum\limits_{k \in \gM} u^1(k, i) \leq \sum\limits_{j \in \gM} {\nu}(i, j),  \hspace{76pt} \forall i \in \gM, \\
     & {\nu}(i, j) + u^2(i, j) \leq q(i, j), \hspace{78pt} \forall i, j \in \gM.
\end{array}
\end{equation}

\paragraph{Random (RDM)} At each step, the entries of the priority-score matrix are sampled independently from a uniform distribution $[0, 1]$.

\paragraph{Lowest Income First (LIF)} All idle taxis are globally sorted by accumulated discounted reward, with taxi index used only for deterministic tie breaking. Each taxi is then considered based on the sorted order. If there are outstanding requests at the current origin, the chosen taxi is matched to the destination maximizing current trip net reward. If no request is available at the current origin, the taxi remains in place. LIF neither relocates taxis proactively nor optimizes price.

\paragraph{Multilayer Perceptron (MLP)} The configuration is kept the same as CP-DRL, but replaces graph message passing by 2-layer fully connected feature processing with a hidden dimension of 128.

\subsection{Additional Results on Fairness Preference Factor} \label{apx:additional-expt-taxi}
Table \ref{tab:nyc_calibrated_lambda_0_objectives} reports a robustness check with values of platform components and percentages to the objectives by GGF, complementing the fairness-aware taxi evaluations in Table \ref{tab:nyc_objectives}. For OD pricing, the CP-DRL advantage over CP-ALP ranges from approximately 37.6\% at 10 taxis to 54.5\% at 50 taxis, showing that the dynamic advantage is not driven only by the fairness term. Across CP-DRL variants, OD pricing is generally strongest. This check indicates that CP-DRL's gains arise from intertemporal coordination of pricing and dispatch rather than from the particular welfare weight alone.

\begin{table}[htb]\TABLE
    {Platform Value and the Percentage to the Fairness-Aware Objective \label{tab:nyc_calibrated_lambda_0_objectives}}{
    \begin{tabular}{cccccc}
    \hline
    \# Taxis & Granularity & CP-DRL & \% & CP-ALP & \% \\
    \hline
    \multirow{3}{*}{10} & OD & $24.89 \pm 0.51$ & $35.23\%$ & $18.09 \pm 0.35$ & $35.79\%$ \\
     & O & $23.24 \pm 0.52$ & $35.59\%$ & $15.65 \pm 0.22$ & $35.91\%$ \\
     & U & $24.38 \pm 0.22$ & $35.28\%$ & $10.77 \pm 0.32$ & $36.43\%$ \\
    \hline
    \multirow{3}{*}{20} & OD & $31.97 \pm 0.48$ & $52.66\%$ & $21.11 \pm 0.37$ & $54.15\%$ \\
     & O & $29.57 \pm 0.51$ & $52.88\%$ & $18.66 \pm 0.47$ & $55.01\%$ \\
     & U & $30.53 \pm 0.43$ & $52.77\%$ & $12.93 \pm 0.33$ & $54.92\%$ \\
    \hline
    \multirow{3}{*}{30} & OD & $33.51 \pm 0.77$ & $63.45\%$ & $21.96 \pm 0.35$ & $65.59\%$ \\
     & O & $31.19 \pm 0.49$ & $63.24\%$ & $19.27 \pm 0.39$ & $65.80\%$ \\
     & U & $31.97 \pm 0.78$ & $63.35\%$ & $13.18 \pm 0.34$ & $66.14\%$ \\
    \hline
    \multirow{3}{*}{40} & OD & $33.63 \pm 0.48$ & $70.50\%$ & $22.01 \pm 0.48$ & $72.36\%$ \\
     & O & $31.97 \pm 0.46$ & $70.24\%$ & $19.72 \pm 0.33$ & $72.51\%$ \\
     & U & $32.11 \pm 1.39$ & $70.25\%$ & $13.24 \pm 0.34$ & $73.69\%$ \\
    \hline
    \multirow{3}{*}{50} & OD & $34.56 \pm 0.53$ & $75.13\%$ & $22.37 \pm 0.43$ & $77.25\%$ \\
     & O & $32.42 \pm 0.50$ & $75.11\%$ & $19.87 \pm 0.26$ & $78.13\%$ \\
     & U & $32.28 \pm 1.68$ & $74.95\%$ & $13.59 \pm 0.32$ & $78.87\%$ \\
    \hline
    \multirow{3}{*}{60} & OD & $34.87 \pm 0.45$ & $78.82\%$ & $22.70 \pm 0.37$ & $81.14\%$ \\
     & O & $32.49 \pm 0.80$ & $78.42\%$ & $19.76 \pm 0.35$ & $81.10\%$ \\
     & U & $32.38 \pm 2.24$ & $78.55\%$ & $13.44 \pm 0.27$ & $82.12\%$ \\
    \hline
    \end{tabular}}
{\textit{Note.} Entries are mean $\pm$ sample standard deviation over available experiments. Values are recomputed as $V^0_0$ and percentages as $ V^0_0/(V^0_0 + 1 \cdot \rho(\V_0))$ to Table \ref{tab:nyc_objectives}.}
\end{table}

\end{document}